\documentclass{article}

\usepackage[margin=1in]{geometry}
\usepackage{hyperref}
\hypersetup{colorlinks,linkcolor=blue,filecolor=blue,citecolor=magenta,urlcolor=blue}
\usepackage{bm,amsmath,amsthm,amssymb,multicol,algorithmic,algorithm,enumitem,graphicx,subfigure}
\usepackage{xargs}
\usepackage{stmaryrd}

\newtheorem{Lemma}{Lemma}

\newtheorem{Theorem}{Theorem} 

\newtheorem{Corollary}{Corollary}

\newtheorem{assumption}{H\!\!}

\newtheorem*{Lemma*}{Lemma}
\newtheorem*{Theorem*}{Theorem}
\newtheorem*{Corollary*}{Corollary}
 
\newcommand{\eqsp}{\;}
\newcommand{\beq}{\begin{equation}}
\newcommand{\eeq}{\end{equation}}
\newcommand{\eqdef}{\mathrel{\mathop:}=}
\def\EE{\mathbb{E}}
\newcommand{\norm}[1]{\left\Vert #1 \right\Vert}
\newcommand{\pscal}[2]{\left\langle#1\,|\,#2 \right\rangle}
\def\major{\mathsf{M}}
\def\rset{\ensuremath{\mathbb{R}}}

\begin{document}

\title{An Optimistic Acceleration of AMSGrad for Nonconvex Optimization}

\author{\textbf{Jun-Kun Wang, Xiaoyun Li, Belhal Karimi, Ping Li} \\\\
Cognitive Computing Lab\\
Baidu Research\\
  10900 NE 8th St. Bellevue, WA 98004, USA
  %\\
  %\texttt{\{belhal.karimi,\  pingli98\}@gmail.com}
}

\date{}
\maketitle

\begin{abstract}
We propose a new variant of AMSGrad~\cite{Proc:Reddi_ICLR18}, a popular adaptive gradient based optimization algorithm widely used for training deep neural networks. 
Our algorithm adds prior knowledge about the sequence of consecutive mini-batch gradients and leverages its underlying structure making the gradients sequentially predictable. 
By exploiting the predictability and ideas from optimistic online learning, the proposed algorithm can accelerate the convergence and increase sample efficiency.
After establishing a tighter upper bound under some convexity conditions on the regret, we offer a complimentary view of our algorithm which generalizes the offline and stochastic version of nonconvex optimization. 
In the nonconvex case, we establish a non-asymptotic convergence bound independently of the initialization.
We illustrate the practical speedup on several deep learning models via numerical experiments.
\end{abstract}

\section{Introduction}

Deep learning models have been successful in several applications, from robotics (e.g.,~\cite{Article:Levine_JMLR16}), computer vision (e.g.,~\cite{Proc:He_CVPR16,Proc:Goodfellow_NIPS14}), reinforcement learning (e.g.,~\cite{mnih2013playing}) and natural language processing (e.g.,~\cite{Proc:Graves_ICASSP13}).
With the sheer size of modern datasets and the dimension of neural networks, speeding up training is of utmost importance.
To do so, several algorithms have been proposed in recent years, such as  \textsc{AMSGrad}~\cite{Proc:Reddi_ICLR18}, \textsc{Adam}~\cite{Proc:Kingma_ICLR15}, \textsc{RMSPROP}~\cite{Tieleman2012}, \textsc{AdADELTA}~\cite{zeiler2012adadelta},  \textsc{NADAM}~\cite{dozat2016incorporating}, 
\textsc{Local Adam}~\cite{Proc:Chen_FODS20}, \textsc{SAGD}~\cite{Proc:Zhou_NeurIPS20}, 
etc. 
All the prevalent algorithms for training deep networks mentioned above combine two ideas: the idea of adaptivity from \textsc{AdaGrad}~\cite{Proc:McMahan_COLT10,Proc:Duchi_JMLR11} and the idea of momentum from \textsc{Nesterov's Method}~\cite{Book:Nesterov_2004} or \textsc{Heavy ball} method~\cite{Article:Polyak_1964}.
\textsc{AdaGrad} is an online learning algorithm that works well compared to the standard online gradient descent when the gradient is sparse.
Its update has a notable feature: it leverages an anisotropic learning rate depending on the magnitude of the gradient for each dimension which helps in exploiting the geometry of the data. 
On the other hand, \textsc{Nesterov's Method} or \textsc{Heavy ball} Method~\cite{Article:Polyak_1964} is an accelerated optimization algorithm which update not only depends on the current iterate and gradient but also depends on the past gradients (i.e., momentum). 
State-of-the-art algorithms such as \textsc{AMSGrad}~\cite{Proc:Reddi_ICLR18} and \textsc{Adam}~\cite{Proc:Kingma_ICLR15} leverage these ideas to accelerate the training of nonconvex objective functions, for instance deep neural networks losses.

In this paper, we propose an algorithm that goes beyond the hybrid of the adaptivity and momentum approach. 
Our algorithm is inspired by \textsc{Optimistic Online learning}~\cite{Proc:Chiang_COLT12,Proc:Rakhlin_NIPS13,Proc:Syrgkanis_NIPS15,Proc:Abernethy_COLT18,Proc:Mertikopoulos_ICLR19}, which assumes that, in each round of online learning, a \emph{predictable process} of the gradient of the loss function is available. 
Then, an action is played exploiting these predictors. 
By capitalizing on this (possibly) arbitrary process, algorithms in \textsc{Optimistic Online learning} enjoy smaller regret than the ones without gradient predictions.
We combine the \textsc{Optimistic Online learning} idea with the adaptivity and the momentum ideas to design a new algorithm --- \textsc{OPT-AMSGrad}. 

A single work along that direction stands out. 
\cite{Proc:Daskalakis_ICLR18} develop \textsc{Optimistic-Adam} leveraging optimistic online mirror descent~\cite{Proc:Rakhlin_NIPS13}.
Yet, \textsc{Optimistic-Adam} is specifically designed to optimize two-player games, e.g., GANs~\cite{Proc:Goodfellow_NIPS14} which is in particular a two-player zero-sum game. 
There have been some related works in \textsc{Optimistic Online learning}~\cite{Proc:Chiang_COLT12,Proc:Rakhlin_NIPS13,Proc:Syrgkanis_NIPS15} showing that if both players use an $\textsc{Optimistic}$ type of update, then accelerating the convergence to the equilibrium of the game is possible.
\cite{Proc:Daskalakis_ICLR18} build on these related works and show that \textsc{Optimistic-Mirror-Descent} can avoid the cycle behavior in a bilinear zero-sum game accelerating the convergence. 
In contrast, in this paper, the proposed algorithm is designed to accelerate nonconvex optimization (e.g., empirical risk minimization).
To the best of our knowledge, this is the first work exploring towards this direction and bridging the unfilled \emph{theoretical} gap at the crossroads of online learning and stochastic optimization.

\vspace{0.1in}
The \textbf{contributions} of this paper are as follows:
\begin{itemize}
\item We derive an optimistic variant of \textsc{AMSGrad} borrowing techniques from online learning procedures. Our method relies on \textsf{(I)} the addition of \emph{prior knowledge} in the sequence of model parameter estimates leveraging a predictable process able to provide guesses of gradients through the iterations; \textsf{(II)} the construction of a \emph{double update} algorithm done sequentially. We interpret this two-projection step as the learning of the global parameter and of an underlying scheme which makes the gradients sequentially predictable.
\item We focus on the  \emph{theoretical} justifications of our method by establishing novel \emph{non-asymptotic} and \emph{global} convergence rates in both convex and nonconvex cases.  Based on \emph{convex regret minimization} and \emph{nonconvex stochastic optimization} views, we prove, respectively, that our algorithm suffers regret of $\mathcal{O}(\sqrt{\sum_{t=1}^T \| g_t - m_t  \|^2_{\psi_{t-1}}})$ and achieves a convergence rate $\mathcal{O}(\sqrt{d/T} +d/T )$, where $g_t$ is the gradient, $m_t$ is its prediction, $d$ is the dimension of the problem and $T$ the total number of iterations.
\end{itemize}
The proposed algorithm adapts to the informative dimensions, exhibits momentum, and also exploits a good guess of the next gradient to facilitate acceleration. 
Besides the complete convergence analysis of \textsc{OPT-AMSGrad}, we conduct numerical experiments and show that the proposed algorithm not only accelerates the training procedure, but also leads to better empirical generalization performance.

\vspace{0.1in}

\textbf{Notations:} 
We follow the notations of adaptive optimization~\cite{Proc:Kingma_ICLR15,Proc:Reddi_ICLR18}. 
For any $u$, $v \in \mathbb R^{d}$,  $u/v$ represents the element-wise division,
$u^{2}$ the element-wise square, $\sqrt{u}$ the element-wise square-root.
We denote $g_{1:T}[i]$ as the sum of the $i_{th}$ element of $g_{1},\dots, g_{T} \in \mathbb R^{d}$ and $\| \cdot \|$ as the Euclidean norm.

\section{Preliminaries}\label{sec:prelim}

We begin by providing some background on both online learning and adaptive methods.

 \textbf{Optimistic Online learning.}\hspace{0.1cm}
The standard setup of \textsc{Online learning} is that, in each round $t$, an online learner selects an action $w_{t} \in \Theta \subseteq \mathbb R^{d}$, observes $\ell_{t}(\cdot)$ and suffers the associated loss $\ell_{t}(w_t)$ after the action is committed.
The goal of the learner is to minimize the regret, 
$$\mathcal{R}_{T}( \{ w_t \} ):= \sum_{t=1}^T \ell_{t}(w_t) - \sum_{t=1}^T \ell_{t}(w^*)\eqsp,$$
which is the cumulative loss of the learner minus the cumulative loss of some benchmark $w^{*} \in \Theta$.
The idea of \textsc{Optimistic Online learning} (e.g.,~\cite{Proc:Chiang_COLT12,Proc:Rakhlin_NIPS13,Proc:Syrgkanis_NIPS15,Proc:Abernethy_COLT18}) is as follows.
In each round $t$, the learner exploits a guess $m_t(\cdot)$ of the gradient $\nabla \ell_t(\cdot)$ to choose an action $w_t$\footnote{Imagine that if the learner would have known $\nabla \ell_t(\cdot)$ (i.e., exact guess) before committing its action, then it would exploit the knowledge to determine its action and consequently minimize the regret.}. 
Consider the \textsc{Follow-the-Regularized-Leader} (\textsc{FTRL},~\cite{hazan2019introduction}) online learning algorithm which update reads
\begin{equation} \notag
\textstyle w_t  = \arg \min_{w \in \Theta} \langle w , L_{t-1} \rangle + \frac{1}{\eta} \textsf{R}(w) \eqsp,
\end{equation}
where $\eta$ is a parameter, $\textsf{R}(\cdot)$ is a $1$-strongly convex function with respect to a given norm on the constraint set $\Theta$, and $L_{t-1}:= \sum_{s=1}^{t-1} g_s$ is the cumulative sum of gradient vectors of the loss functions up to round $t-1$. It has been shown that FTRL has regret at most $\mathcal{O}(\sqrt{\sum_{t=1}^T \| g_t \|_*^2})$.
The update of its optimistic variant, called \textsc{Optimistic-FTRL} and developed in~\cite{Proc:Syrgkanis_NIPS15} reads
\begin{equation}\notag
\textstyle w_t  = \arg \min_{w \in \Theta} \langle w , L_{t-1} + m_t \rangle + \frac{1}{\eta} \textsf{R}(w)\eqsp,
\end{equation}
where $\{m_{t}\}_{t>0}$ is a predictable process incorporating (possibly arbitrary) knowledge about the sequence of gradients $\{ g_{t}:=\nabla \ell_t(w_t)\}_{t>0}$.
Under the assumption that the loss functions are convex, it has been shown in~\cite{Proc:Syrgkanis_NIPS15} that the regret of \textsc{Optimistic-FTRL} is at most $\mathcal{O}(\sqrt{\sum_{t=1}^T \| g_t - m_t \|_*^2 } )$.

\vspace{0.1in}

\textit{Remark:} Note that the usual worst-case bound is preserved even when the predictors $\{m_{t}\}_{t>0}$ do not predict well the gradients. Indeed, if we take the example of \textsc{Optimistic-FTRL}, the bound reads $\sqrt{\sum_{t=1}^T \| g_t - m_t \|_*^2 } \leq 2 \max \limits_{w \in \Theta} \| \nabla \ell_t(w) \| \sqrt{T}$ which is equal to the usual bound up to a factor $2$~\cite{Proc:Rakhlin_NIPS13}, under certain boundedness assumptions on $\Theta$ detailed below.
Yet, when the predictors $\{m_{t}\}_{t>0}$ are well designed, the resulting regret will be lower. 
We will have a similar argument when comparing \textsc{OPT-AMSGrad} and \textsc{AMSGrad} regret bounds in Section~\ref{sec:convex}.

We emphasize, in Section~\ref{sec:opt}, the importance of leveraging a good guess $m_t$ for updating $w_t$ in order to get a fast convergence rate (or equivalently, small regret) and introduce in Section~\ref{sec:numerical} a simple predictable process $\{m_{t}\}_{t>0}$ leading to empirical acceleration on various applications.

\vspace{0.1in}
\textbf{Adaptive optimization methods.}\hspace{0.1cm}
Adaptive optimization has been popular in various deep learning applications due to their superior empirical performance.
\textsc{Adam}~\cite{Proc:Kingma_ICLR15}, a popular adaptive algorithm, combines momentum~\cite{Article:Polyak_1964} and anisotropic learning rate of \textsc{AdaGrad}~\cite{Proc:Duchi_JMLR11}.
More specifically, the learning rate of \textsc{AdaGrad} at time $T$ for dimension $j$ is proportional to the inverse of $\sqrt{ \Sigma_{t=1}^T g_t[j]^2 }$, where $g_t[j]$ is the $j$-th element of the gradient vector $g_t$ at time $t$.
This adaptive learning rate helps accelerating the convergence when the gradient vector is sparse~\cite{Proc:Duchi_JMLR11}, yet, when applying \textsc{AdaGrad} to train deep neural networks, it is observed that the learning rate might decay too fast, see~\cite{Proc:Kingma_ICLR15} for more details.
Therefore,~\cite{Proc:Kingma_ICLR15} put forward \textsc{Adam} that uses a moving average of the gradients divided by the square root of the second moment of this moving average (element-wise multiplication), for updating the model parameter $w$.
A variant, called \textsc{AMSGrad} and detailed in Algorithm~\ref{alg:amsgrad}, has been developed in~\cite{Proc:Reddi_ICLR18} to fix \textsc{Adam} failures.

\noindent{\centering
\begin{minipage}{.5\linewidth}
\begin{algorithm}[H]
\caption{\textsc{AMSGrad}~\cite{Proc:Reddi_ICLR18}} \label{alg:amsgrad}
\begin{algorithmic}[1]
%\small
\STATE \textbf{Required}: parameter $\beta_1$, $\beta_2$, and $\eta_t$. 
\STATE Init: $w_{1} \in \Theta \subseteq \mathbb R^d $ and $v_{0} = \epsilon 1 \in \mathbb R^{d}$.
\FOR{$t=1$ to $T$}
\STATE Get mini-batch stochastic gradient $g_t$ at $w_t$.
\STATE $\theta_t = \beta_1 \theta_{t-1} + (1 - \beta_1) g_t$.
\STATE $v_t = \beta_2 v_{t-1} + (1 - \beta_2) g_t^2$. 
\STATE \label{line:maxop}$\hat{v}_t = \max( \hat{v}_{t-1} , v_t )$. 
\STATE $w_{t+1} = w_t - \eta_t \frac{\theta_t}{ \sqrt{\hat{v}}_t }$.
\text{ (element-wise division)}
\ENDFOR
\end{algorithmic}
\end{algorithm}
 \end{minipage}
 \par
 }

\vspace{0.3in}

The difference between \textsc{Adam} and \textsc{AMSGrad} lies in line 7 of Algorithm~\ref{alg:amsgrad}.
The \textsc{AMSGrad} algorithm~\cite{Proc:Reddi_ICLR18} applies the $\texttt{max}$ operation on the second moment to guarantee a non-increasing learning rate $\eta_t / \sqrt{\hat{v}}_t $, which helps for the convergence (i.e., average regret $\mathcal{R}_T/T \rightarrow 0$).

\section{\textsc{OPT-AMSGRAD} Algorithm}\label{sec:opt}

We formulate in this section the proposed optimistic acceleration of AMSGrad, namely \textsc{OPT-AMSGrad}, and detailed in Algorithm~\ref{alg:optamsgrad}.  
It combines the idea of \emph{adaptive optimization} with \emph{optimistic learning}. 
At each iteration, the learner computes a gradient vector $g_{t}:= \nabla \ell_t( w_t)$ at $w_{t}$ (line 4), then maintains an exponential moving average of $\theta_{t} \in \mathbb R^{d}$ (line 5) and $v_{t} \in \mathbb R^{d}$ (line 6),  followed by the $\texttt{max}$ operation to get $\hat{v}_{t} \in \mathbb R^{d}$ (line 7). 
The learner first updates an auxiliary variable $\tilde{w}_{t+1} \in \Theta$ (line 8), then computes the next model parameter $w_{t+1}$ (line 9).
Observe that the proposed algorithm does not reduce to \textsc{AMSGrad} when $m_{t}=0$, contrary to the optimistic variant of \textsc{FTRL}.
Furthermore, combining line 8 and line 9 yields the following single step $w_{t+1}= \tilde{w}_{t}  - \eta_t (\theta_t + h_{t+1} )/ \sqrt{\hat{v}}_t $. 
Compared to \textsc{AMSGrad}, the algorithm is characterized by a \emph{two-level} update that interlinks some \emph{auxiliary state} $\tilde{w}_{t}$ and the model parameter state, $w_t$, similarly to the \textsc{Optimistic Mirror Descent} algorithm developed in~\cite{Proc:Rakhlin_NIPS13}.
It leverages the auxiliary variable (hidden model) to update and commit $w_{t+1}$, which exploits the guess $m_{t+1}$, see Figure~\ref{fig:scheme}.

\vspace{0.2in}

\noindent \begin{minipage}{0.5\linewidth}
\begin{algorithm}[H]
\begin{algorithmic}[1] 
\caption{\textsc{OPT-AMSGrad}} \label{alg:optamsgrad}
\STATE \textbf{Required}: parameter $\beta_1$, $\beta_2$, $\epsilon$, and $\eta_t$. 
\STATE Init: $w_1 = w_{-1/2} \in \Theta \subseteq \mathbb R^d $ and $v_{0} = \epsilon 1 \in \mathbb R^{d}$.
\FOR{$t=1$ to $T$}
\STATE Get mini-batch stochastic gradient $g_t$ at $w_t$.
\STATE $\theta_t = \beta_{1} \theta_{t-1} + (1 - \beta_{1}) g_t$.
\STATE $v_t = \beta_2 v_{t-1} + (1 - \beta_2) g_t^{2}$.
\STATE $\hat{v}_t = \max( \hat{v}_{t-1} , v_t )$. 
\STATE $\tilde{w}_{t+1} =  \tilde{w}_{t} - \eta_t \frac{\theta_t}{ \sqrt{\hat{v}_t }  } $.
\STATE $w_{t+1} = \tilde{w}_{t+1} - \eta_{t} \frac{h_{t+1}}{ \sqrt{\hat{v}_t } } $,  \\  
where $h_{t+1}:= \beta_{1} \theta_{t-1} + (1 - \beta_{1}) m_{t+1}$ with $m_{t+1}$ being the guess of $g_{t+1}$. 
\ENDFOR 
\end{algorithmic}
\end{algorithm}
\end{minipage}
\hfill
\begin{minipage}{0.5\linewidth}
\begin{figure}[H]
\centering
    \includegraphics[width=3.3in]{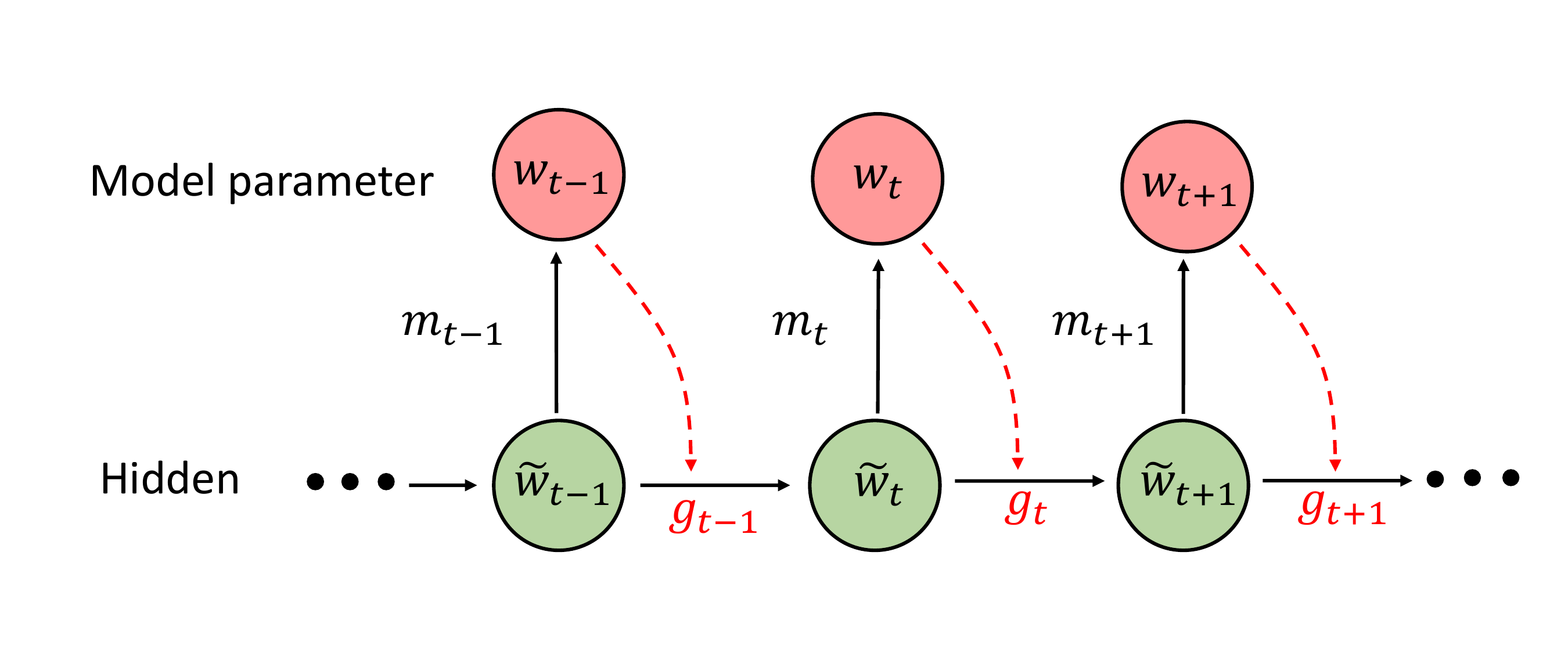}
    \caption{\textsc{OPT-AMSGrad} underlying structure.}
     \label{fig:scheme}
\end{figure}
\end{minipage}
\vspace{0.2in}

In the following analysis, we show that this interleaving actually leads to some cancellation in the regret bound.
Such two-levels method where the guess $m_t$ is equal to the last known gradient $g_{t-1}$ has been exhibited recently in~\cite{Proc:Chiang_COLT12}.
The gradient prediction process plays an important role as discussed in Section~\ref{sec:numerical}.
The proposed \textsc{OPT-AMSGrad} algorithm inherits three properties: \textsf{(i)} Adaptive learning rate of each dimension as \textsc{AdaGrad}~\cite{Proc:Duchi_JMLR11} (line 6, line 8 and line 9). \textsf{(ii)} Exponential moving average of the past gradients as \textsc{Nesterov's method}~\cite{Book:Nesterov_2004} and the \textsc{Heavy-Ball} method~\cite{Article:Polyak_1964} (line 5). \textsf{(iii)} Optimistic update that exploits \emph{prior knowledge} of the next gradient vector as in optimistic online learning algorithms~\cite{Proc:Chiang_COLT12,Proc:Rakhlin_NIPS13,Proc:Syrgkanis_NIPS15} (line 9).
The first property helps for acceleration when the gradient has a sparse structure.
The second one is from the long-established idea of momentum which can also help for acceleration. 
The last property can lead to an acceleration when the prediction of the next gradient is good as mentioned above when introducing the regret bound for the \textsc{Optimistic-FTRL} algorithm.
This property will be elaborated whilst establishing the theoretical analysis of \textsc{OPT-AMSGrad}.

\section{Convergence Analysis}\label{sec:analysis}
In this section, we provide regret analysis of the proposed method and show that it can improve the bound of vanilla \textsc{AMSGrad} with a good guess of the gradient.
The convex result is followed by a global non-asymptotic analysis for nonconvex objective functions where the optimization problem tackled in this paper is cast as an offline and stochastic nonconvex optimization problem.

\vspace{0.08in}
 \textbf{More notations.}\hspace{0.1in}
We denote the Mahalanobis norm by $\|\cdot\|_H := \sqrt{ \langle \cdot, H \cdot \rangle }$ for some positive semidefinite (PSD) matrix $H$.
We let $\psi_t(x) := \langle x, \text{diag}\{\hat{v}_t\}^{1/2} x \rangle$ for a PSD matrix $H_t^{1/2}:= \text{diag}\{\hat{v}_t\}^{1/2}$, where $\text{diag}\{\hat{v}_t\}$ represents the diagonal matrix which $i_{th}$ diagonal element is $\hat{v}_t[i]$ defined in Algorithm~\ref{alg:optamsgrad}.
We define its corresponding Mahalanobis norm by $\| \cdot \|_{\psi_t}:=  \sqrt{ \langle \cdot, \text{diag}\{\hat{v}_t\}^{1/2} \cdot \rangle }$,
where we abuse the notation $\psi_t$ to represent the PSD matrix $H_t^{1/2}:=\text{diag}\{\hat{v}_t\}^{1/2}$.
Note that $\psi_t(\cdot)$ is 1-strongly convex with respect to the norm $\| \cdot \|_{\psi_t}$, i.e., $\psi_t(\cdot)$ satisfies $\psi_t(u) \geq \psi_t(v) + \langle \psi_t(v), u - v \rangle + \frac{1}{2} \| u - v\|^2_{\psi_t}$ for any point $(u,v) \in \Theta^2$.
A consequence of 1-strong convexity of $\psi_t(\cdot)$ is that $B_{\psi_t}(u,v) \geq \frac{1}{2} \| u - v \|^2_{\psi_t}$, where the Bregman divergence $B_{\psi_t}(u,v)$ is defined as $B_{\psi_t}(u,v) := \psi_t(u) - \psi_t(v) - \langle \psi_t(v), u - v \rangle$ with $\psi_t(\cdot)$ as the distance generating function.
We also define the corresponding dual norm $\| \cdot \|_{\psi_t^*}:= \sqrt{ \langle \cdot, \text{diag}\{\hat{v}_t\}^{-1/2} \cdot \rangle }$.
The proofs of the following theoretical results are deferred to the Appendix.

\subsection{Convex Regret Analysis}\label{sec:convex}

In the following, we assume convexity of $\{\ell_t\}_{t>0}$ and that $\Theta$ has a bounded diameter $D_{\infty}$, which is a standard assumption for adaptive methods~\cite{Proc:Reddi_ICLR18,Proc:Kingma_ICLR15} and is necessary in regret analysis.

\begin{Theorem} \label{thm:mainconvex}\vspace{0.05in}
Suppose the learner incurs a sequence of convex loss functions $\{ \ell_{t}(\cdot) \}$.
Then,  \textsc{OPT-AMSGrad} (Algorithm~\ref{alg:optamsgrad}) has regret 
\begin{equation} \notag
\begin{split}
\mathcal{R}_T \leq    \frac{ B_{\psi_1}(w^*, \tilde{w}_{1})}{\eta_1}
+ \sum_{t=1}^T\frac{\eta_t}{2} \| g_t - \tilde{m}_t  \|_{\psi_{t-1}^*}^2  + \frac{D_{\infty}^2}{\eta_{\min}}  \sum_{i=1}^d \hat{v}_{T}^{1/2}[i] + D_{\infty}^2\beta_1^2   \sum_{t=1}^T  \| g_t - \theta_{t-1}  \|_{\psi_{t-1}^*}\eqsp,
\end{split}
\end{equation}
where $ \tilde{m}_{t+1}  = \beta_1 \theta_{t-1} +(1-\beta_1) m_{t+1}$, $g_{t}:= \nabla \ell_{t}(w_t)$, $\eta_{{\min}} := \min_{{t}} \eta_{t}$ and $D_{\infty}^2$ is the diameter of the bounded set $\Theta$.
The result holds for any benchmark $w^{*} \in \Theta$ and any step size sequence $\{ \eta_t \}_{t>0}$.
\end{Theorem}
\vspace{0.05in}
\begin{Corollary}\label{cor:corollary}
Suppose $\beta_1=0$ and $\{v_t\}_{t>0}$ is a monotonically increasing sequence, then we obtain the following regret bound for any $w^{*} \in \Theta$ and sequence of stepsizes $\{ \eta_t = \eta/\sqrt{t}\}_{t>0}$: 
\begin{equation}\notag
\begin{split}
\mathcal{R}_T \leq  \frac{ B_{\psi_1}}{\eta_1}
+ \frac{\eta \sqrt{1 + \log T}}{\sqrt{1 - \beta_2}} \sum_{i=1}^d \| (g-m)_{1:T}[i] \|_2  +\frac{D_{\infty}^2}{\eta_{\min}} \sum_{i=1}^d \left[ (1-\beta_2) \sum_{s=1}^{T} \beta_2^{T-s} g^2_s[i] \right]^{1/2} \eqsp,
\end{split}
\end{equation}
where $B_{\psi_1} := B_{\psi_1}(w^*, \tilde{w}_{1})$, $g_{t}:= \nabla \ell_{t}(w_t)$ and $\eta_{{\min}} := \min_{{t}} \eta_{t}$.
\end{Corollary}
We can compare the bound of Corollary~\ref{cor:corollary} with that of \textsc{AMSGrad}~\cite{Proc:Reddi_ICLR18} with $\eta_t = \eta/\sqrt{t}$:
\beq\label{eq:boundAMS}
\mathcal{R}_T \leq \frac{ \eta \sqrt{ 1 + \log T}    }{ \sqrt{1 - \beta_2}  } \sum_{i=1}^d \| g_{1:T}[i]  \|_2 +\frac{\sqrt{T}}{2 \eta} D^2_{\infty} \sum_{i=1}^d \hat{v}_T[i]^2 \eqsp.
\eeq
For convex regret minimization, Corollary~\ref{cor:corollary} yields a regret of $\mathcal{O}(\sqrt{\sum_{t=1}^T \| g_t - m_t  \|^2_{\psi^*_{t-1}}})$ with an access to an arbitrary predictable process $\{m_t\}_{t>0}$ of the mini-batch gradients.
We notice from the second term in Corollary~\ref{cor:corollary} compared to the first term in \eqref{eq:boundAMS} that better predictors lead to lower regret.
The construction of the predictions $\{m_t\}_{t>0}$ is thus of utmost importance for achieving optimal acceleration and can be learned through the iterations~\cite{Proc:Rakhlin_NIPS13}.
%We will not deal with the latter in this paper for the sake of space and clarity.
In Section~\ref{sec:numerical}, we derive a basic, yet effective, gradient prediction algorithm, see Algorithm~\ref{alg:algex}, embedded in \textsc{OPT-AMSGrad}.

\subsection{Finite-Time Analysis in  Nonconvex Case}

We discuss the offline and stochastic nonconvex optimization properties of our online framework.
As stated in the introduction, this paper is about solving optimization problems instead of solving zero-sum games.  
Classically, the optimization problem we are tackling reads:
\beq\label{eq:minproblem}
\min \limits_{w \in \Theta} f(w) \eqdef  \EE[ f(w, \xi)] = n^{-1} \sum_{i=1}^n  \EE[f(w, \xi_i)] \eqsp,
\eeq
for a fixed batch of $n$ samples $\{ \xi_i \}_{i=1}^n$.
The objective function $f(\cdot)$ is (potentially) nonconvex and has Lipschitz gradients.
Set the terminating number, $T \in \{0,\dots,T_{\sf M}-1\}$, as a discrete r.v.~with:
\beq \label{eq:random}
   P( T = \ell ) = \frac{ \eta_{\ell} }{\sum_{j=0}^{T_{\sf M}-1} \eta_j} \eqsp,
\eeq
where $T_{\sf M}$ is the maximum number of iteration.
The random termination number \eqref{eq:random} is inspired by~\cite{Article:Ghadimi_SJOPT13} and is widely used to derive novel results in nonconvex optimization. 
Consider the following assumptions: 

\begin{assumption}\label{ass:boundedparam}
For any $t >0$, the estimated parameter $w_t$ stays within a $\ell_{\infty}-$ball. There exists a constant $W >0$ such that $\norm{w_t}_{\infty} \leq W$ almost surely.
\end{assumption}

\vspace{0.1in}

\begin{assumption}\label{ass:smooth}
The function $f$ is $L$-smooth (has $L$-Lipschitz gradients) w.r.t. the parameter w.
There exist some constant $L > 0$ such that for $(w, \vartheta) \in \Theta^2$, $f(w) - f(\vartheta) - \nabla f(\vartheta)^\top(w - \vartheta) \leq \frac{L}{2} \norm{w - \vartheta}^2\eqsp.$
\end{assumption}

For nonconvex analysis, we assume the following:
\begin{assumption}\label{ass:guessbound}
For any $t >0$, $0 < \pscal{m_t}{ g_t} = a_t \|g_t\|^2$ with some $0<a_t\leq 1$, and $\| m_t \|\leq \| g_t \|$, where $\pscal{}{}$ denotes the inner product.
\end{assumption}
H\ref{ass:guessbound} assumes that the predicted gradient is in general reasonable, in the sense that $m_t$ has acute angle with $g_t$ and bounded norm, as the shadowed area in Figure~\ref{fig:assumption}.

\begin{figure}[H]
\centering
    \includegraphics[width=3.3in]{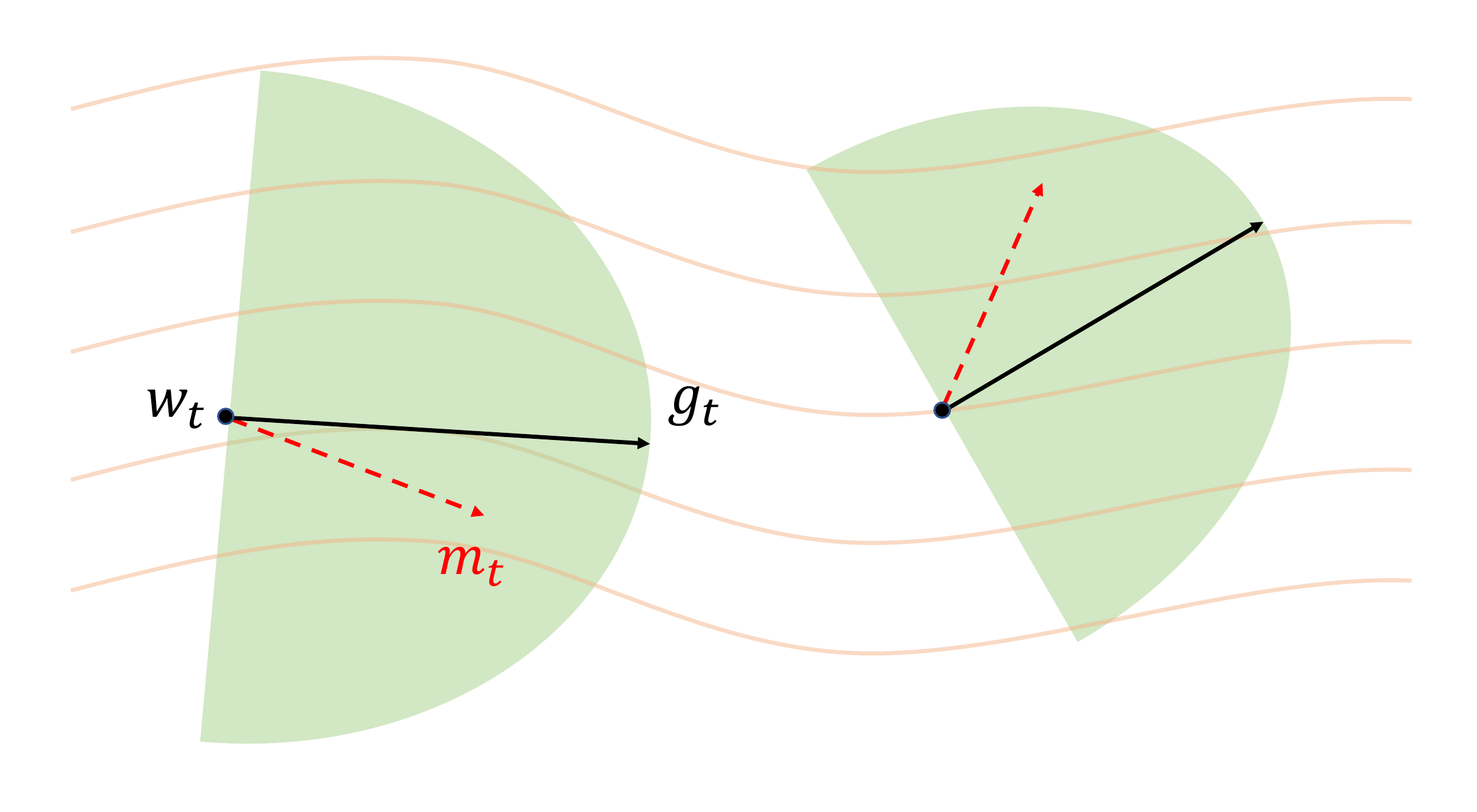}
    \caption{Illustration of H\ref{ass:guessbound}.}
     \label{fig:assumption}
\end{figure}

Lastly, We make a classical assumption in nonconvex optimization on the magnitude of the gradient:
\begin{assumption}\label{ass:bounded}
There exists a constant $\major >0$ such that for any $w$ and $\xi$, it holds that $\norm{\nabla f(w, \xi)} < \major$.
\end{assumption}

We now derive important results for our global analysis.
The first one ensures bounded norms of quantities of interests (resulting from the bounded stochastic gradient assumption):
\begin{Lemma}\label{lem:bound}\vspace{0.05in}
Assume H\ref{ass:bounded}, then the quantities defined in Algorithm~\ref{alg:optamsgrad} satisfy for any $w \in \Theta$ and $t>0$, $ \|\nabla f(w_t)\| < \major ,~~~\|\theta_t \| < \major$ and $\|\hat{v}_t\| < \major^2$.
\end{Lemma}
We now formulate the main result of our paper yielding a finite-time upper bound of the suboptimality condition defined as $\EE\left[\|\nabla f(w_T)\|^2\right]$ (set as the convergence criterion of interest, see~\cite{Article:Ghadimi_SJOPT13}):
\begin{Theorem}\label{thm:boundopt}\vspace{0.05in}
Assume H\ref{ass:boundedparam}-H\ref{ass:bounded}, $\beta_1 < \beta_2 \in [0,1)$ and a sequence of decreasing stepsizes $\{\eta_t\}_{t>0}$, then the following result holds:
\beq\notag
\begin{split}
\EE\left[\|\nabla f(w_T)\|_2^2\right] \leq \tilde{C}_1 \sqrt{\frac{d}{T_{\sf M}}} + \tilde{C}_2 \frac{1}{T_{\sf M}} \eqsp,
\end{split}
\eeq
where $T$ is a random termination number distributed according \eqref{eq:random}.
The constants are defined as:
{\fontsize{9.5}{9}
\begin{align*}
&\tilde{C}_1 =\frac{\major}{(1 - a_m\beta_1) + (\beta_1 + a_m)} \left[ \frac{a_m(1 - \beta_1)^2}{1-\beta_2} + 2L \frac{1}{1-\beta_2} +  \Delta f  +   \frac{4L \beta_1^2(1 + \beta_1^2) }{(1 - \beta_1)(1 - \beta_2)(1-\gamma)} \right],\\
&\tilde{C}_2 = \frac{ \major^2 }{(1 - \beta_1) \left((1 - a_m\beta_1) + (\beta_1 + a_m)\right)} \left[a_m\beta_1^2 -2 a_m \beta_1 + \beta_1\right] \EE\left[ \norm{\hat{v}_{0}^{-1/2}}    \right]  \eqsp,
\end{align*}
}
where $\Delta f = f(\overline{w}_{1}) - f(\overline{w}_{T_{\sf M}+1})$ and $a_m=\displaystyle{\min_{t=1,...,T}}a_t$.
\end{Theorem}
Firstly, the bound for our OPT-AMSGrad method matches the complexity bound of $\mathcal{O}( \sqrt{d/T_{\sf M}} + 1/T_{\sf M} )$ of~\cite{Article:Ghadimi_SJOPT13} for SGD considering the dependence of T only, and of~\cite{zhou2018convergence} for AMSGrad method.To see the influence of prediction quality, we can show that when $(1-\beta_1)(\beta_2-\beta_1^2-2L(1-\beta_1))-\frac{4L\beta_1^2(1+\beta_1^2)}{1-\gamma}<0$, $\tilde C_1$ and $\tilde C_2$ both decrease as $a_m$ approaches 1, i.e., as the prediction gets more accurate. Therefore, similar to the convex case, our nonconvex bound also improves with better gradient prediction.

\subsection{Checking H\ref{ass:boundedparam} for a Deep Neural Network}

As boundedness assumption H\ref{ass:boundedparam} is generally hard to verify, we now show, for illustrative purposes, that the weights of a fully connected feed forward neural network stay in a bounded set when being trained using our method. 
The activation function for this section will be sigmoid function and we use a $\ell_2$ regularization. 
We consider a fully connected feed forward neural network with $L$ layers modeled by the function $\textsf{MLN}(w, \xi): \Theta^d \times \rset^p \to \rset$ defined as:
\beq\label{eq:dnnmodel}
\textsf{MLN}( w, \xi) = \sigma\left(w^{(L)} \sigma\left(w^{(L-1)} \ldots \sigma\left(w^{(1)} \xi \right)\right)\right) \eqsp,
\eeq
where $w = [w^{(1)}, w^{(2)}, \cdots , w^{(L)}]$ is the vector of parameters, $\xi \in \rset^p$ is the input data and $\sigma$ is the sigmoid activation function. We assume a $p$ dimension input data and a scalar output for simplicity.
In this setting, the stochastic objective function \eqref{eq:minproblem} reads 
$$
f(w, \xi) = \mathcal{L}(\textsf{MLN}( w, \xi), y) +\frac{\lambda}{2}\norm{w}^2\eqsp,
$$
where $\mathcal{L}(\cdot, y)$ is the loss function (e.g., cross-entropy), $y$ are the true labels and $\lambda >0$ is the regularization parameter.
We establish that the boundedness assumption H\ref{ass:boundedparam} is satisfied with model~\eqref{eq:dnnmodel} via the following:
\begin{Lemma}\label{lem:dnnh2}
Given the multilayer model \eqref{eq:dnnmodel}, assume the boundedness of the input data and of the loss function, i.e., for any $\xi \in \rset^p$ and $y \in \rset$ there is a constant $T >0$ such that $\norm{\xi} \leq 1 \quad \textrm{a.s.}$ and $|\mathcal{L}'(\cdot, y)| \leq T$ where $\mathcal{L}'(\cdot, y)$ denotes its derivative w.r.t. the parameter. Then for each layer $\ell \in [1,L]$, there exists a constant $A_{(\ell)}$ such that $\norm{w^{(\ell)}} \leq A_{(\ell)}$.
\end{Lemma}

\section{Comparison to related methods} \label{sec:related}
We give in this section some comparable methods to our \textsc{OPT-AMSGrad} algorithm such as AO-FTRL~\cite{Proc:Mohri_AISTATS16} or \textsc{Optimistic-Adam}~\cite{Proc:Daskalakis_ICLR18}.

\vspace{0.05in}

 \textbf{Comparison to nonconvex optimization works.}\hspace{0.1in}
Recently,~\cite{Proc:Zaheer_NeurIPS18,Proc:Chen_ICLR19,Proc:Ward_ICML19,zhou2018convergence,zou2018convergence,Proc:Li_AISTATS19} provide some theoretical analysis of \textsc{Adam}-type algorithms when applying them to smooth nonconvex optimization problems. 
For example,~\cite{Proc:Chen_ICLR19} provide the following bound $\min_{t \in [T]} \mathbb{E}[\| \nabla f(w_t) \|^2 ] = \mathcal{O}(\log T / \sqrt{T}) $.
Yet, this data independent bound does not show any advantage over standard stochastic gradient descent. 
Similar concerns appear in other related works.
To get some adaptive data dependent bound written in terms of the gradient norms observed along the trajectory when applying  \textsc{OPT-AMSGrad} to nonconvex optimization, one can follow the approach of~\cite{Proc:Agarwal_ICML19} or~\cite{Proc:Chen_Yuan_ICLR19}.
They provide a modular approach to convert algorithms with adaptive data dependent regret bound for convex loss functions (e.g., \textsc{AdaGrad}) to algorithms that can find an approximate stationary point of nonconvex objectives. 
These variants can outperform the ones instantiated by other \textsc{Adam}-type algorithms when the gradient prediction $m_t$ is close to the true gradient $g_t$.

\vspace{0.1in}
\textbf{Comparison to AO-FTRL~\cite{Proc:Mohri_AISTATS16}.}\hspace{0.1in}In~\cite{Proc:Mohri_AISTATS16}, the authors propose \textsc{AO-FTRL}, which update reads $w_{t+1} = \arg\min_{{w \in \Theta}} ( \sum_{s=1}^t g_s )^{\top}  w + m_{t+1}^\top w + r_{0:t}(w) $, where $r_{0:t}(\cdot)$ is a 1-strongly convex loss function with respect to some norm $\| \cdot\|_{(t)}$ that may be different for different iteration $t$. 
Data dependent regret bound provided in~\cite{Proc:Mohri_AISTATS16} reads $r_{{0:T}}(w^*) + \sum_{t=1}^T \| g_t - m_t \|_{(t)^*}$ for any benchmark $w^{*} \in \Theta$. 
We remark that if one selects $r_{0:t}(w) := \langle w, \text{diag}\{\hat{v}_t\}^{1/2} w \rangle$  and $\| \cdot \|_{(t)}:=  \sqrt{ \langle \cdot, \text{diag}\{\hat{v}_t\}^{1/2} \cdot \rangle }$, then the update might be viewed as an optimistic variant of $\textsc{AdaGrad}$. 
However, no experiments were provided in~\cite{Proc:Mohri_AISTATS16} to back those findings.

\vspace{0.1in}
\textbf{Comparison to \textsc{Optimistic-Adam}~\cite{Proc:Daskalakis_ICLR18}.} This is an optimistic variant of ADAM, namely \textsc{Optimistic-Adam}. 
A slightly modified version is summarized in Algorithm~\ref{OPT-DISZ}.  Here, \textsc{Optimistic-Adam$+\hat{v}_t$} corresponds to \textsc{Optimistic-Adam} with the additional max operation $\hat{v}_t = \max ( \hat{v}_{t-1}, v_t)$ to guarantee that the weighted second moment is monotone increasing.
We want to emphasize that the motivations of our optimistic algorithm are different. 
\textsc{Optimistic-Adam} is designed to optimize two-player games (e.g., GANs~\cite{Proc:Goodfellow_NIPS14}), while our proposed algorithm \textsc{OPT-AMSGrad} is designed to accelerate optimization (e.g., solving empirical risk minimization).
\cite{Proc:Daskalakis_ICLR18} focuses on training GANs~\cite{Proc:Goodfellow_NIPS14} as a two-player zero-sum game. 
\cite{Proc:Daskalakis_ICLR18} was inspired by these related works and showed that \textsc{Optimistic-Mirror-Descent} can avoid the cycle behavior in a bilinear zero-sum game, which accelerates the convergence.

\noindent{\centering
\begin{minipage}{.5\linewidth}
\begin{algorithm}[H]
\begin{algorithmic}[1]
\caption{\textsc{Optimistic-Adam~\cite{Proc:Daskalakis_ICLR18}+$\hat{v}_t$}. \label{OPT-DISZ}}
\STATE \textbf{Required}: parameter $\beta_1$, $\beta_2$, and $\eta_t$.
\STATE Init: $w_1 \in \Theta$ and $\hat{v}_0 = v_{0} = \epsilon 1 \in \rset^{d}$.
\FOR{$t=1$ to $T$}
\STATE Compute stochastic gradient vector $g_t$ at $w_t$.
\STATE $\theta_t = \beta_{1} \theta_{t-1} + (1 - \beta_{1}) g_t$.
\STATE $v_t = \beta_2 v_{t-1} + (1 - \beta_2) g_t^2$.
\STATE $\hat{v}_t = \max( \hat{v}_{t-1} , v_t )$.
\STATE $w_{t+1} = \Pi_{k}[ w_{t} - 2 \eta_t \frac{\theta_t}{ \sqrt{\hat{v}_t }}
+ \eta_t \frac{\theta_{t-1}}{ \sqrt{\hat{v}_{t-1}} }]$.
\ENDFOR
\end{algorithmic}
\end{algorithm}
 \end{minipage}
 \par
 }

\section{Numerical Experiments}\label{sec:numerical}

In this section, we provide experiments on classification tasks with various neural network architectures and datasets to demonstrate the effectiveness of \textsc{OPT-AMSGrad} in practice and justify its theoretical convergence acceleration.
We start with giving an overview of the gradient predictor process before presenting our numerical runs.

\subsection{Gradient Estimation}

Based on the analysis in the previous section, we understand that the choice of the prediction $m_t$ plays an important role in the convergence of \textsc{Optimistic-AMSGrad}.
Some classical works in gradient prediction methods include \textsc{Anderson} acceleration~\cite{Article:Walker_SJNA11}, \textsc{Minimal Polynomial Extrapolation}~\cite{cabay1976polynomial} and  \textsc{Reduced Rank Extrapolation}~\cite{eddy1979extrapolating}.
These methods typically assume that the sequence $\{g_t\} \in \mathbb R^d$ has a linear relation $g_t = A( g_{t-1} - g^* ) + g^*$ where $A \in \mathbb R^{d \times d}$ is an unknown, not necessarily symmetric, matrix.
Then, these methods aim at finding a fixed point $g^{*}$ and assume that $\{g_t \in \mathbb R^d\}_{t>0} $ has the following linear relation:
\begin{equation} \label{nox}
g_t - g^* = A( g_{t-1} - g^* ) + e_t,
\end{equation}
where $e_t$ is a second order term satisfying $\| e_t \|_2  = \mathcal{O}( \| g_{t-1} - g^* \|_2^2)$, see~\cite{Article:Scieur_MP20} for details and results.
For our numerical experiments, we run \textsc{OPT-AMSGrad} using Algorithm~\ref{alg:algex} to construct the sequence $\{m_t\}_{t>0}$ which is based on estimating the limit of a sequence using the last iterates~\cite{brezinski2013extrapolation}.

\noindent{\centering
\begin{minipage}{.8\linewidth}
\begin{algorithm}[H]
\begin{algorithmic}[1] 
\caption{Regularized Approximated Minimal Polynomial Extrapolation~\cite{Article:Scieur_MP20} } \label{alg:algex}
\STATE \textbf{Input:} sequence $\{ g_s \in \mathbb R^d \}_{s=0}^{s=r-1}$, parameter $\lambda > 0$.
\STATE Compute matrix  $U = [ g_1 - g_0, \dots, g_{r} - g_{r-1}] \in \mathbb R^{d \times r}$.
\STATE Obtain $z$ by solving $(U^\top U + \lambda I ) z = \mathbf{1}$.
\STATE Get $c= z / (z^\top \mathbf{1})$.
\STATE \textbf{Output:} $\Sigma_{i=0}^{r-1} c_i g_i$, the approximation of the fixed point $g^*$.
\end{algorithmic}
\end{algorithm}
 \end{minipage}
 \par
 }
 \vspace{0.2in}
 
Specifically, at iteration $t$, $m_t$ is obtained by \textsf{(a)} calling Algorithm~\ref{alg:algex} with a sequence of $r$ past gradients, $\{ g_{t-1},g_{t-2}, \dots, g_{t-r} \}$ as input yielding the vector $c = [c_0, \dots, c_{r-1}] $ and \textsf{(b)} setting $m_t:= \Sigma_{i=0}^{r-1} c_i g_{t-r+i}$.
To understand why the output from the extrapolation method may be a reasonable estimation, assume that the update converges to a stationary point (i.e., $g^*:=\nabla f(w^*) = 0$ for the underlying function $f$). Then, we might rewrite (\ref{nox}) as $g_t = A g_{t-1}  + \mathcal{O}( \| g_{t-1} \|_2^2 ) u_{t-1}$, for some unit vector $u_{t-1}$.
This equation suggests that the next gradient vector $g_{t}$ is a linear transform of $g_{{t-1}}$ plus an error vector that may not be in the span of $A$.
If the algorithm converges to a stationary point, the magnitude of the error will converge to zero. 
We note that prior known gradient prediction methods are mainly designed for convex functions.
Algorithm~\ref{alg:algex} is used in our following numerical applications given its empirical success in Deep Learning, see~\cite{scieur2018nonlinear}, yet, any gradient prediction method can be embedded in our \textsc{OPT-AMSGrad} framework.
The search for the optimal prediction process in order to accelerate \textsc{OPT-AMSGrad} is an interesting research direction.

\vspace{0.1in}
\textbf{Computational cost:}
 This extrapolation step consists in: \textsf{(a)} Constructing the linear system $(U^\top U)$ which cost can be optimized to $\mathcal{O}(d)$, since the matrix $U$ only changes one column at a time. \textsf{(b)} Solving the linear system which cost is $\mathcal{O}(r^3)$, and is negligible for a small $r$ used in practice.\textsf{ (c)} Outputting a weighted average of previous gradients which cost is $\mathcal{O}(r \times d)$ yielding a computational overhead of $\mathcal O\left((r+1)d+r^3\right)$.
Yet, steps \textsf{(a)} and \textsf{(c)} are parallelizable in the final implementation.

\subsection{Classification Experiments}

\textbf{Methods.}
We consider two baselines. The first one is the original \textsc{AMSGrad}. 
The hyper-parameters are set to be $\beta_1 = 0.9$ and $\beta_2 = 0.999$, see~\cite{Proc:Reddi_ICLR18}. 
The other benchmark method is the \textsc{Optimistic-Adam$+\hat{v}_t$}~\cite{Proc:Daskalakis_ICLR18}, which described Algorithm~\ref{OPT-DISZ}. 
We use cross-entropy loss, a mini-batch size of $128$ and tune the learning rates over a fine grid and report the best result for all methods.
For \textsc{OPT-AMSGrad}, we use $\beta_1 = 0.9$ and $\beta_2 = 0.999$ and the best step size $\eta$ of \textsc{AMSGrad} for a fair evaluation of the optimistic step. In our implementation, \textsc{OPT-AMSGrad} has an additional parameter $r$ that controls the number of previous gradients used for gradient prediction. 
We use $r=5$ past gradient for empirical reasons, see Section~\ref{sec:choicer}.
The algorithms are initialized at the same point and the results are averaged over 5 repetitions.

\vspace{0.1in}
\textbf{Datasets.}\hspace{0.1in}Following~\cite{Proc:Reddi_ICLR18,Proc:Kingma_ICLR15}, we compare different algorithms on \textit{MNIST}, \textit{CIFAR10},
\textit{CIFAR100}, and \textit{IMDB} datasets. 
For \textit{MNIST}, we use two noisy variants namely \textit{MNIST-back-rand} and \textit{MNIST-back-image} from~\cite{Proc:Larochelle_ICML07} (which was also heavily used in~\cite{Proc:ABC_UAI10} to evaluate tree algorithms). 
They both have $12\,000$ training samples and $50\,000$ test samples, where random background is inserted to the original \textit{MNIST} hand-written digit images. 
For \textit{MNIST-back-rand}, each image is inserted with a random background, which pixel values are generated uniformly from 0 to 255, while \textit{MNIST-back-image} takes random patches from a black and white noisy background.
The input dimension is $784$ ($28\times 28$) and the number of classes is $10$. 
\textit{CIFAR10} and \textit{CIFAR100} are popular computer-vision datasets of $50\,000$ training images and $10\,000$ test images, of size $32\times 32$. 
The \textit{IMDB} movie review dataset is a binary classification dataset with $25\,000$ training and testing samples respectively. 
It is a popular dataset~for~text~classification.

\vspace{0.1in}
\textbf{Network architectures.} We adopt a multi-layer fully connected neural network with hidden layers of $200$ connected to another layer with $100$ neurons (using \textrm{ReLU} activations and \textrm{Softmax} output). 
This network is tested on \textit{MNIST} variants.
For convolutional networks, we adopt a simple four layer CNN which has 2 convolutional layers following by a fully connected layer. In addition, we also apply residual networks, Resnet-18 and Resnet-50~\cite{Proc:He_CVPR16}, which have achieved state-of-the-art results.
For the texture \textit{IMDB} dataset, we consider a Long-Short Term Memory (LSTM) network~\cite{Article:Gers_NC00}.
The latter network includes a word embedding layer with $5\,000$ input entries representing most frequent words embedded into a $32$ dimensional space. 
The output of the embedding layer is passed to $100$ LSTM units then connected to $100$ fully connected~\textrm{ReLU}~layers.

\begin{figure}[h]
\vspace{-0.1in}
\centering
\mbox{%\hspace{-0.15in}
\includegraphics[width=2.1in]{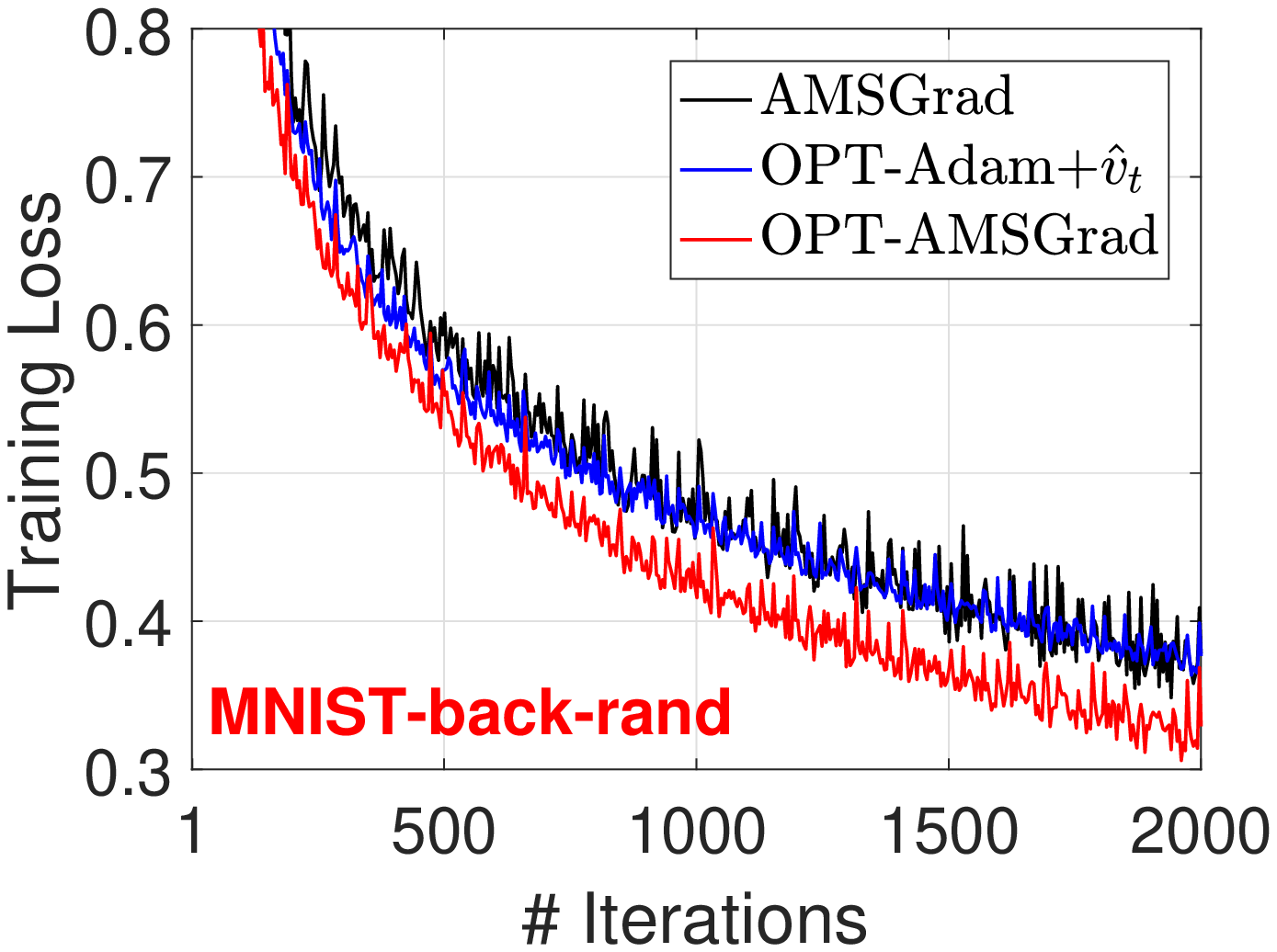}\hspace{-0.1in}
\includegraphics[width=2.1in]{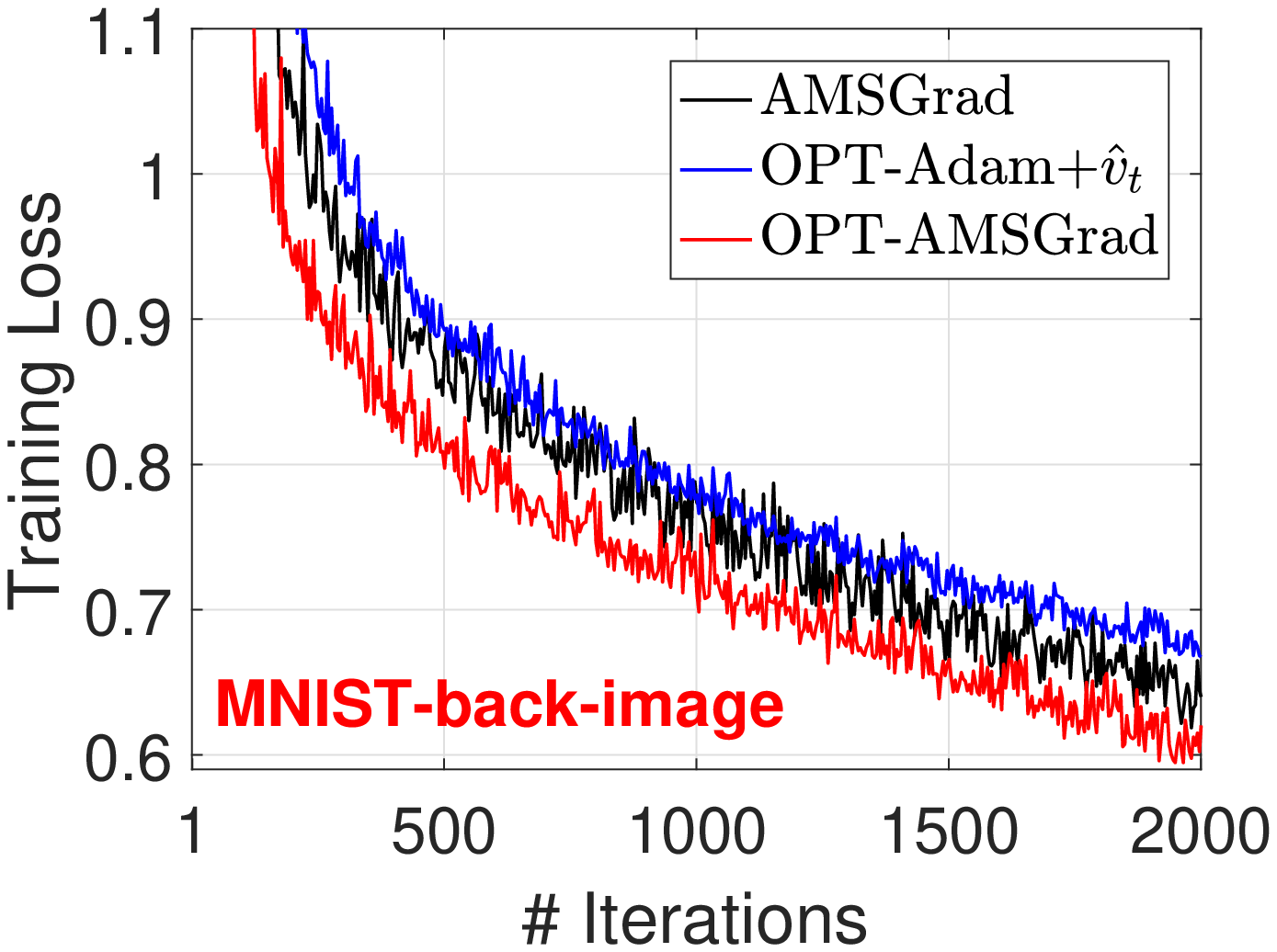}\hspace{-0.1in}
\includegraphics[width=2.1in]{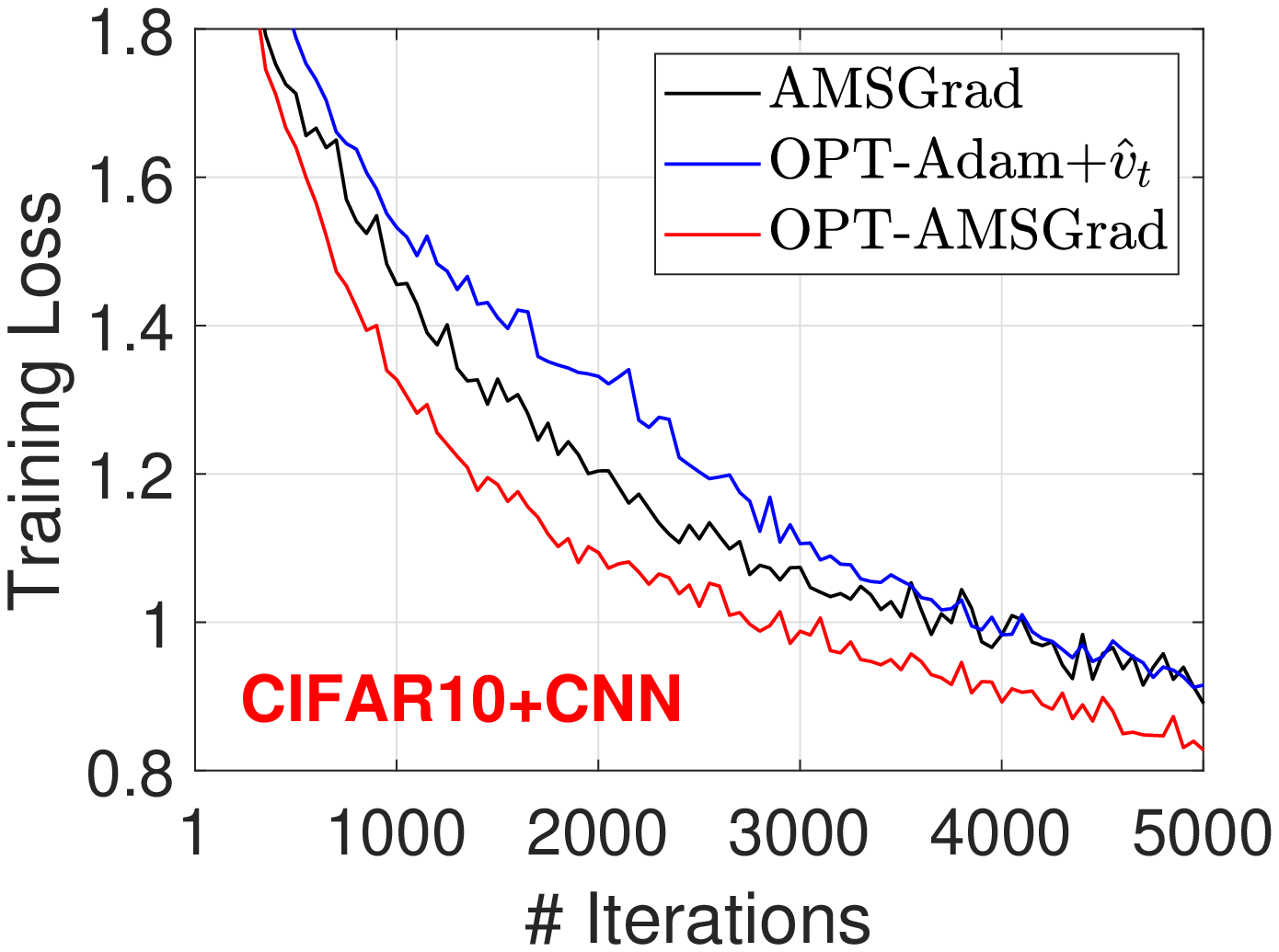}
}

\mbox{%\hspace{-0.15in}
\includegraphics[width=2.1in]{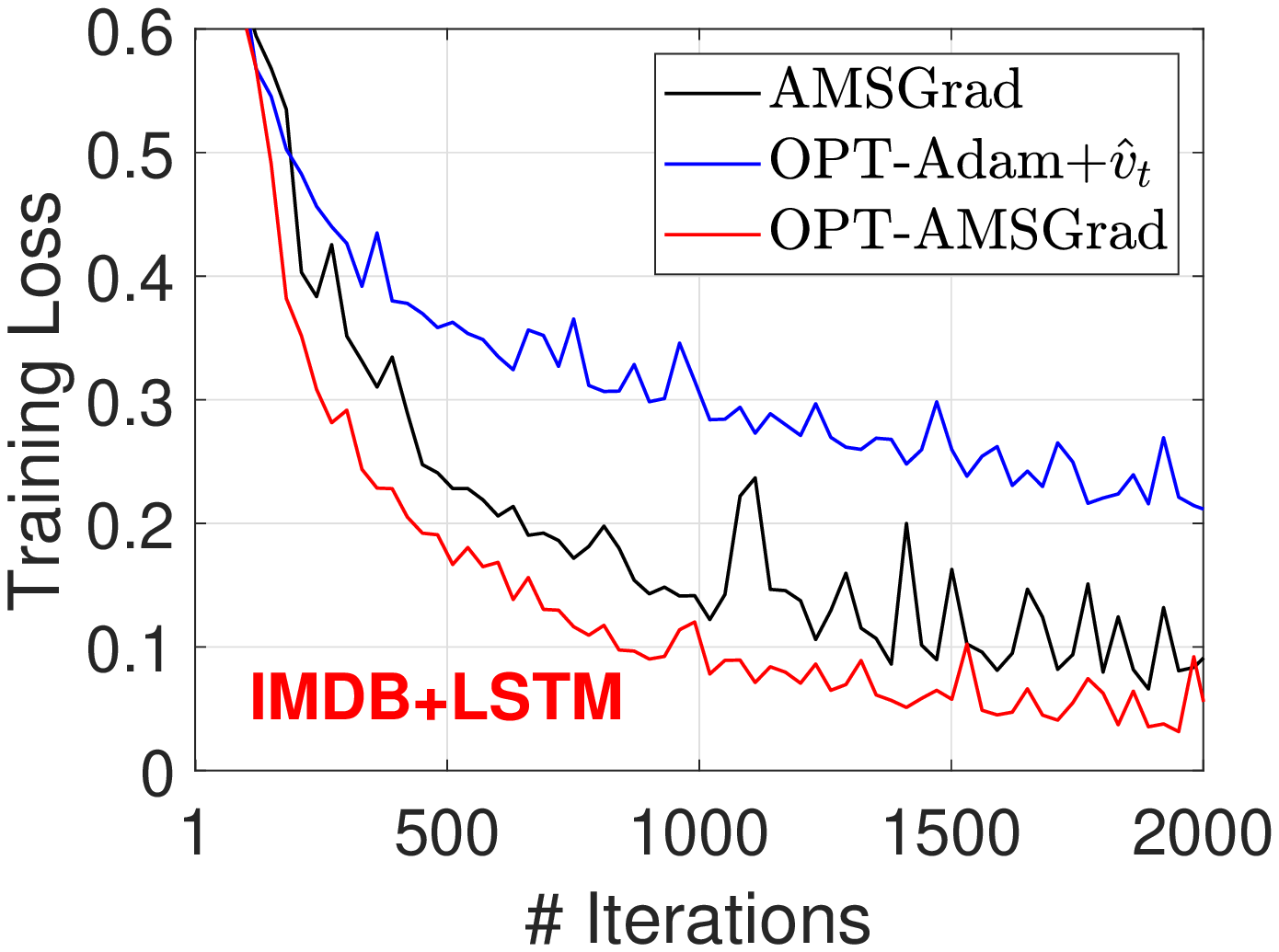}\hspace{-0.1in}
\includegraphics[width=2.1in]{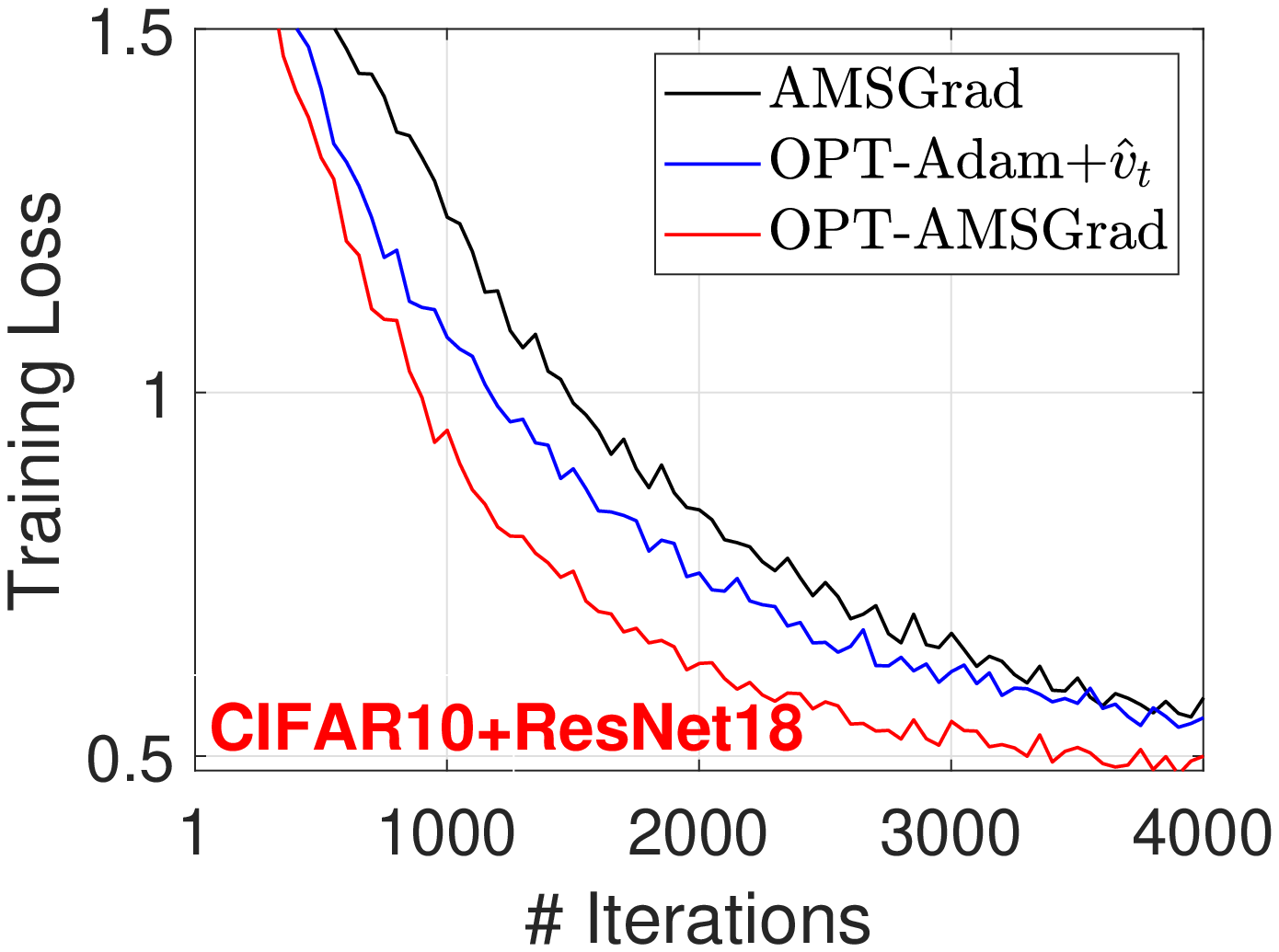}\hspace{-0.1in}
\includegraphics[width=2.1in]{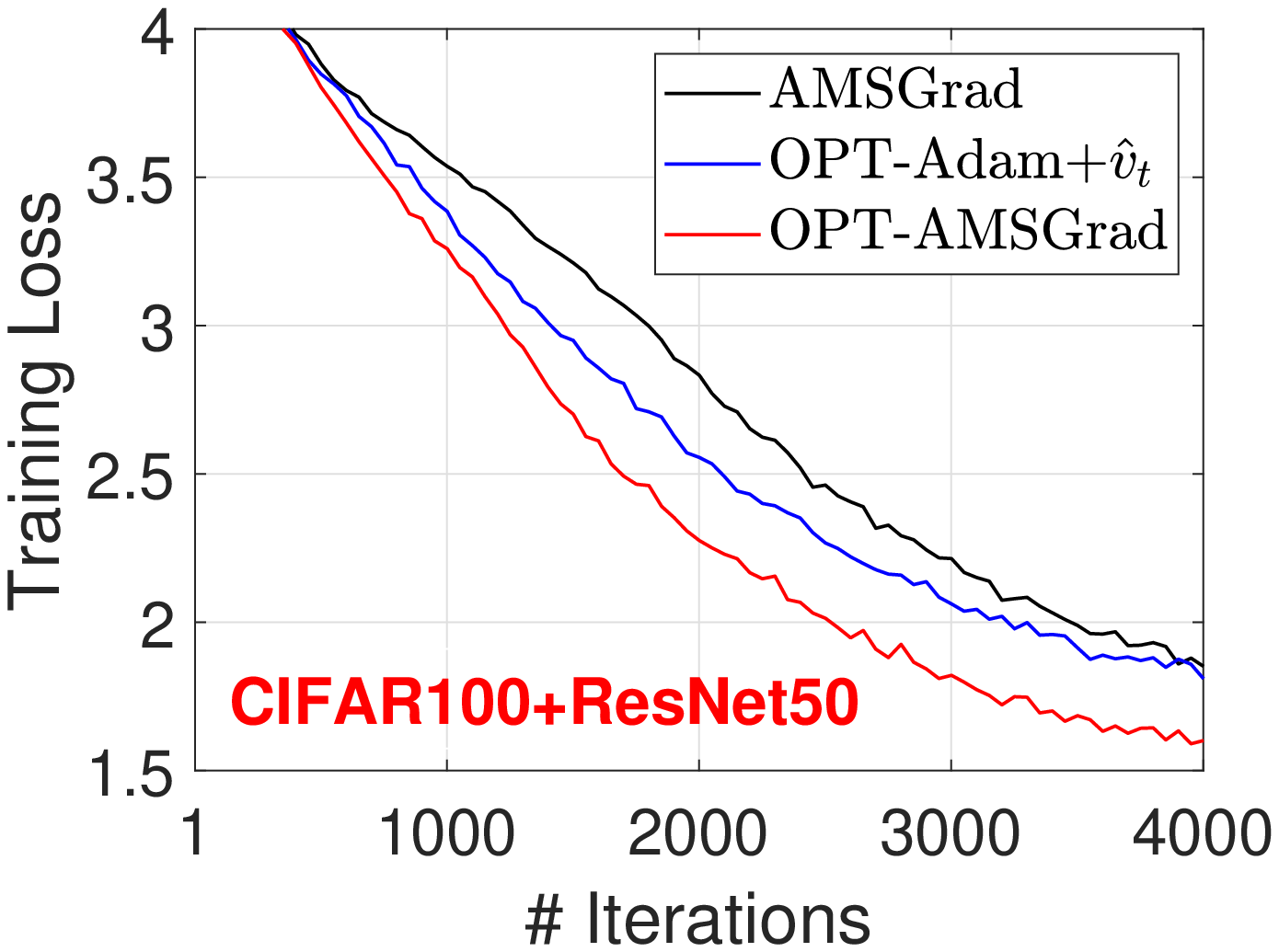}
}
\vspace{-0.1in}
\caption{Training loss vs. number of iterations for fully connected NN, CNN, LSTM and ResNet.}
\label{fig:train_loss}\vspace{-0.4in}
\end{figure}

\textbf{Results.} 
Firstly, to illustrate the acceleration effect of \textsc{OPT-AMSGrad} at early stage, we provide the training loss against number of iterations in Figure~\ref{fig:train_loss}. Clearly, on all datasets, the proposed \textsc{OPT-AMSGrad} converges faster than the other competing methods since fewer iterations are required to achieve the same precision, validating one of the main edges of \textsc{OPT-AMSGrad}.
We are also curious about the long-term performance and generalization of the proposed method in test phase.
In Figure~\ref{fig:testandtrain}, we plot the results when the model is trained until the test accuracy stabilizes. 
We observe: \textsf{(1)} in the long term, \textsc{OPT-AMSGrad} algorithm may converge to a better point with smaller objective function value, and \textsf{(2)} in these three applications, \textsc{OPT-AMSGrad} also outperforms the competing methods in terms of test~accuracy.

\begin{figure}[h]
\centering
\mbox{%\hspace{-0.1in}
\includegraphics[width=2.1in]{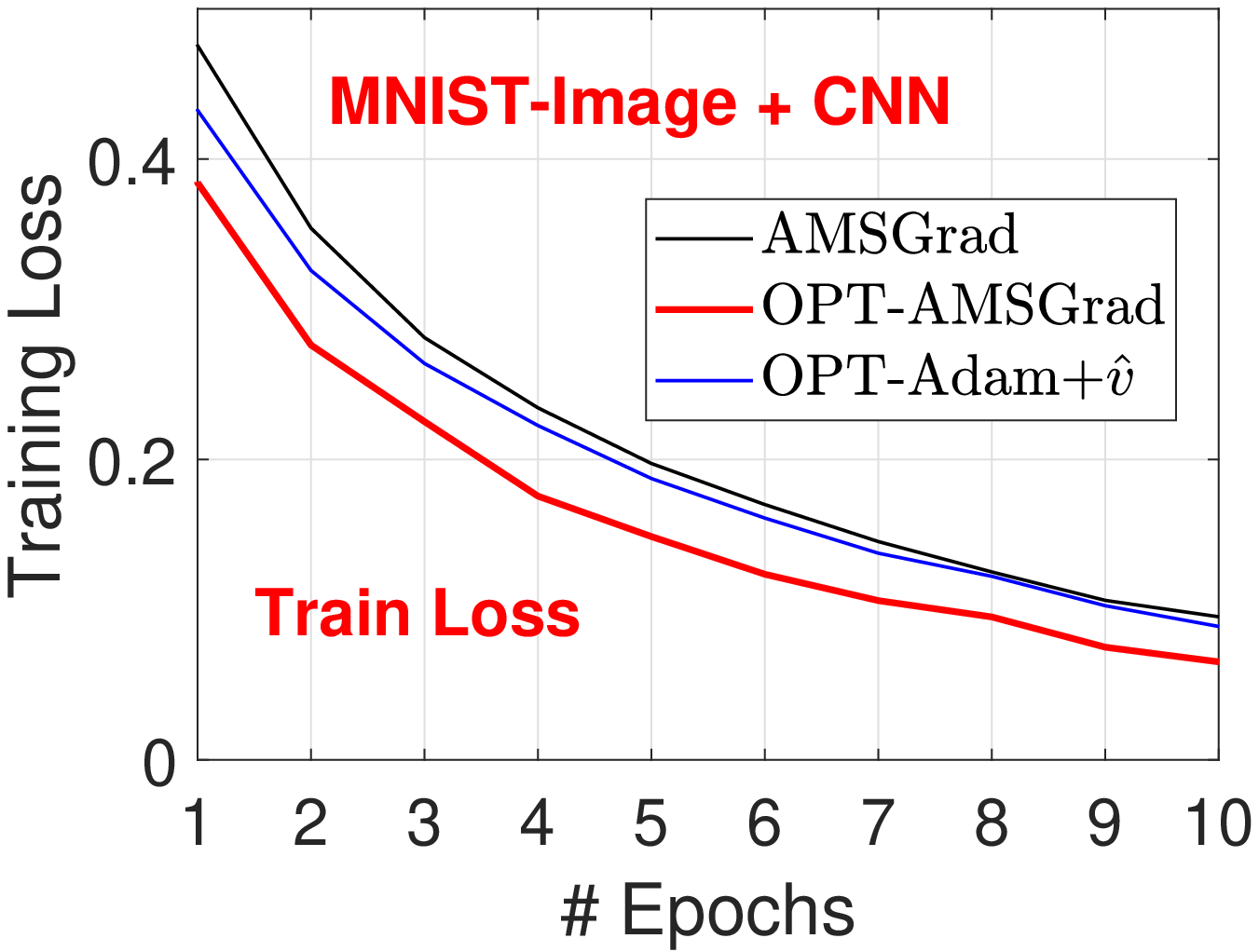}\hspace{-0.12in}
\includegraphics[width=2.1in]{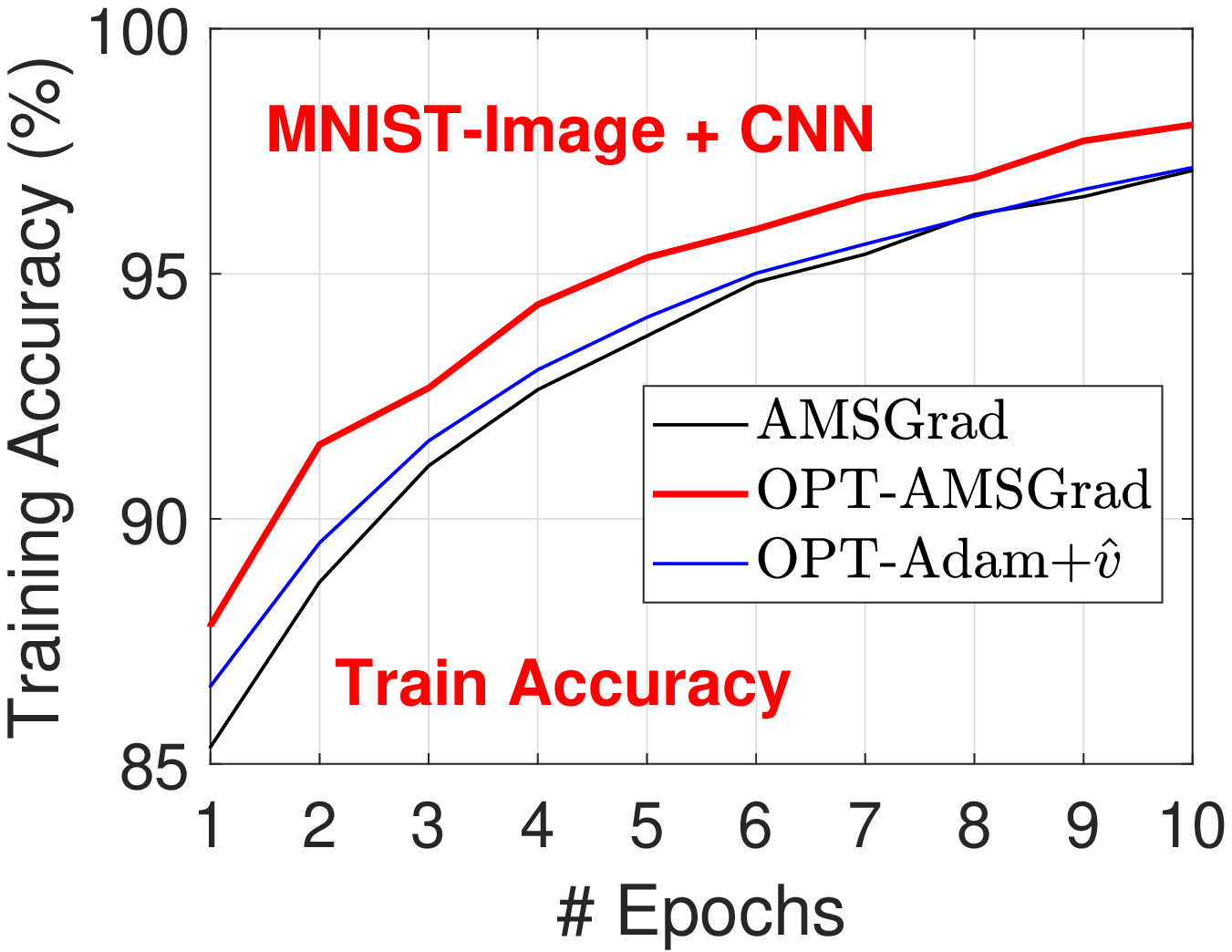}\hspace{-0.12in}
\includegraphics[width=2.1in]{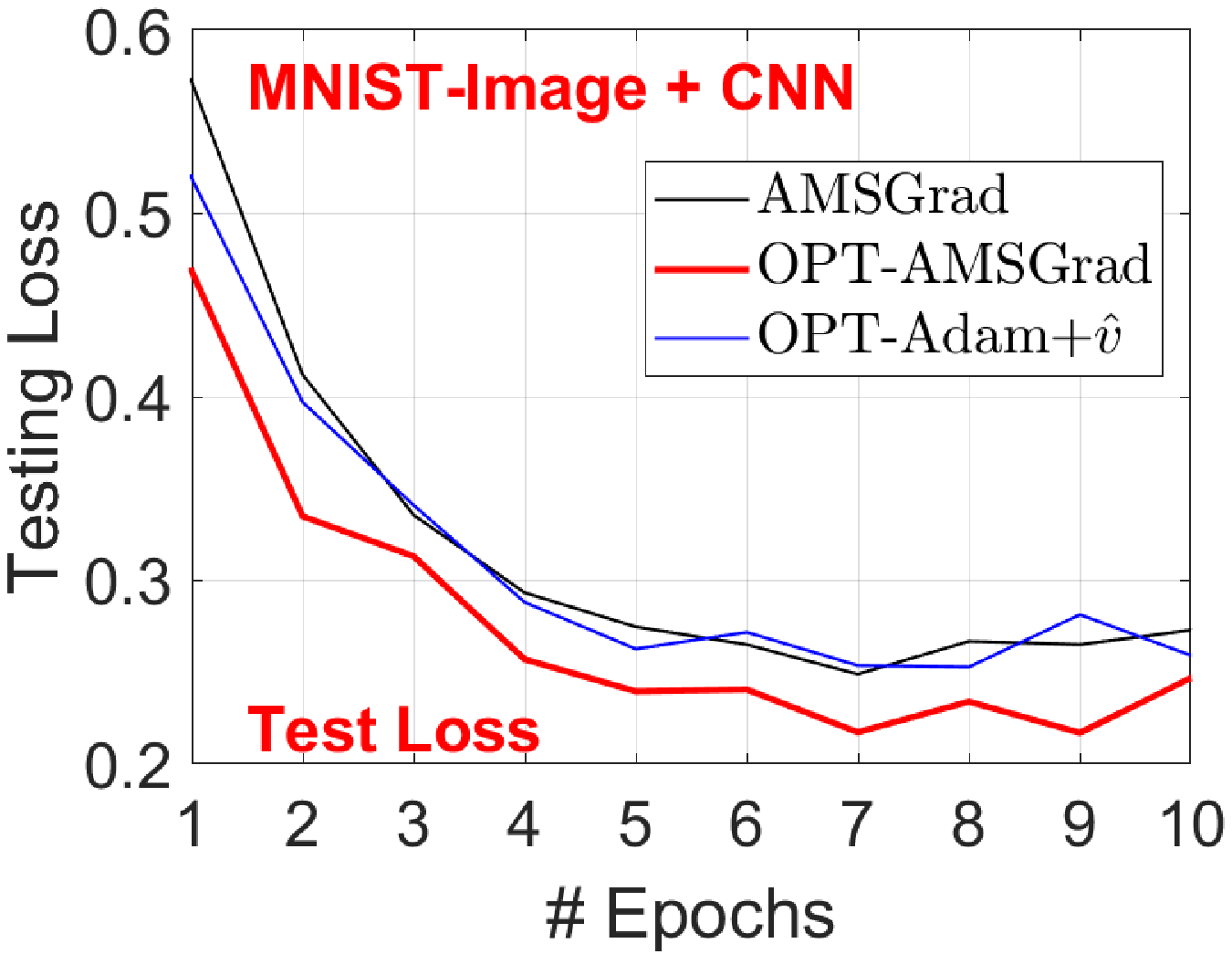}
}

\mbox{%\hspace{-0.1in}
\includegraphics[width=2.1in]{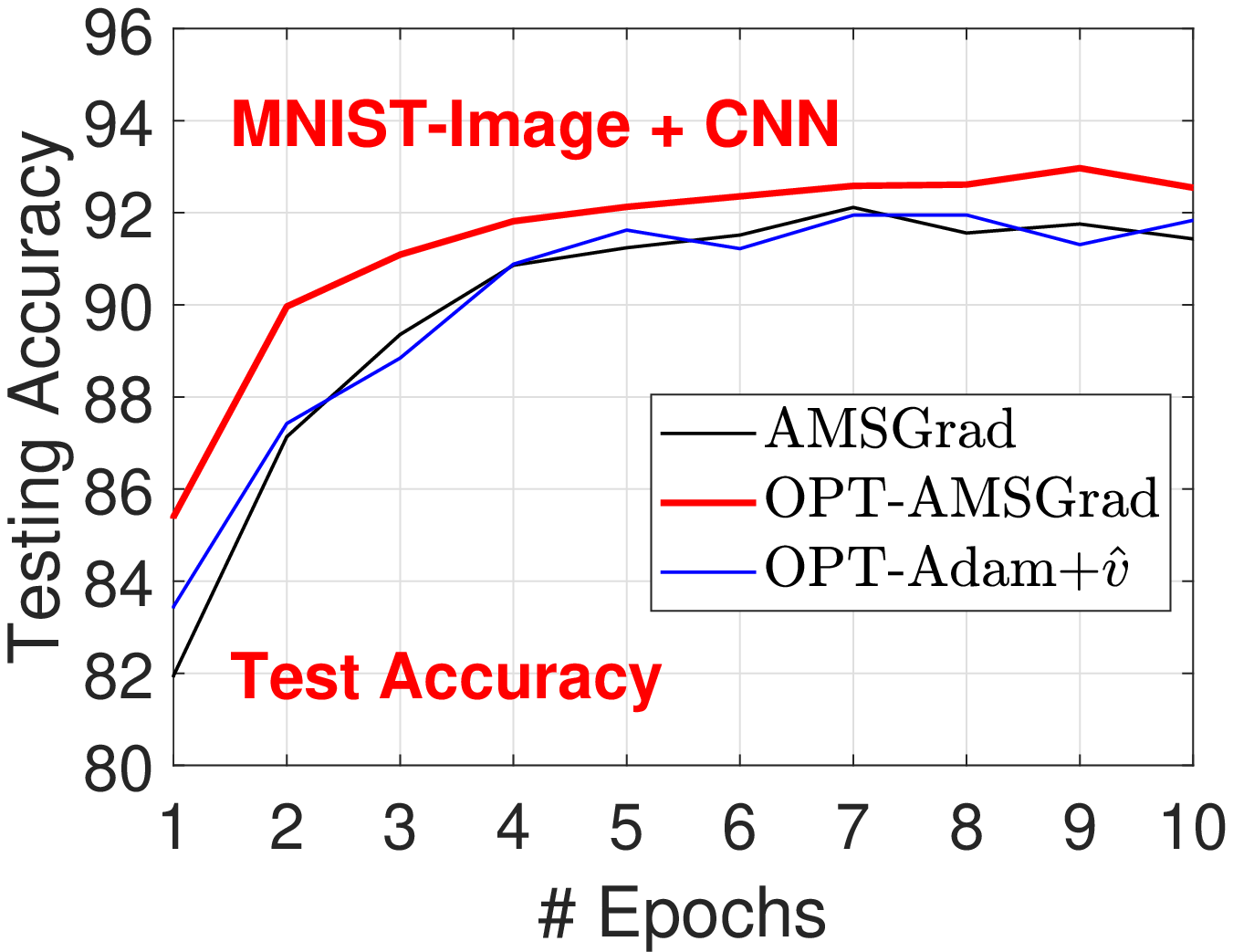}\hspace{-0.12in}
\includegraphics[width=2.1in]{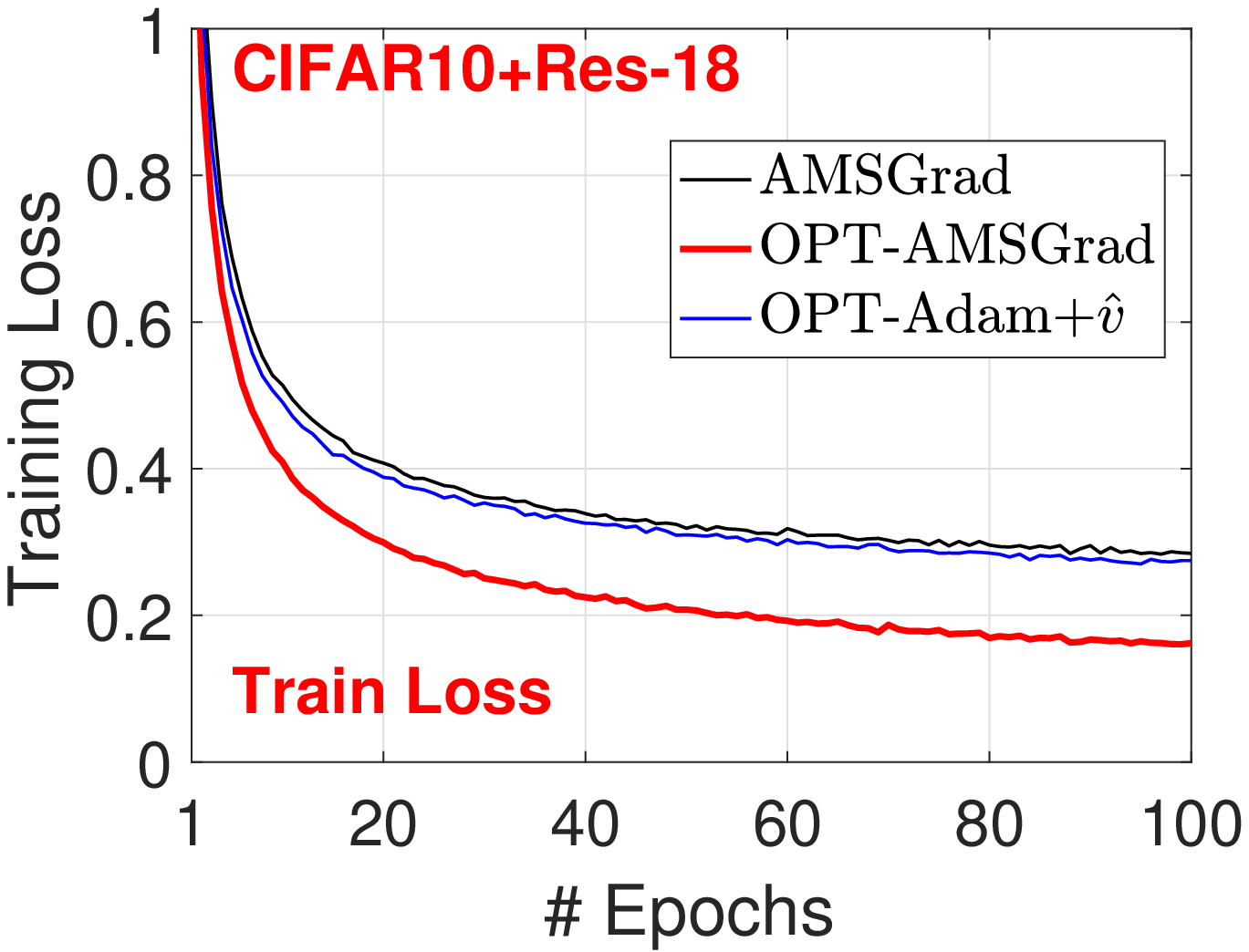}\hspace{-0.12in}
\includegraphics[width=2.1in]{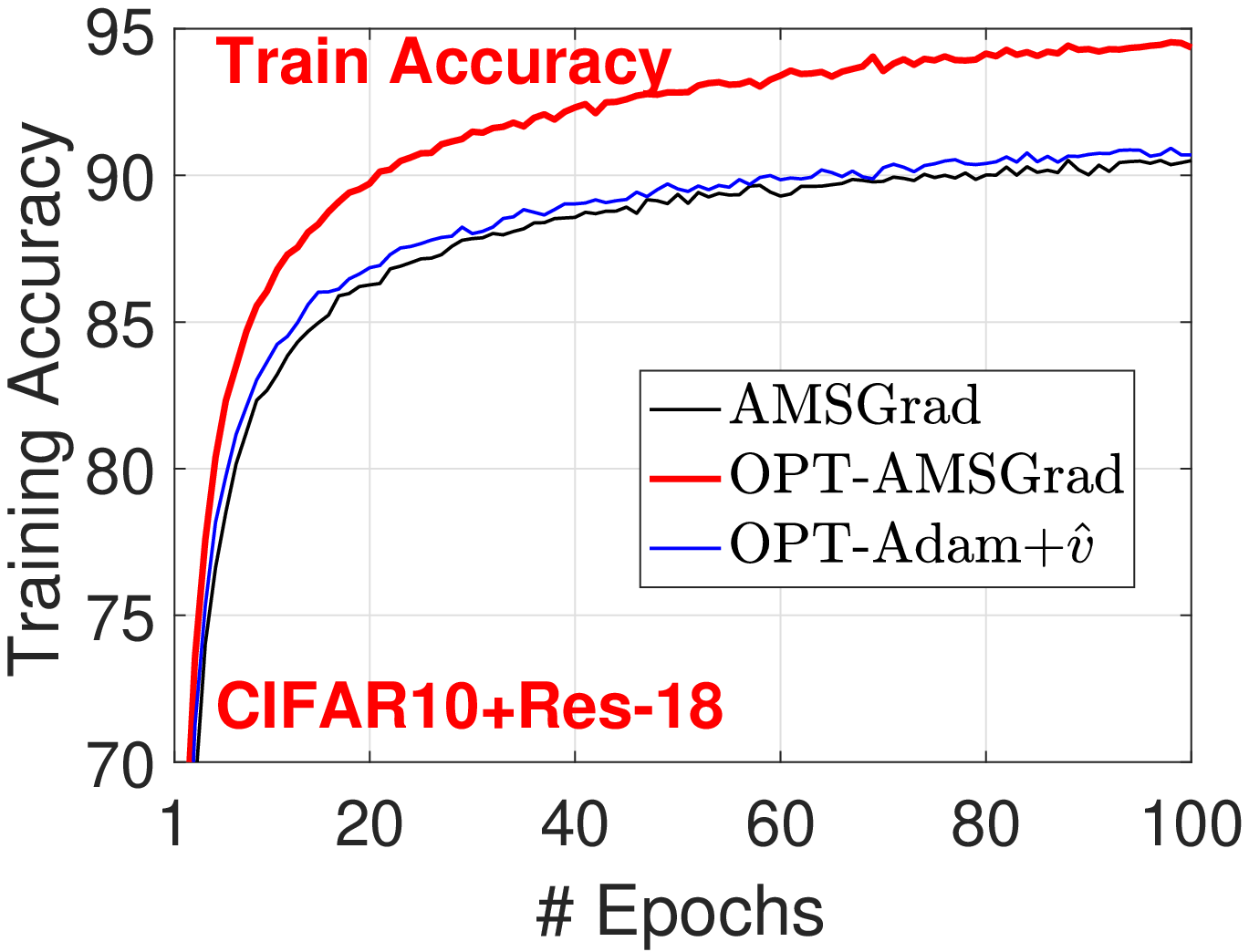}
}

\mbox{%\hspace{-0.1in}
\includegraphics[width=2.1in]{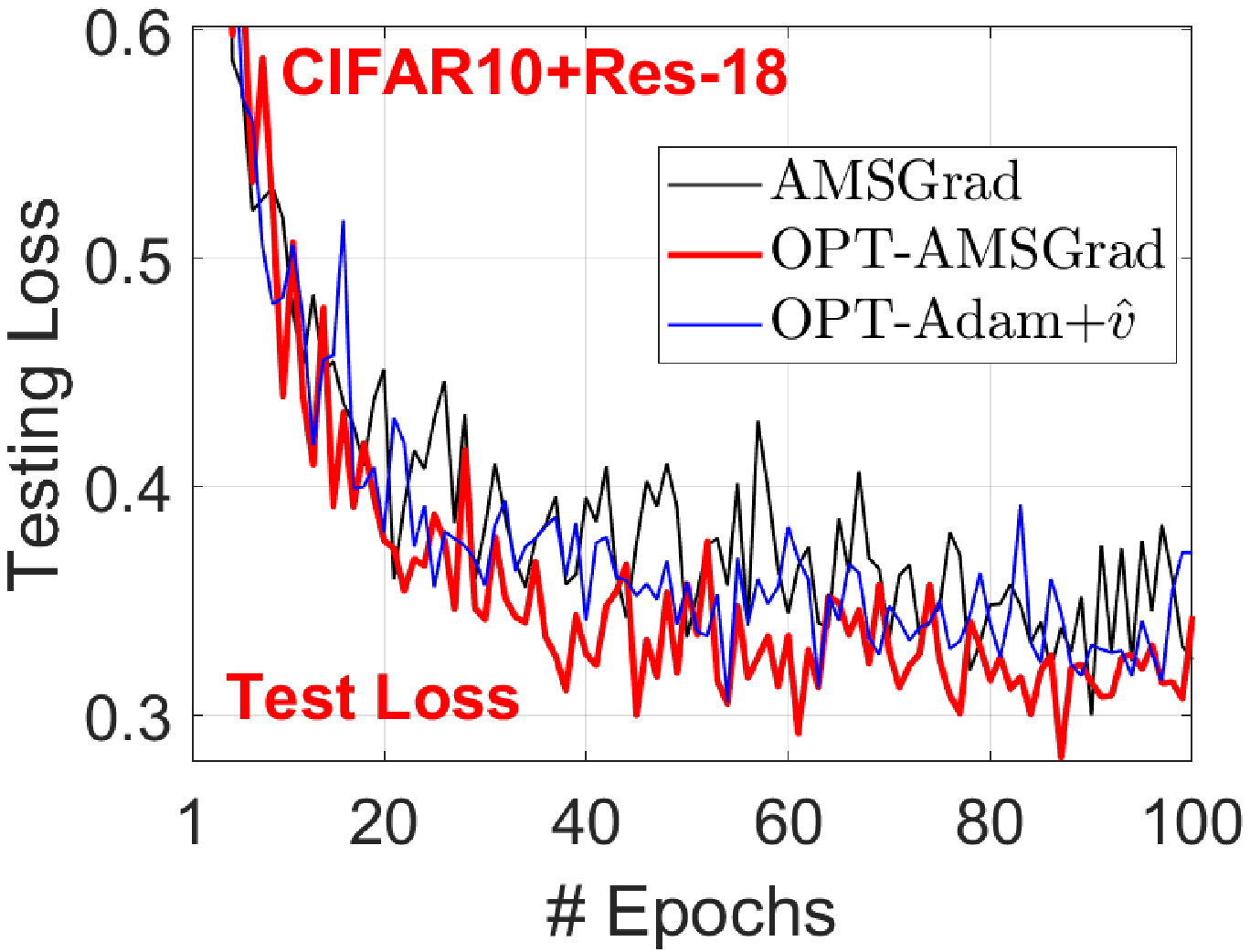}\hspace{-0.12in}
\includegraphics[width=2.1in]{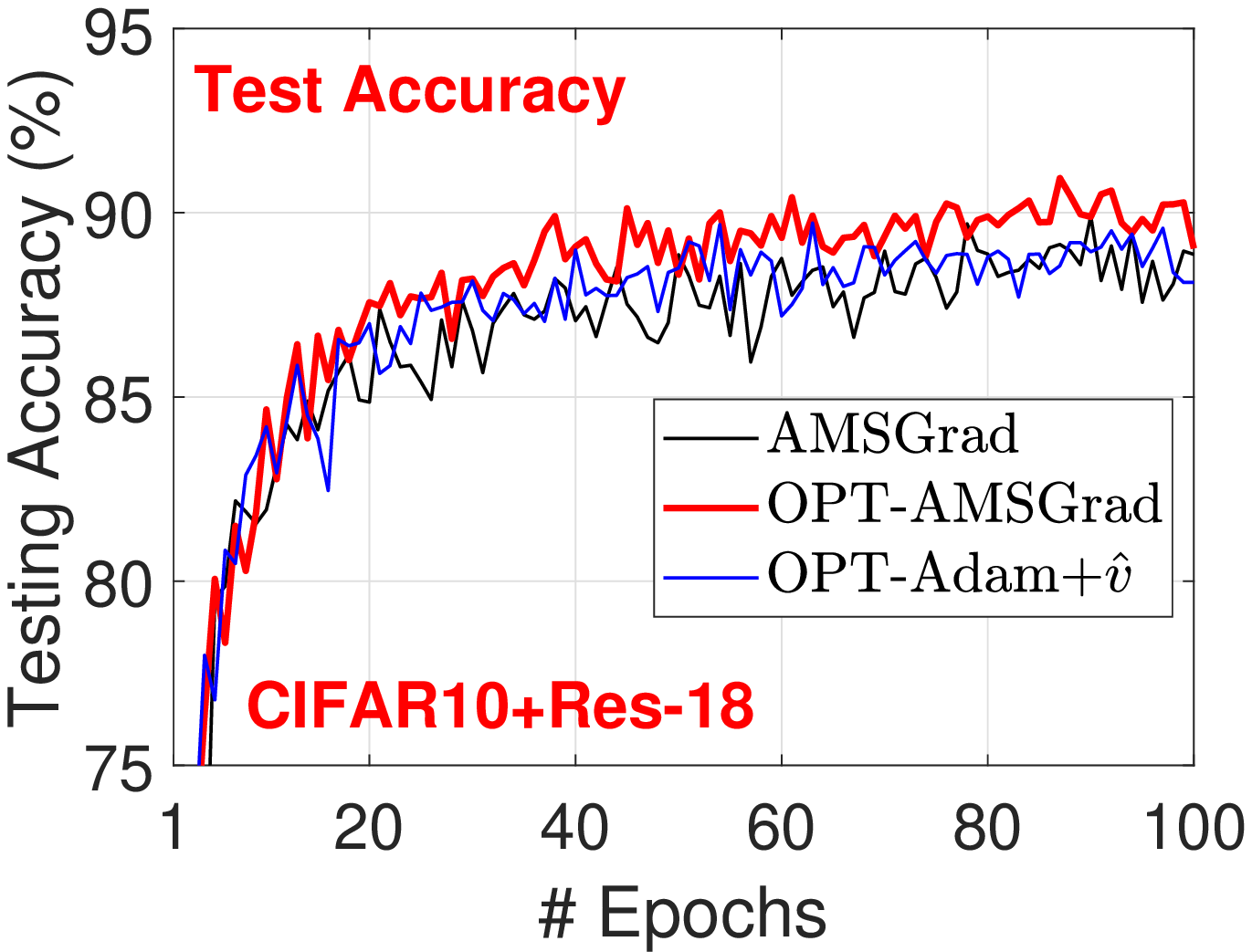}\hspace{-0.12in}
\includegraphics[width=2.1in]{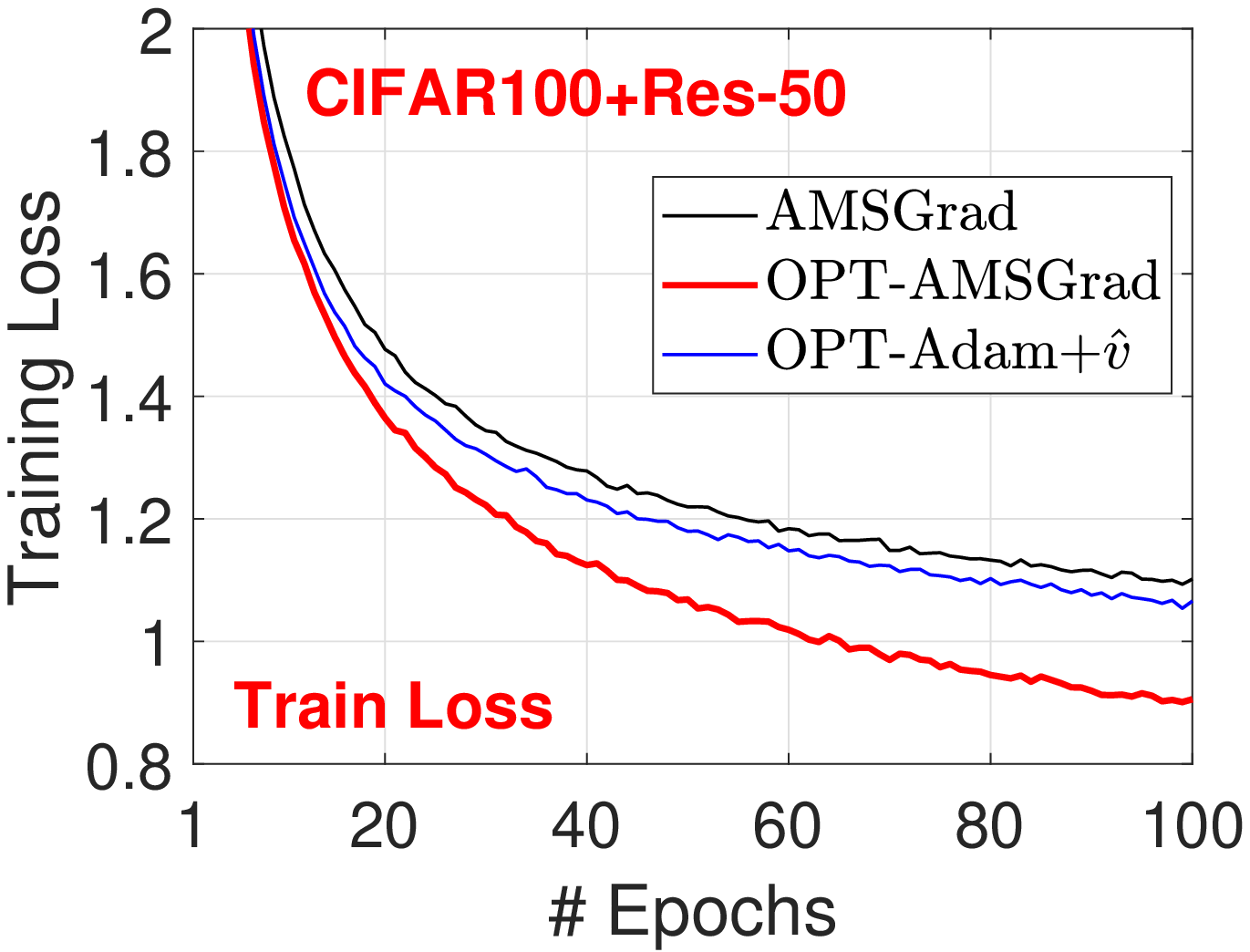}
}
\mbox{%\hspace{-0.1in}
\includegraphics[width=2.1in]{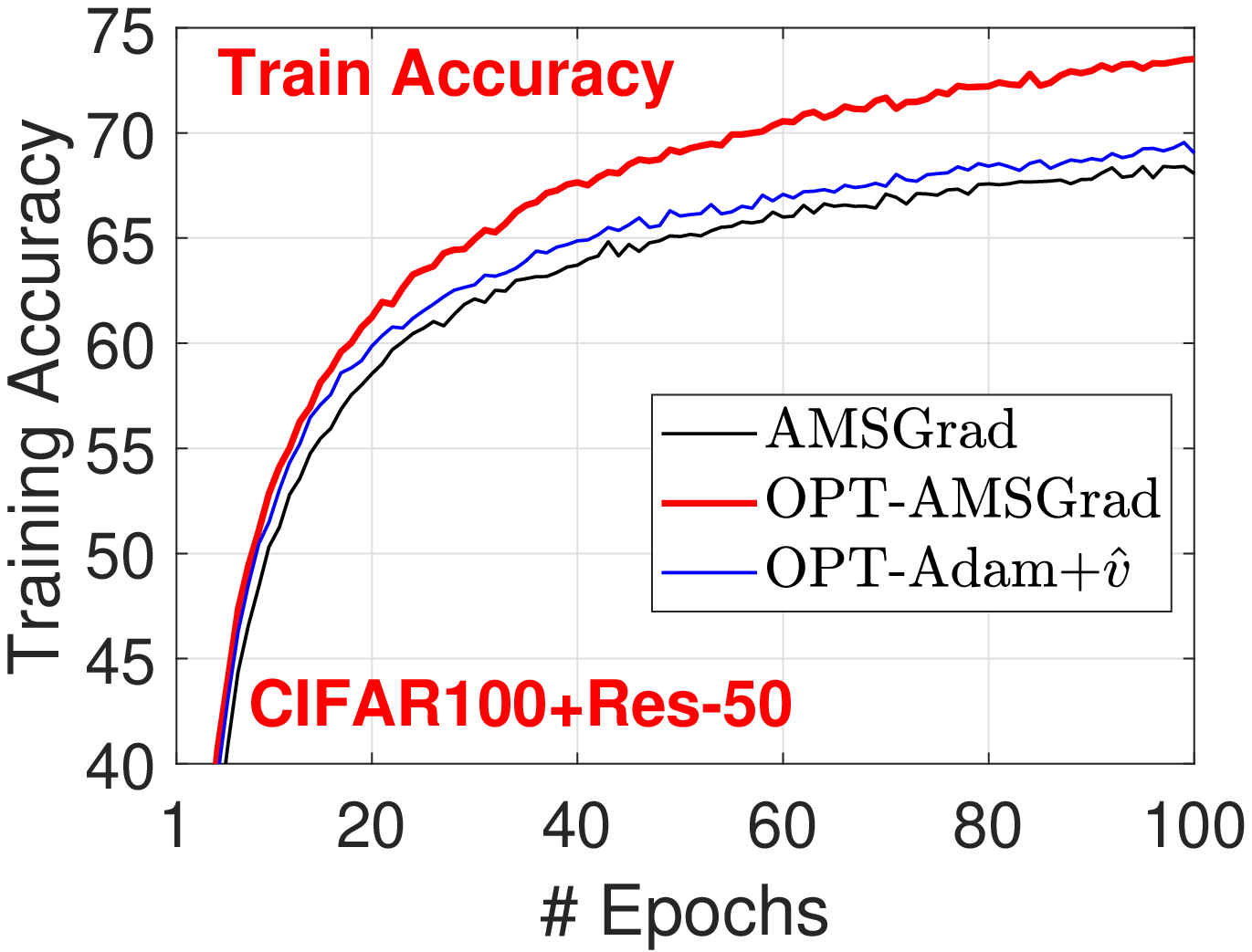}\hspace{-0.12in}
\includegraphics[width=2.1in]{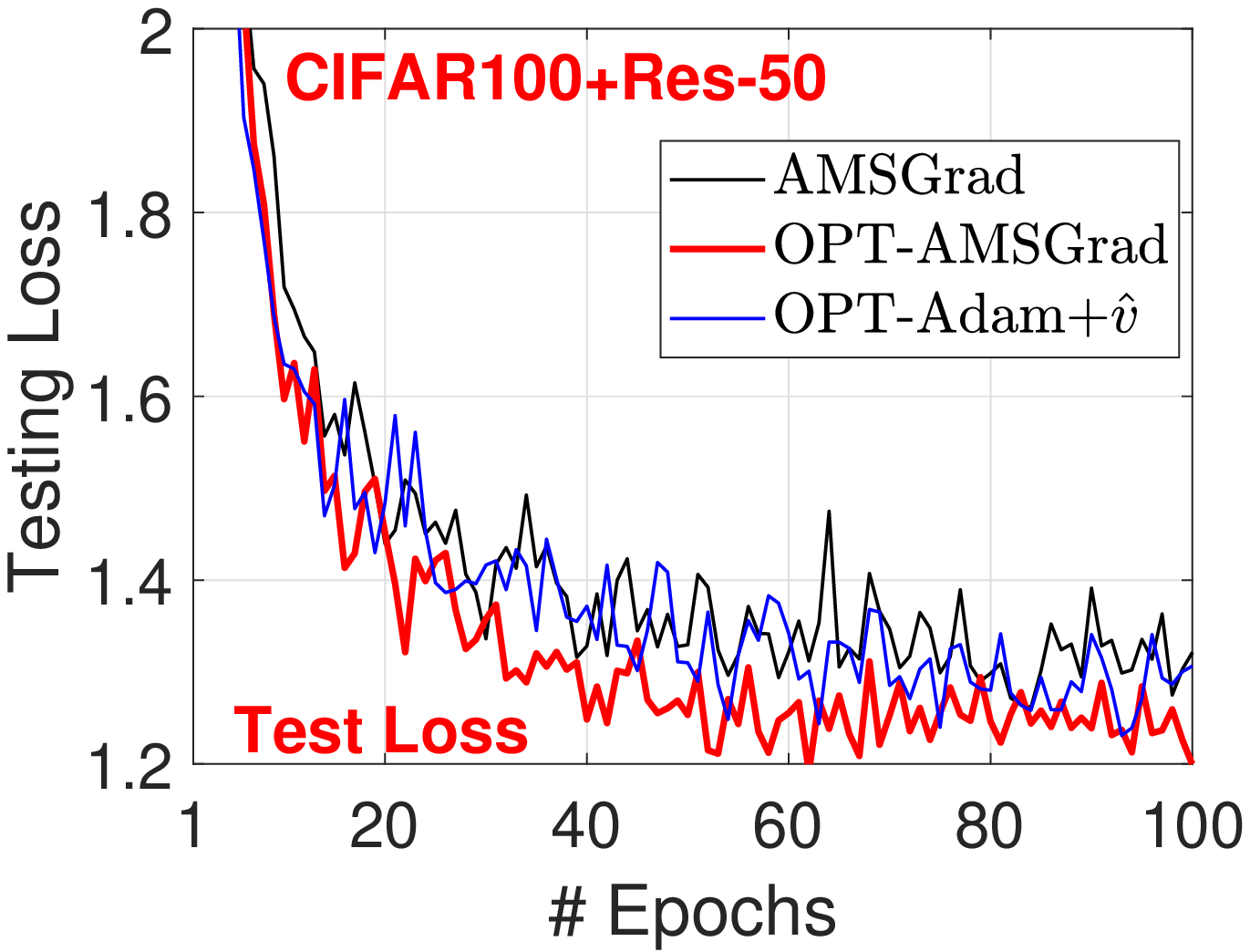}\hspace{-0.12in}
\includegraphics[width=2.1in]{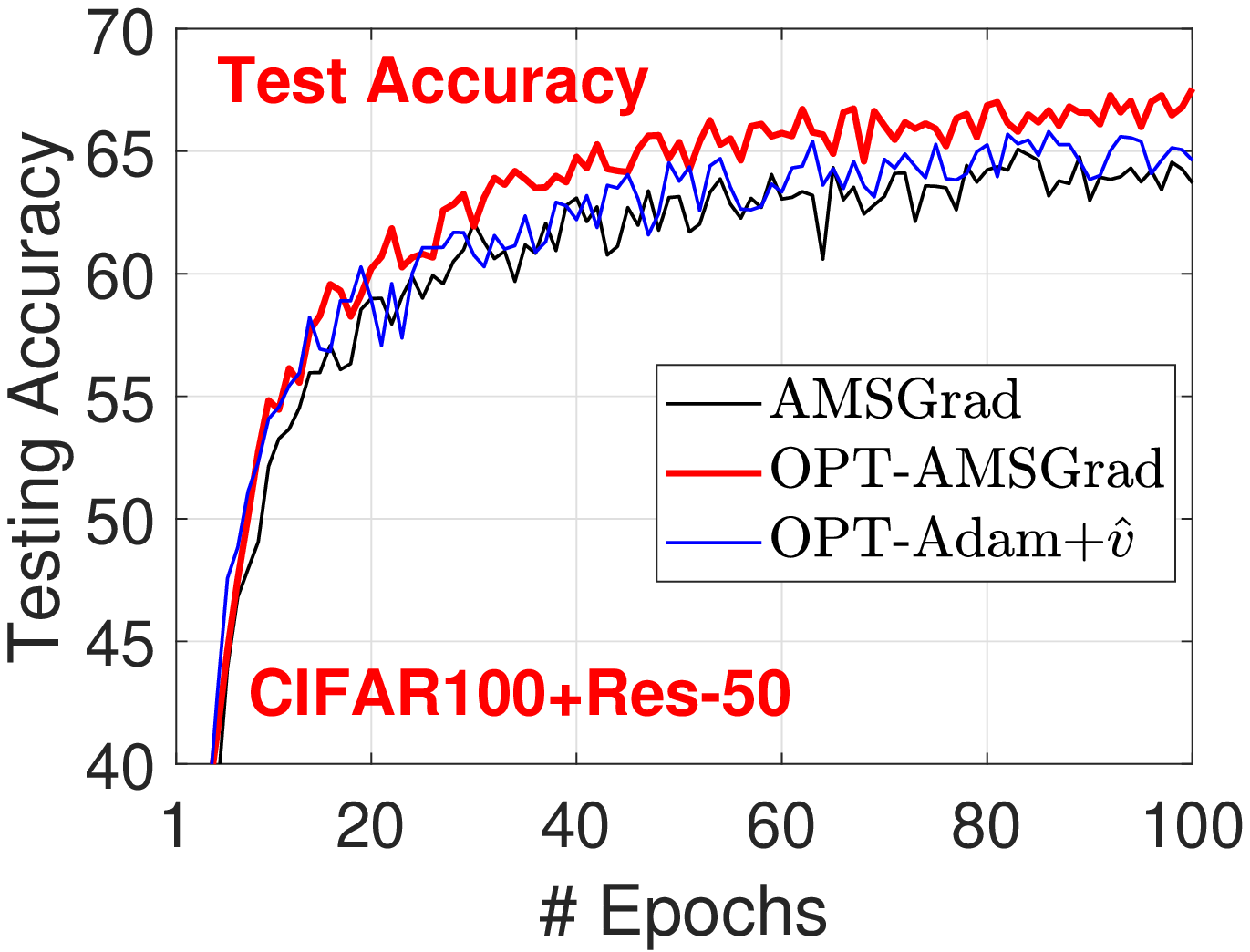}
}

\caption{\textit{MNIST-back-image} + CNN, \textit{CIFAR10} + Res-18 and \textit{CIFAR100} + Res-50 . We compare three methods in terms of training (cross-entropy) loss and accuracy, testing loss and accuracy.} \label{fig:testandtrain}%\vspace{-0.2in}
\end{figure}

\newpage

\subsection{Choice of parameter $r$}\label{sec:choicer}

Since the number of past gradients $r$ is important in gradient prediction (Algorithm~\ref{alg:algex}), we compare Figure~\ref{fig:compare} the performance under different values $r=3,5,10$ on two datasets. 
From the results we see that, taking into consideration both quality of gradient prediction and computational cost, $r=5$ is a good choice for most applications. 
We remark that, empirically, the performance comparison among $r=3,5,10$ is not absolutely consistent (i.e., more means better) in all cases. 
We suspect one possible reason is that for deep neural networks, the diversity of computed gradients through the iterations, due to the highly nonconvex loss, makes them inefficient for sequentially building the predictable process $\{m_t\}_{t>0}$. 
Thus, sometimes, the recent gradient vectors (e.g., $r\leq 5$) can be more informative. 
Yet, in some sense, this characteristic, very specific to deep neural networks, is itself a fundamental problem of gradient prediction methods.

\begin{figure}[h]
\centering
\mbox{\hspace{-0.15in}
\includegraphics[width=2.1in]{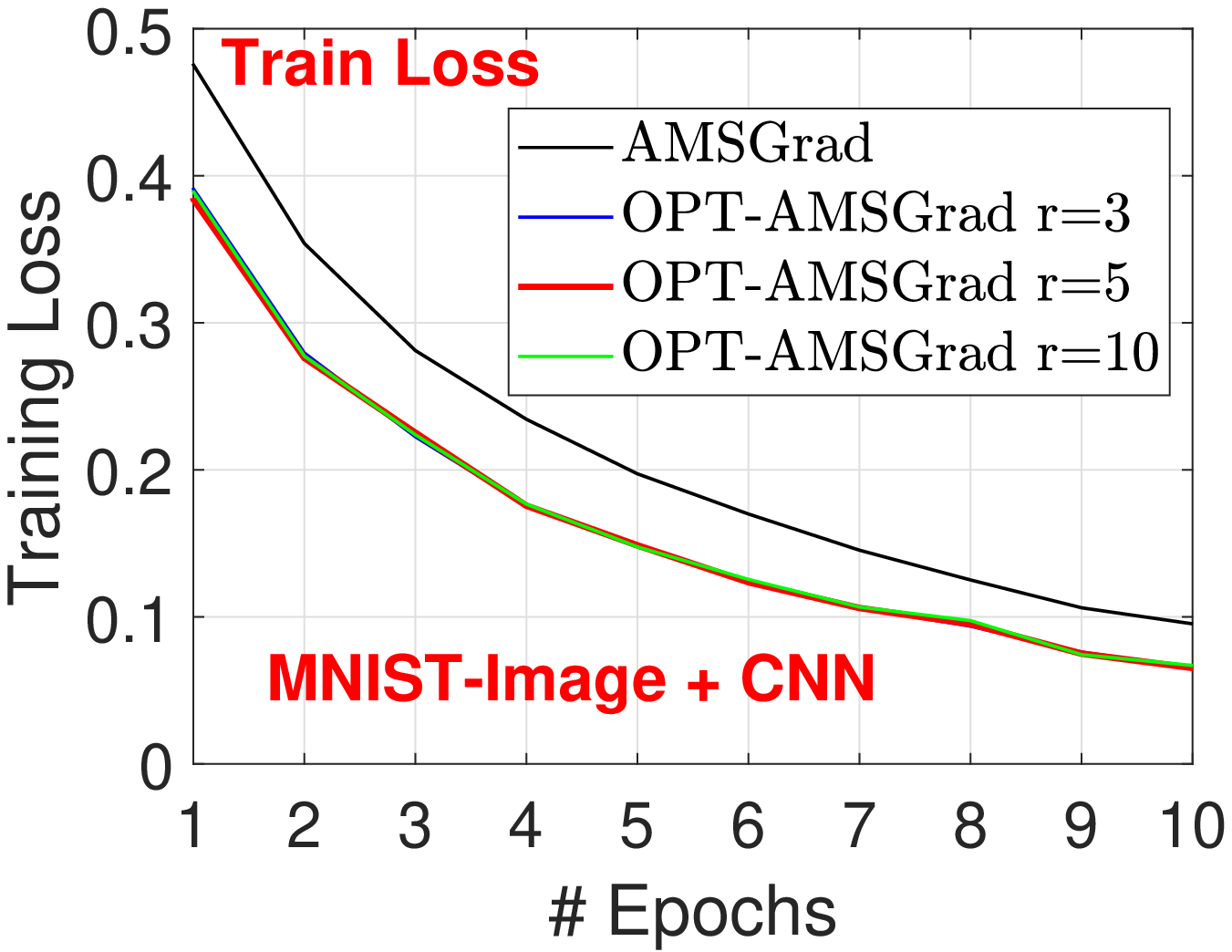}\hspace{-0.1in}
\includegraphics[width=2.1in]{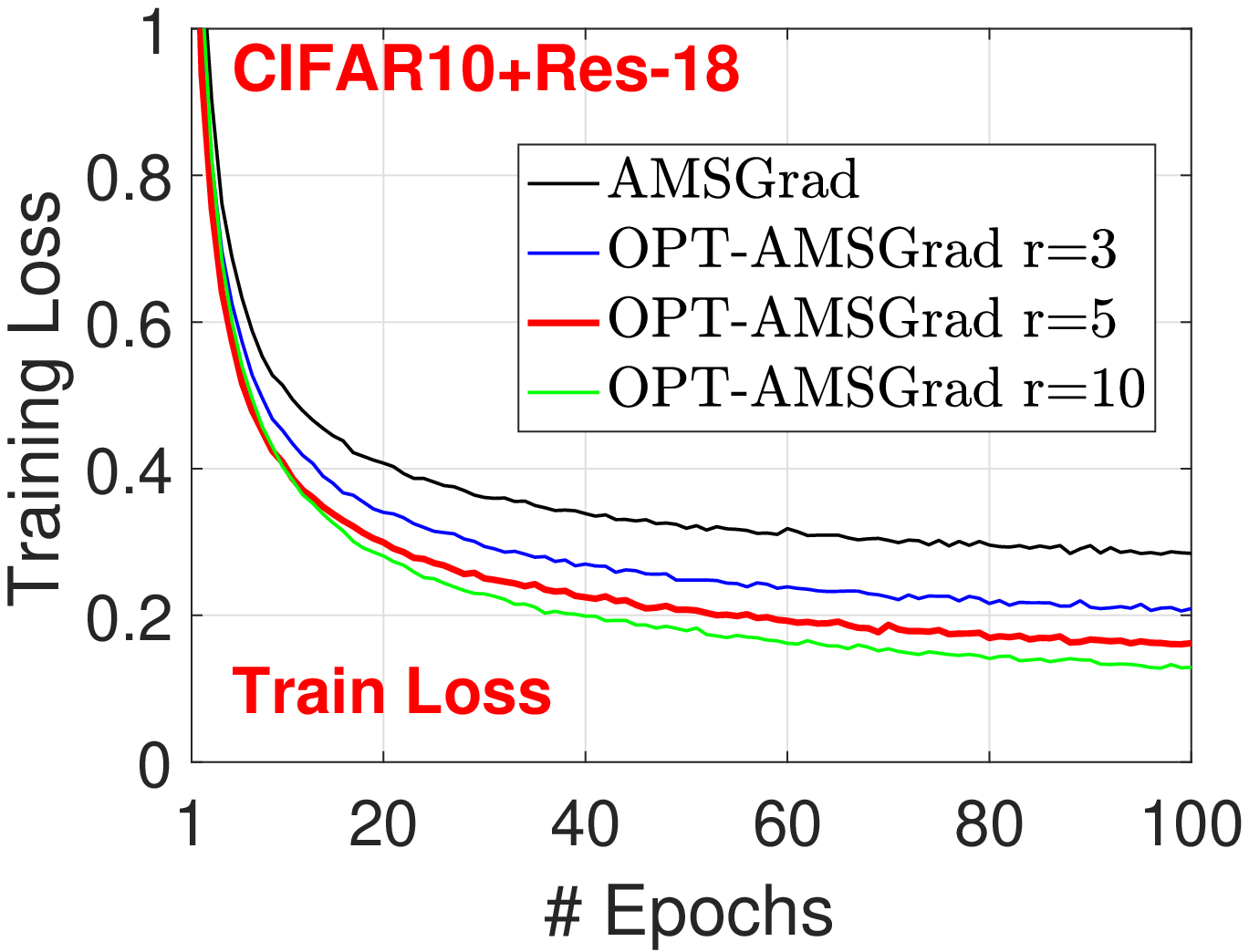}
}

\caption{Training loss w.r.t. different $r$ values.}\label{fig:compare}
\end{figure}

\section{Conclusion}

In this paper, we propose \textsc{OPT-AMSGrad}, which combines optimistic online learning and \textsc{AMSGrad} to improve sample efficiency and
accelerate the training process, in particular for deep neural networks. 
Given a good gradient prediction process, we demonstrate that the regret can be smaller than that of standard \textsc{AMSGrad}.
We also establish finite-time convergence bound on the second order moment of the gradient of the objective function matching that of state-of-the-art algorithms.
Experiments on various deep learning problems demonstrate the effectiveness of the proposed algorithm in accelerating the empirical risk minimization procedure and empirically show better generalization properties of \textsc{OPT-AMSGrad}.

\newpage\clearpage

\bibliographystyle{plain}
\bibliography{standard}

\clearpage

\appendix

\section{Additional Remarks on the Gradient Prediction Process}

\textbf{Two illustrative examples.}\hspace{0.1in} We provide two toy examples to demonstrate how \textsc{OPT-AMSGrad} works with the chosen extrapolation method. 
First, consider minimizing a quadratic function $H(w) := \frac{b}{2} w^2 $ with vanilla gradient descent method $w_{t+1} = w_t - \eta_t \nabla H(w_t)$. 
The gradient $g_{t}:= \nabla H(w_{t})$ can be recursively expressed as  $g_{t+1} = b w_{t+1} = b ( w_t  - \eta_t g_t ) = g_t - b \eta_t g_t  $.
Thus, the update can be written in the form of 
$$g_t = A g_{t-1}  + \mathcal{O}( \| g_{t-1} \|_2^2 ) u_{t-1}\, ,$$
where $A = (1 - b \eta)$ and $u_{t-1}=0$ by setting $\eta_t=\eta$ (constant step size).
Specifically, consider optimizing $H(w) := w^2/2 $ by the following three algorithms with the same step size.
One is Gradient Descent (GD): $w_{t+1} = w_t - \eta_t g_t$, while the other two are \textsc{OPT-AMSGrad} with $\beta_1=0$ and the second moment term $\hat{v}_t$ being dropped: $w_{t+\frac{1}{2}} = \Pi_{\Theta}\big[ w_{t-\frac{1}{2}} - \eta_t g_t \big]$, $w_{t+1} = \Pi_{\Theta}\big[ w_{t+\frac{1}{2}} - \eta_{t+1} m_{t+1} \big]$. 
We denote the algorithm that sets $m_{t+1}= g_t$ as \textsc{Opt-1}, and denote the algorithm that uses the extrapolation method to get $m_{t+1}$ as \textsc{Opt-extra}.
We let $\eta_t=0.1$ and the initial point $w_0=5$ for all three methods.

\begin{figure}[h]
\begin{center}
    \mbox{\hspace{-0.1in}
    \subfigure[]{
    \includegraphics[width=1.7in]{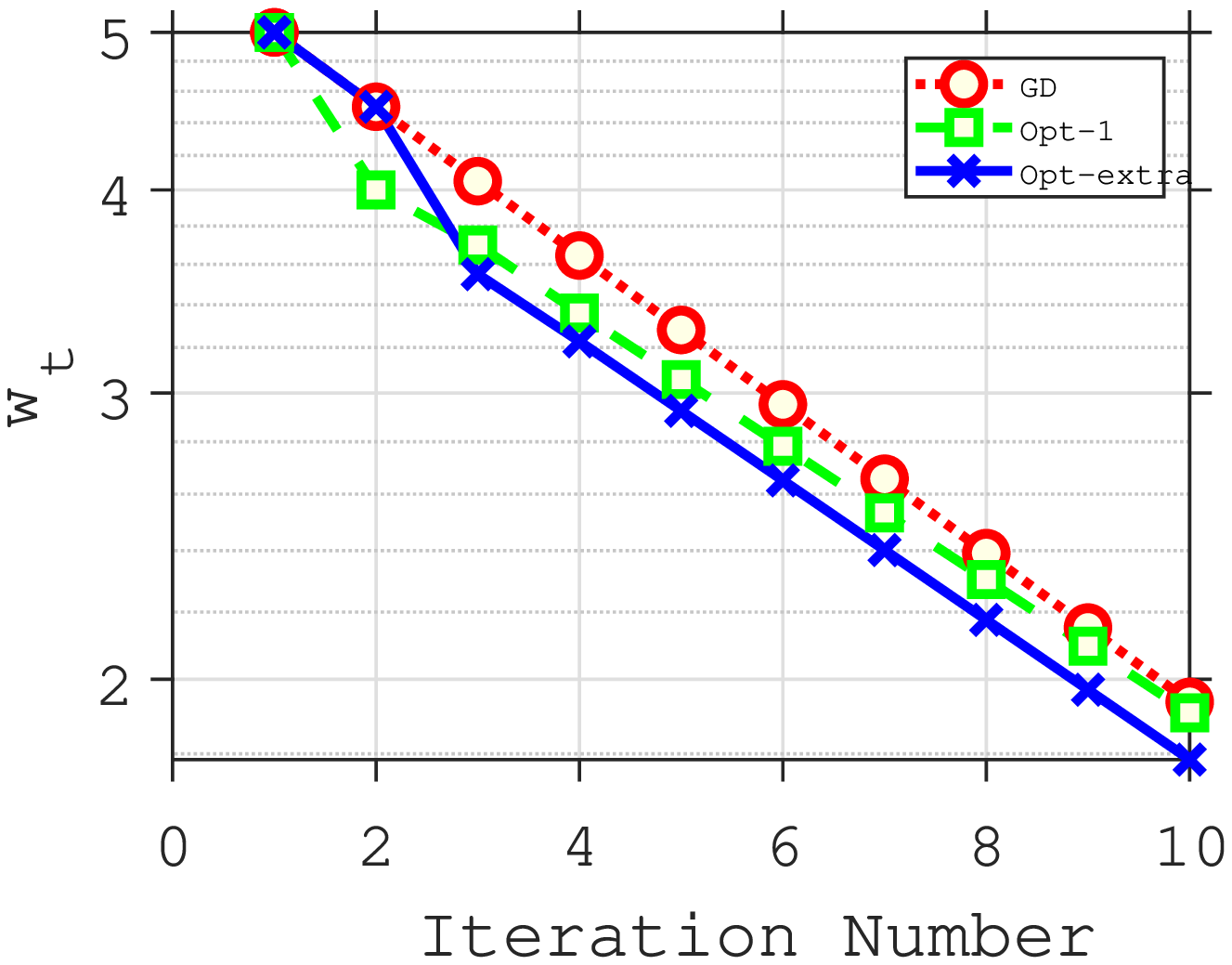}\hspace{-0.2in}
    }
    \subfigure[]{
    \includegraphics[width=1.7in]{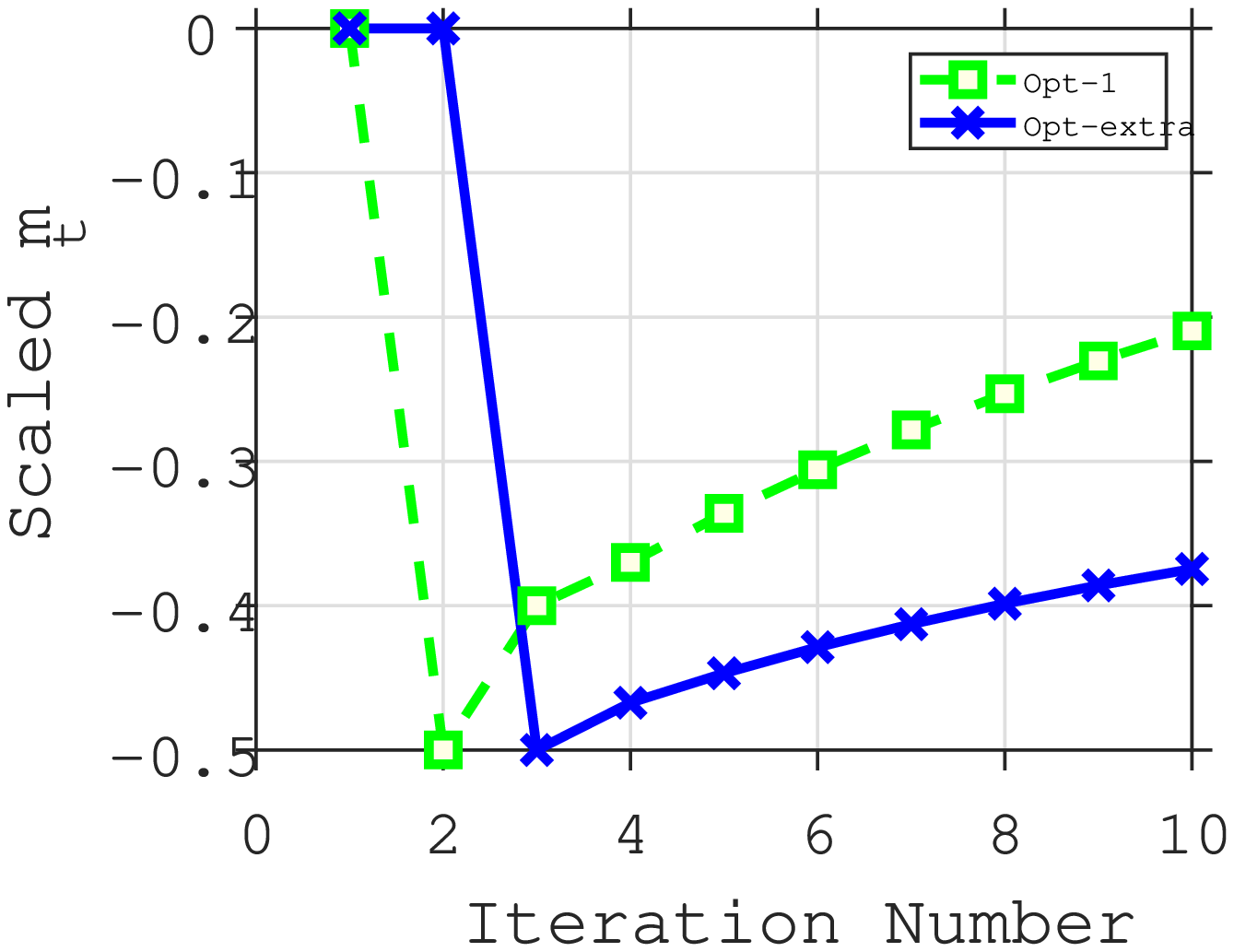}\hspace{-0.2in}
    }
%    }
%    \mbox{%\hspace{-0.2in}
    \subfigure[]{
    \includegraphics[width=1.7in]{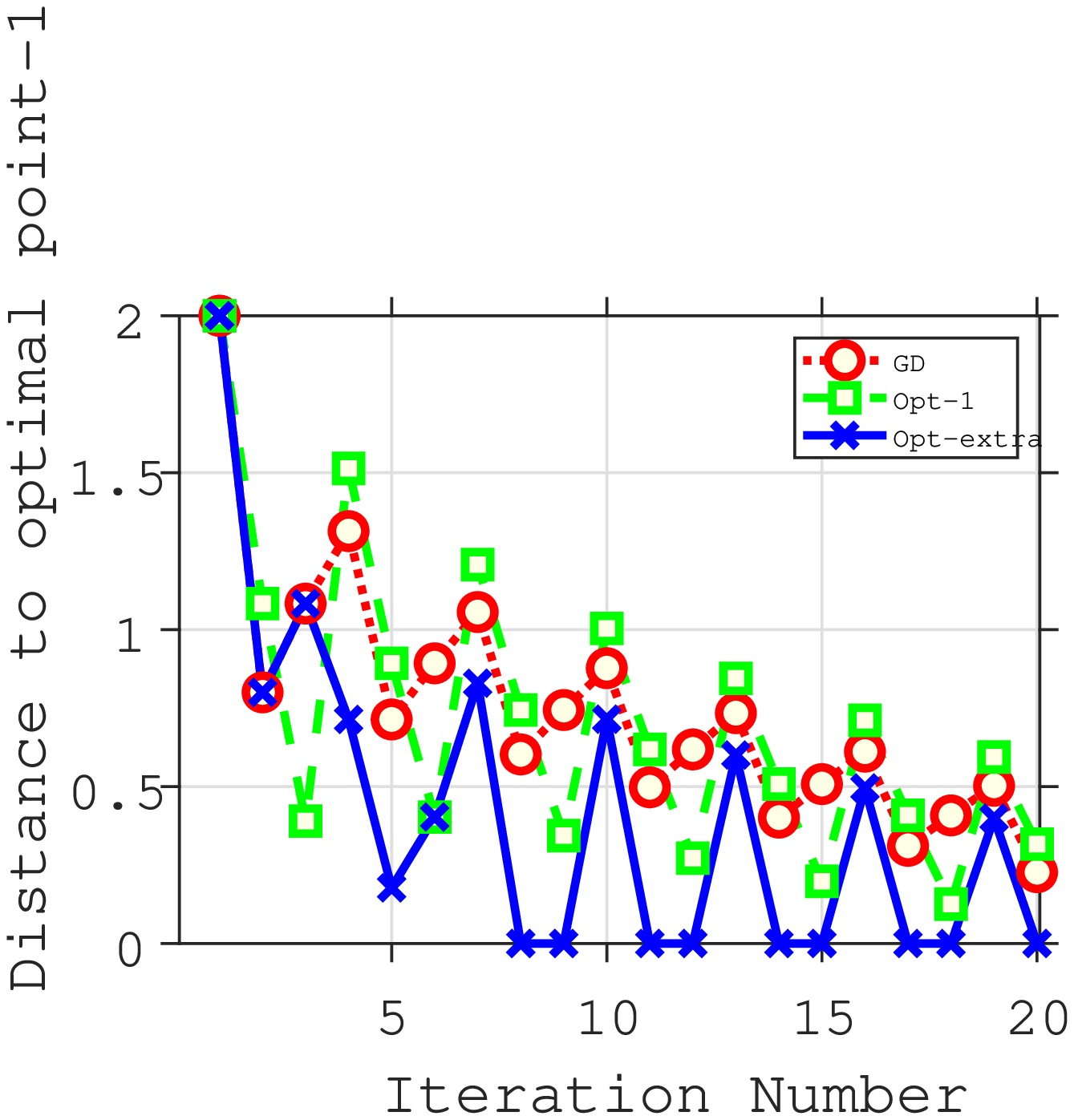}\hspace{-0.2in}
    }
        \subfigure[]{
    \includegraphics[width=1.7in]{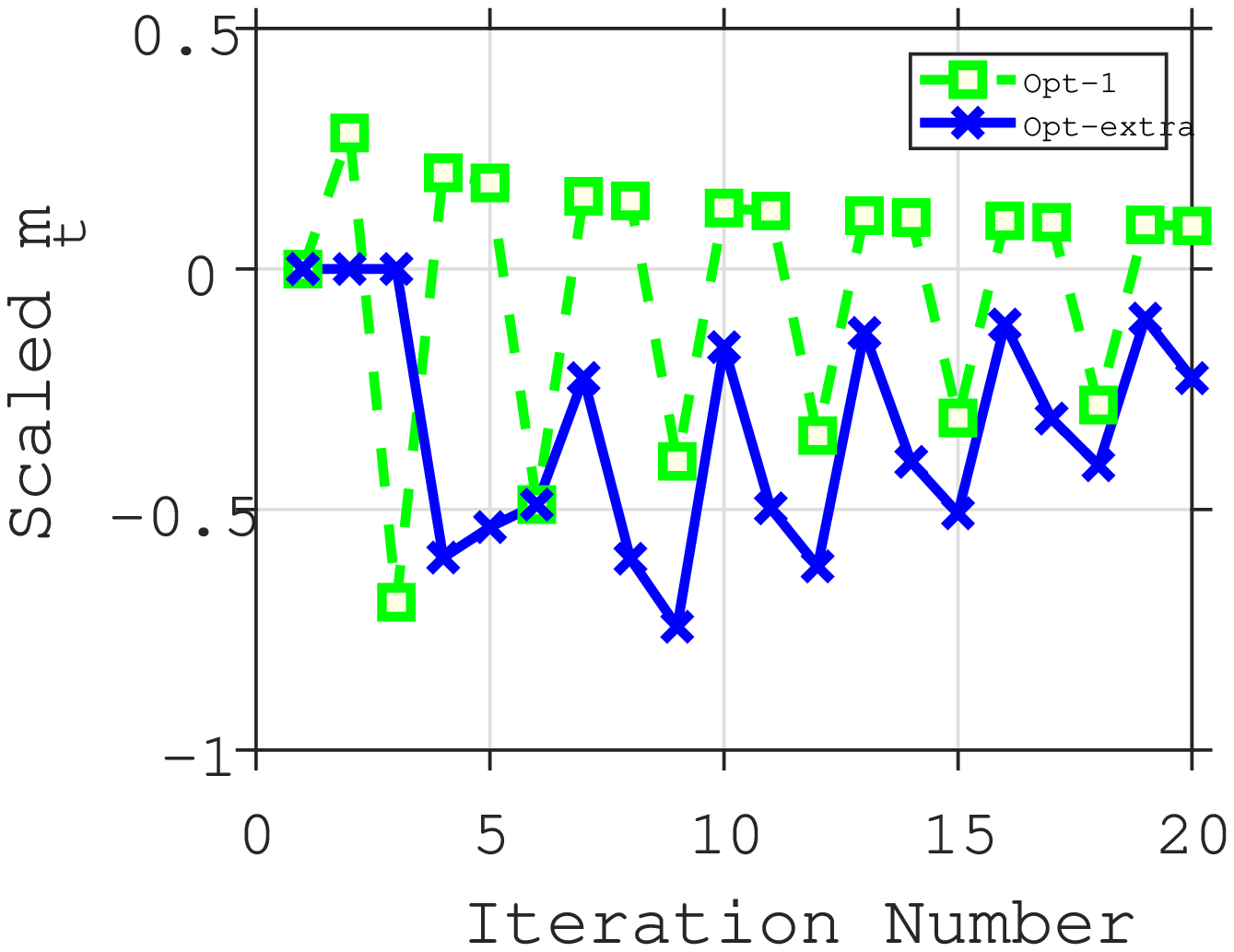}
    }
    }
\end{center}\vspace{-0.1in}
     \caption{\small (a): The iterate $w_t$; the closer to the optimal point $0$ the better. (b): A scaled and clipped version of $m_t$: $w_t - w_{t-1/2}$, which measures how the prediction of $m_t$ drives the update towards the optimal point. In this scenario, the more negative the better.
     (c): Distance to the optimal point $-1$. The smaller the better. (d): A scaled and clipped version of $m_t$: $w_t - w_{t-1/2}$, which measures how the prediction of $m_t$ drives
the update towards the optimal point. In this scenario, the more negative the better.
     }
     \label{simu}
\end{figure}

The simulation results are on Figure~\ref{simu} (a) and (b). Sub-figure (a) plots the updates $\{w_t\}_{t>0}$ through the iterations, where the updates go towards the optimal point $0$.
Sub-figure (b) displays a scaled and clipped version of $m_t$, defined as $w_t - w_{t-1/2}$, which can be viewed as $- \eta_t m_{t}$ if the projection (if existing) is lifted.
Sub-figure (a) shows that \textsc{Opt-extra} converges faster than the other methods. 
Furthermore, sub-figure (b) shows that the prediction by the extrapolation method is better than the prediction by simply using the previous gradient. 
The sub-figure shows that $-m_t$ from both methods points to $0$ for each iteration and the magnitude is larger for the one produced by the extrapolation method after iteration $2$. \footnote{The extrapolation needs at least two gradients for prediction. Thus, in the first two iterations,~$m_t=0$.}

Now let us consider another problem: an online learning problem proposed in~\cite{Proc:Reddi_ICLR18}
\footnote{\cite{Proc:Reddi_ICLR18} uses this example to show that \textsc{Adam}~\cite{Proc:Kingma_ICLR15} fails to converge.}.
Assume the learner's decision space is $\Theta=[-1,1]$, and the loss function is $\ell_t(w) = 3 w$ if $t \text{ mod } 3 = 1$, and $\ell_t(w) = - w$ otherwise.
The optimal point to minimize the cumulative loss is $w^*=-1$.
We let $\eta_t=0.1 / \sqrt{t}$ and the initial point $w_0=1$ for all three methods.
The parameter $\lambda$ of the extrapolation method is set to $\lambda=10^{-3}>0$. 
The results are reported Figure~\ref{simu} (c) and (d).
Sub-figure (c) shows that \textsc{Opt-extra} converges faster than the other methods while \textsc{Opt-1} is not performing better than GD.
The reason is that the gradient changes from $-1$ to $3$ at $t \text{ mod } 3 = 1$ and it changes from $3$ to $-1$ at $t \text{ mod } 3 = 2$.
Consequently, using the current gradient as the guess for the next is empirically not a good choice, since the next gradient is in the opposite direction of the current one, according to our experiments.
Sub-figure (d) shows that $-m_t$, obtained with the extrapolation method, always points to $w^*=-1$, while the one obtained by using the previous negative direction points to the opposite direction in two thirds of rounds. 
It empirically shows that the extrapolation method is much less affected by the gradient oscillation and always makes the prediction in the right direction, which suggests that the method can capture the aggregate effect.

%\newpage
\section{Proof of Theorem~\ref{thm:mainconvex}}\label{app:thmmainconvex}
\begin{Theorem*}
Suppose the learner incurs a sequence of convex loss functions $\{ \ell_{t}(\cdot) \}$.
Then,  \textsc{OPT-AMSGrad} (Algorithm~\ref{alg:optamsgrad}) has regret 
\begin{equation}\notag
\begin{aligned}
\mathcal{R}_T \leq &   \frac{ B_{\psi_1}(w^*, \tilde{w}_{1})}{\eta_1}
+ \sum_{t=1}^T\frac{\eta_t}{2} \| g_t - \tilde{m}_t  \|_{\psi_{t-1}^*}^2  + \frac{D_{\infty}^2}{\eta_{\min}}  \sum_{i=1}^d \hat{v}_{T}^{1/2}[i] + D_{\infty}^2\beta_1^2   \sum_{t=1}^T  \| g_t - \theta_{t-1}  \|_{\psi_{t-1}^*}\eqsp,
\end{aligned}
\end{equation}
where $ \tilde{m}_{t+1}  = \beta_1 \theta_{t-1} +(1-\beta_1) m_{t+1}$, $g_{t}:= \nabla \ell_{t}(w_t)$, $\eta_{{\min}} := \min_{{t}} \eta_{t}$ and $D_{\infty}^2$ is the diameter of the bounded set $\Theta$.
The result holds for any benchmark $w^{*} \in \Theta$ and any step size sequence $\{ \eta_t \}_{t>0}$.
\end{Theorem*}

\begin{proof}
Beforehand, we denote:
\beq\notag
\begin{split}
& \tilde{g}_t  = \beta_1 \theta_{t-1} +(1 - \beta_1) g_t \, , \\
& \tilde{m}_{t+1}  = \beta_1 \theta_{t-1} +(1-\beta_1) m_{t+1} \, ,
\end{split}
\eeq
where we recall that $g_t$ and $m_{t+1}$ are respectively the gradient $\nabla \ell_t(w_t)$ and the predictable guess.
By regret decomposition, we have that
\begin{equation} \label{nn1}
\begin{aligned}
 & \mathcal{R}_T:= \sum_{t=1}^T \ell_t(w_t) - \min_{w \in \Theta} \sum_{t=1}^T \ell_t(w)  \textstyle  \\
  \leq & \sum_{t=1}^T  \langle w_t - w^*, \nabla \ell_t(w_t) \rangle
\\  = &\sum_{t=1}^T \langle  w_t - \tilde{w}_{t+1} , g_t - \tilde{m}_t \rangle + \langle w_t - \tilde{w}_{t+1}, \tilde{m}_t \rangle + \langle \tilde{w}_{t+1} - w^*, \tilde{g}_t  \rangle+ \langle \tilde{w}_{t+1} - w^*,g_t - \tilde{g}_t  \rangle \eqsp.
\end{aligned}
\end{equation}

Recall the notation $\psi_t(x)$ and the Bregman divergence $B_{\psi_t}(u,v)$ defined Section~\ref{sec:analysis}.
We exploit a useful inequality (which appears in e.g.,~\cite{tseng2008accelerated}).
For any update of the form $\hat{w} = \arg\min_{w \in \Theta} \langle w, \theta \rangle + B_{\psi}(w, v)$, it holds that
\begin{equation} \label{ii}
\langle \hat{w} - u, \theta \rangle \leq B_{\psi}( u, v ) - B_{\psi}( u, \hat{w}) - B_{\psi}( \hat{w}, v) \quad \textrm{for any $u \in \Theta$} \eqsp.
\end{equation}
For $\beta_1=0$, we can rewrite the update on line 8 of (Algorithm~\ref{alg:optamsgrad}) as
\begin{equation} \label{nc1}
\tilde{w}_{t+1}= \arg\min_{w \in \Theta} \eta_t \langle w, \tilde{g}_t \rangle + B_{\psi_t}(w, \tilde{w}_{t} ) \eqsp.
\end{equation}
By using (\ref{ii}) for (\ref{nc1}) with $\hat{w} = \tilde{w}_{t+1}$ (the output of the minimization problem), $u = w^*$ and $v = \tilde{w}_{t}$, we have
\begin{equation} \label{nn2}
\begin{aligned}
 \langle \tilde{w}_{t+1} - w^*, & \tilde{g}_t \rangle \leq \frac{1}{\eta_t}\big[ B_{\psi_t}( w^*, \tilde{w}_{t}) -B_{\psi_t}(w^*,  \tilde{w}_{t+1} ) - B_{\psi_t}(\tilde{w}_{t+1}, \tilde{w}_{t}) \big] \eqsp.
\end{aligned}
\end{equation}

We can also rewrite the update on line 9 of (Algorithm~\ref{alg:optamsgrad}) at time $t$ as
\begin{equation} \label{nc2}
\textstyle w_{t+1} = \arg\min_{w \in \Theta} \eta_{t+1} \langle w, \tilde{m}_{t+1} \rangle + B_{\psi_t}(w, \tilde{w}_{t+1} ) \eqsp.
\end{equation}
and, by using \eqref{ii} for \eqref{nc2} (written at iteration $t$), with $\hat{w} = w_{t}$ (the output of the minimization problem), $u = \tilde{w}_{t+1}$ and $v = \tilde{w}_{t}$, we have
\begin{equation} \label{nn3}
\begin{aligned}
\langle w_t -\tilde{w}_{t+1}, & \tilde{m}_t  \rangle \leq \frac{1}{\eta_t}\big[ B_{\psi_{t-1}}(\tilde{w}_{t+1}, \tilde{w}_{t}) - B_{\psi_{t-1}}(\tilde{w}_{t+1}, w_t ) - B_{\psi_{t-1}}(w_{t}, \tilde{w}_{t}) \big] \eqsp.
\end{aligned}
\end{equation}
By (\ref{nn1}), (\ref{nn2}), and (\ref{nn3}), we obtain
\begin{equation} \notag
\begin{aligned}
 \mathcal{R}_T & \overset{(\ref{nn1})}{\leq} \sum_{t=1}^T \langle  w_t - \tilde{w}_{t+1} , g_t - \tilde{m}_t \rangle + \langle w_t - \tilde{w}_{t+1}, \tilde{m}_t \rangle + \langle \tilde{w}_{t+1} - w^*, \tilde{g}_t  \rangle+ \langle \tilde{w}_{t+1} - w^*,g_t - \tilde{g}_t  \rangle \\
& \overset{(\ref{nn2}), (\ref{nn3})}{\leq}  \sum_{t=1}^T \| w_t - \tilde{w}_{t+1} \|_{\psi_{t-1}} \| g_t - \tilde{m}_t  \|_{\psi_{t-1}^*} + \|  \tilde{w}_{t+1} - w^* \|_{\psi_{t-1}} \| g_t - \tilde{g}_t  \|_{\psi_{t-1}^*}\\
&+ \frac{1}{\eta_t} \big[ B_{\psi_{t-1}}(\tilde{w}_{t+1}, \tilde{w}_{t}) - B_{\psi_{t-1}}(\tilde{w}_{t+1}, w_t ) - B_{\psi_{t-1}}(w_{t}, \tilde{w}_{t}) \\
& +  B_{\psi_t}( w^*, \tilde{w}_{t}) -B_{\psi_t}(w^*,  \tilde{w}_{t+1} ) - B_{\psi_t}(\tilde{w}_{t+1}, \tilde{w}_{t}) \big] \eqsp,
\end{aligned}
\end{equation}
which is further bounded by
\begin{equation} \label{nnnn}
\begin{aligned}
  \mathcal{R}_T \leq \sum_{t=1}^T \Big\{  & \frac{1}{2 \eta_t} \| w_t - \tilde{w}_{t+1} \|_{\psi_{t-1}}^2 + \frac{\eta_t}{2} \| g_t - m_t  \|_{\psi_{t-1}^*}^2+ \|  \tilde{w}_{t+1} - w^* \|_{\psi_{t-1}} \| g_t - \tilde{g}_t  \|_{\psi_{t-1}^*}\\\
 &  + \frac{1}{\eta_t} \big(\underbrace{B_{\psi_{t-1}}(\tilde{w}_{t+1}, \tilde{w}_{t}) - B_{\psi_t}(\tilde{w}_{t+1}, \tilde{w}_{t}) }_{A_1}- \frac{1}{2} \| \tilde{w}_{t+1} - w_t \|_{\psi_{t-1}}^2\\
 &+\underbrace{B_{\psi_t}( w^*, \tilde{w}_{t}) -B_{\psi_t}(w^*,  \tilde{w}_{t+1} )}_{A_2}  \big) \Big\} \eqsp,
\end{aligned}
\end{equation}
where the inequality is due to $ \| w_t - \tilde{w}_{t+1}   \|_{\psi_{t-1}} \| g_t - m_t  \|_{\psi_{t-1}^*} = \inf_{ \beta > 0 }   \frac{1}{2\beta} \| w_t - \tilde{w}_{t+1} \|_{\psi_{t-1}}^2 +  \frac{\beta}{2} \| g_t - m_t  \|_{\psi_{t-1}^*}^2$ by Young's inequality and the 1-strongly convex of $\psi_{t-1}(\cdot)$ with respect to $\| \cdot \|_{\psi_{t-1}}$ which yields that $B_{\psi_{t-1}}(\tilde{w}_{t+1}, w_t )  \geq \frac{1}{2} \| \tilde{w}_{t+1} -  w_t  \|^2_{\psi_t} \geq 0$. 

\vspace{0.1in}

To proceed, notice that
\begin{equation} \label{nn5}
\begin{aligned}
A_1& :=  B_{\psi_{t-1}}(\tilde{w}_{t+1}, \tilde{w}_{t}) - B_{\psi_t}(\tilde{w}_{t+1}, \tilde{w}_{t})  \\
& = \langle \tilde{w}_{t+1} - \tilde{w}_{t} , \text{diag}(\hat{v}_{t-1}^{1/2} -\hat{v}_t^{1/2} ) ( \tilde{w}_{t+1}- \tilde{w}_{t} ) \rangle \leq 0 \eqsp,
\end{aligned}
\end{equation}
as the sequence $\{\hat{v}_t\}$ is non-decreasing. And that
\begin{equation}\label{nn4}
\begin{aligned}
A_2 := B_{\psi_t}( w^*, \tilde{w}_{t}) -B_{\psi_t}(w^*,  \tilde{w}_{t+1} )  &= \langle w^* - \tilde{w}_{t+1}  , \text{diag}(\hat{v}_{t+1}^{1/2} -\hat{v}_t^{1/2} ) ( w^* - \tilde{w}_{t+1}  ) \rangle\\
  & \leq ( \max_i (w^*[i] -  \tilde{w}_{t+1} [i] )^2  )\cdot ( \sum_{i=1}^d \hat{v}_{t+1}^{1/2}[i] -\hat{v}_t^{1/2}[i] ) \eqsp.
\end{aligned}
\end{equation}
Therefore, by (\ref{nnnn}),(\ref{nn4}),(\ref{nn5}), we have
\begin{equation}\notag
\begin{aligned}
\mathcal{R}_T \leq  \frac{D_{\infty}^2}{\eta_{\min}}  \sum_{i=1}^d \hat{v}_{T}^{1/2}[i] + \frac{ B_{\psi_1}(w^*, \tilde{w}_{1})}{\eta_1}
+ \sum_{t=1}^T\frac{\eta_t}{2} \| g_t - \tilde{m}_t  \|_{\psi_{t-1}^*}^2  +D_{\infty}^2 \beta_1^2  \sum_{t=1}^T \| g_t - \theta_{t-1}  \|_{\psi_{t-1}^*}  \eqsp,
\end{aligned}
\end{equation}
since $  \| g_t - \tilde{g}_t  \|_{\psi_{t-1}^*} =  \| g_t - \beta_1 \theta_{t-1} -(1- \beta_1) g_t \|_{\psi_{t-1}^*} = \beta^2 \| g_t - \theta_{t-1}  \|_{\psi_{t-1}^*} $.
This completes the proof.

\end{proof}

\vspace{0.2in}

\section{Proof of Corollary~\ref{cor:corollary}}
\begin{Corollary*}
Suppose $\beta_1=0$ and $\{v_t\}_{t>0}$ is a monotonically increasing sequence, then we obtain the following regret bound for any $w^{*} \in \Theta$ and sequence of stepsizes $\{ \eta_t = \eta/\sqrt{t}\}_{t>0}$: 
\begin{equation}\notag
\begin{aligned}
\mathcal{R}_T \leq & \frac{ B_{\psi_1}}{\eta_1}
+ \frac{\eta \sqrt{1 + \log T}}{\sqrt{1 - \beta_2}} \sum_{i=1}^d \| (g-m)_{1:T}[i] \|_2 +\frac{D_{\infty}^2}{\eta_{\min}} \sum_{i=1}^d \left[ (1-\beta_2) \sum_{s=1}^{T} \beta_2^{T-s} g^2_s[i] \right]^{1/2} \eqsp,
\end{aligned}
\end{equation}
where $B_{\psi_1} := B_{\psi_1}(w^*, \tilde{w}_{1})$, $g_{t}:= \nabla \ell_{t}(w_t)$ and $\eta_{{\min}} := \min_{{t}} \eta_{t}$.
\end{Corollary*}

\begin{proof}
Recall the bound in Theorem~\ref{thm:mainconvex}:
\begin{equation}\notag
\begin{aligned}
\mathcal{R}_T \leq &   \frac{ B_{\psi_1}(w^*, \tilde{w}_{1})}{\eta_1}
+ \sum_{t=1}^T\frac{\eta_t}{2} \| g_t - \tilde{m}_t  \|_{\psi_{t-1}^*}^2  + \frac{D_{\infty}^2}{\eta_{\min}}  \sum_{i=1}^d \hat{v}_{T}^{1/2}[i] + D_{\infty}^2\beta_1^2   \sum_{t=1}^T  \| g_t - \theta_{t-1}  \|_{\psi_{t-1}^*}\eqsp.
\end{aligned}
\end{equation}

The second term reads:
\begin{align*}
     &\sum_{t=1}^T \frac{\eta_t}{2} \|g_t - m_t  \|_{\psi_{t-1}^*}^2 \\
      =&\sum_{t=1}^{T-1} \frac{\eta_t}{2} \|g_t - m_t  \|_{\psi_{t-1}^*}^2   + \eta_T \sum_{i=1}^d \frac{ (g_{T}[i] - m_{T}[i])^2 }{ \sqrt{ v_{T-1}[i]} }\\
=&\sum_{t=1}^{T-1} \frac{\eta_t}{2} \|g_t - m_t  \|_{\psi_{t-1}^*}^2 + \eta \sum_{i=1}^d \frac{ (g_{T}[i] - m_{T}[i])^2 }{ \sqrt{ T \big( (1-\beta_2) \sum_{s=1}^{T-1} \beta_2^{T-1-s} (g_{s}[i] - m_{s}[i])^2 } \big) }\\
\leq &  \eta \sum_{i=1}^d \sum_{t=1}^T \frac{ (g_{t}[i] - m_{t}[i])^2 }{ \sqrt{ t \big( (1-\beta_2) \sum_{s=1}^{t-1} \beta_2^{t-1-s} (g_{s}[i] - m_{s}[i])^2 } \big) } \eqsp.
\end{align*}
To interpret the bound, let us make a rough approximation such that
$\sum_{s=1}^{t-1} \beta_2^{t-1-s} (g_{s}[i] - m_{s}[i])^2  \simeq (g_{t}[i] - m_t[i])^2 $.
Then, we can further get an upper-bound as 
\begin{align*}
    \sum_{t=1}^T \frac{\eta_t}{2} \|g_t - m_t  \|_{\psi_{t-1}^*}^2 \leq
    \frac{\eta}{\sqrt{1 - \beta_2}} \sum_{i=1}^d \sum_{t=1}^{T} \frac{ | g_{t}[i] - m_{t}[i] | }{ \sqrt{t} }\leq \frac{\eta \sqrt{1 + \log T}}{\sqrt{1 - \beta_2}} \sum_{i=1}^d \| (g-m)_{1:T}[i] \|_2 ,
\end{align*}
where the last inequality is due to Cauchy-Schwarz.

\end{proof}

\vspace{0.2in}

\section{Proofs of Auxiliary Lemmas}
Following~\cite{Proc:Yan_IJCAI18} and their study of the SGD with Momentum we denote for any $t >0$:
\beq\label{eq:deftilde}
\overline{w}_t = w_t + \frac{\beta_1}{1 - \beta_1} (w_t - \tilde{w}_{t-1}) = \frac{1}{1 - \beta_1} w_t -  \frac{\beta_1}{1 - \beta_1} \tilde{w}_{t-1} \eqsp.
\eeq
\begin{Lemma}\label{lem:momentum}
Assume a strictly positive and non increasing sequence of stepsizes $\{\eta_t \}_{t>0}$, $\beta_1 < \beta_2 \in [0,1)$, then the following holds:
\beq\notag
\overline{w}_{t+1} - \overline{w}_t \leq \frac{\beta_1}{1 - \beta_1} \tilde{\theta}_{t-1} \left[ \eta_{t-1} \hat{v}_{t-1}^{-1/2} - \eta_{t} \hat{v}_{t}^{-1/2}\right] - \eta_{t} \hat{v}_{t}^{-1/2} \tilde{g}_t \eqsp,
\eeq
where $\tilde{\theta}_t = \theta_t + \beta_1 \theta_{t-1}$ and $\tilde{g}_t = g_t - \beta_1 m_t + \beta_1 g_{t-1} + m_{t+1} $.
\end{Lemma}
\begin{proof}
By definition \eqref{eq:deftilde} and using the Algorithm updates, we have:
\beq\notag
\begin{split}
\overline{w}_{t+1} - \overline{w}_t  &= \frac{1}{1 - \beta_1} ( w_{t+1} - \tilde{w}_t)  -  \frac{\beta_1}{1 - \beta_1}( w_{t} - \tilde{w}_{t-1})\\
& = - \frac{1}{1 - \beta_1} \eta_{t} \hat{v}_{t}^{-1/2} (\theta_t + h_{t+1})  +  \frac{\beta_1}{1 - \beta_1}\eta_{t-1} \hat{v}_{t-1}^{-1/2} (\theta_{t-1} + h_{t})\\
& = - \frac{1}{1 - \beta_1}  \eta_{t} \hat{v}_{t}^{-1/2} (\theta_t + \beta_1 \theta_{t-1}) -\frac{1}{1 - \beta_1}  \eta_{t} \hat{v}_{t}^{-1/2} (1- \beta_1) m_{t+1}\\
& + \frac{\beta_1}{1 - \beta_1} \eta_{t-1} \hat{v}_{t-1}^{-1/2} (\theta_{t-1} + \beta_1 \theta_{t-2}) + \frac{\beta_1}{1 - \beta_1}  \eta_{t-1} \hat{v}_{t-1}^{-1/2} (1- \beta_1) m_{t} \eqsp.
\end{split}
\eeq
Denote $\tilde{\theta}_t = \theta_t + \beta_1 \theta_{t-1}$ and $\tilde{g}_t = g_t - \beta_1 m_t + \beta_1 g_{t-1} + m_{t+1} $.
Notice that $\tilde{\theta}_t = \beta_1 \tilde{\theta}_{t-1} + (1 - \beta_1) (g_t + \beta_1 g_{t-1})$.
\beq\notag
\begin{split}
\overline{w}_{t+1} - \overline{w}_t \leq \frac{\beta_1}{1- \beta_1} \tilde{\theta}_{t-1} \left[ \eta_{t-1} \hat{v}_{t-1}^{-1/2} - \eta_{t} \hat{v}_{t}^{-1/2} \right] - \eta_t \hat{v}_t^{-1/2} \tilde{g}_t \eqsp.
\end{split}
\eeq
\end{proof}

\begin{Lemma}\label{lem:squarev}
Assume H\ref{ass:bounded}, a strictly positive and a sequence of constant stepsizes $\{\eta_t \}_{t>0}$, $(\beta_1, \beta_2) \in [0,1]$, then the following holds:
\beq\notag
\sum_{t=1}^{T_{\sf M}} \eta_{t}^{2} \EE \left[\left\|\hat{v}_{t}^{-1/2} \theta_{t}\right\|_{2}^{2}\right] \leq  \frac{\eta^{2} d T_{\sf M} (1- \beta_1)}{(1 - \beta_2)(1-\gamma)} \eqsp.
\eeq
\end{Lemma}
\begin{proof}
We denote by index $p \in [1,d]$ the dimension of each component of vectors of interest. 
Noting that for any $t >0$ and dimension $p$ we have $\hat{v}_{t,p} \geq v_{t,p}$, then:
\beq\notag
\begin{split}
\eta_{t}^{2} \EE \left[\left\|\hat{v}_{t}^{-1/2} \theta_{t}\right\|_{2}^{2}\right] =\eta_{t}^{2} \mathbb{E}\left[\sum_{p=1}^{d} \frac{\theta_{t, p}^{2}}{\hat{v}_{t, p}}\right]  
 \leq \eta_{t}^{2} \mathbb{E}\left[\sum_{i=1}^{d} \frac{\theta_{t, p}^{2}}{v_{t, p}}\right] 
 \leq \eta_{t}^{2} \mathbb{E}\left[\sum_{i=1}^{d} \frac{( \sum_{r=1}^t (1 - \beta_1) \beta_1^{t-r} g_{r,p})^{2}}{ \sum_{r=1}^t (1 - \beta_2) \beta_2^{t-r} g^2_{r,p}}\right]  \eqsp,
\end{split}
\eeq
where the last inequality is due to initializations.
Denote $\gamma = \frac{\beta_1}{\beta_2}$.
Then,
\beq\notag
\begin{split}
\eta_{t}^{2} \EE \left[\left\|\hat{v}_{t}^{-1/2} \theta_{t}\right\|_{2}^{2}\right] &\leq \frac{\eta_{t}^{2} (1- \beta_1)^2}{1 - \beta_2}  \mathbb{E}\left[\sum_{i=1}^{d} \frac{( \sum_{r=1}^t \beta_1^{t-r} g_{r,p})^{2}}{ \sum_{r=1}^t \beta_2^{t-r} g^2_{r,p}}\right] \\
& \overset{(a)}{\leq}\frac{\eta_{t}^{2} (1- \beta_1)}{1 - \beta_2}  \mathbb{E}\left[\sum_{i=1}^{d} \frac{ \sum_{r=1}^t \beta_1^{t-r} g_{r,p}^{2}}{ \sum_{r=1}^t \beta_2^{t-r} g^2_{r,p}}\right]\\
& \leq \frac{\eta_{t}^{2} (1- \beta_1)}{1 - \beta_2}  \mathbb{E}\left[\sum_{i=1}^{d}\sum_{r=1}^t \gamma^{t-r}\right]  = \frac{\eta_{t}^{2} d (1- \beta_1)}{1 - \beta_2}  \mathbb{E}\left[\sum_{r=1}^t  \gamma^{t-r}\right] \eqsp,
\end{split}
\eeq
where $(a)$ is due to $ \sum_{r=1}^t \beta_1^{t-r} \leq \frac{1}{1 - \beta_1}$.
Summing from  $t =1$ to $t = T_{\sf M}$ on both sides yields:
\beq\notag
\begin{split}
\sum_{t=1}^{T_{\sf M}} \eta_{t}^{2} \EE \left[\left\|\hat{v}_{t}^{-1/2} \theta_{t}\right\|_{2}^{2}\right] &\leq   \frac{\eta_{t}^{2} d (1- \beta_1)}{1 - \beta_2}  \mathbb{E}\left[ \sum_{t=1}^{T_{\sf M}} \sum_{r=1}^t  \gamma^{t-r}\right]\\
& \leq  \frac{\eta^{2} d T (1- \beta_1)}{1 - \beta_2}  \mathbb{E}\left[ \sum_{t=t}^t   \gamma^{t-r}\right]\\
& \leq  \frac{\eta^{2} d T (1- \beta_1)}{(1 - \beta_2)(1-\gamma)} \eqsp,
\end{split}
\eeq
where the last inequality is due to $\sum_{r=1}^t   \gamma^{t-r} \leq \frac{1}{1 - \gamma}$ by definition of $\gamma$.
\end{proof}

\subsection{Proof of Lemma~\ref{lem:bound}}\label{app:lembound}
\begin{Lemma*}
Assume assumption H\ref{ass:bounded}, then the quantities defined in Algorithm~\ref{alg:optamsgrad} satisfy for any $w \in \Theta$ and $t>0$:
$$ \|\nabla f(w_t)\| < \major ,~~~\|\theta_t \| < \major ,~~~\|\hat{v}_t\| < \major^2 \eqsp.$$
\end{Lemma*}
\begin{proof}
Assume assumption H\ref{ass:bounded} we have:
$$
\norm{\nabla f(w)} = \norm{\EE[\nabla f(w, \xi)]} \leq \EE[\norm{\nabla f(w, \xi)}] \leq \major \eqsp.
$$
By induction reasoning, since $\norm{\theta_0} = 0 \leq \major$ and suppose that for $\norm{\theta_t}\leq \major$ then we have 
\beq\notag
\begin{split}
\norm{\theta_{t+1}}  =\norm{\beta_{1} \theta_{t}+\left(1-\beta_{1}\right) g_{t+1}} \leq \beta_1 \norm{\theta_{t}} + (1 - \beta_1) \norm{g_{t+1}} \leq \major \eqsp.
\end{split}
\eeq
Using the same induction reasoning we prove that
\beq\notag
\begin{split}
\norm{\hat{v}_{t+1}}  =\norm{\beta_{2} \hat{v}_{t}+\left(1-\beta_{2}\right) g_{t+1}^2} \leq \beta_2 \norm{\hat{v}_{t}} + (1 - \beta_1) \norm{g^2_{t+1}} \leq \major^2 \eqsp.
\end{split}
\eeq
\end{proof}

\vspace{0.2in}

\section{Proof of Theorem~\ref{thm:boundopt}}\label{app:thmboundopt}
\begin{Theorem*}
Assume H\ref{ass:boundedparam}-H\ref{ass:bounded}, $\beta_1 < \beta_2 \in [0,1)$ and a sequence of decreasing stepsizes $\{\eta_t\}_{t>0}$, then the following result holds:
\beq\notag
\begin{split}
\EE\left[\|\nabla f(w_T)\|_2^2\right] \leq \tilde{C}_1 \sqrt{\frac{d}{T_{\sf M}}} + \tilde{C}_2 \frac{1}{T_{\sf M}} \eqsp,
\end{split}
\eeq
where $T$ is a random termination number distributed according \eqref{eq:random}.
The constants are defined as:
{\fontsize{9.5}{9}
\begin{align*}
&\tilde{C}_1 =\frac{\major}{(1 - a_m\beta_1) + (\beta_1 + a_m)}  \left[ \frac{a_m(1 - \beta_1)^2}{1-\beta_2} + 2L \frac{1}{1-\beta_2} +  \Delta f  +   \frac{4L \beta_1^2(1 + \beta_1^2) }{(1 - \beta_1)(1 - \beta_2)(1-\gamma)} \right],\\
&\tilde{C}_2 = \frac{ (a_m\beta_1^2 -2 a_m \beta_1 + \beta_1)\major^2 }{(1 - \beta_1) \left((1 - a_m\beta_1) + (\beta_1 + a_m)\right)}  \EE\left[ \norm{\hat{v}_{0}^{-1/2}}    \right]  \eqsp,
\end{align*}
}
where $\Delta f = f(\overline{w}_{1}) - f(\overline{w}_{T_{\sf M}+1})$ and $a_m=\displaystyle{\min_{t=1,...,T}}a_t$.
\end{Theorem*}

\begin{proof}
Using H\ref{ass:smooth} and the iterate $\overline{w}_t$ we have:
\beq\label{eq:smoothness}
\begin{split}
f(\overline{w}_{t+1})  \leq & f(\overline{w}_t) + \nabla f(\overline{w}_t)^\top (\overline{w}_{t+1} - \overline{w}_t) + \frac{L}{2} \|\overline{w}_{t+1} - \overline{w}_t\|^2\\
 \leq &f(\overline{w}_t) + \underbrace{ \nabla f(w_t)^\top (\overline{w}_{t+1} - \overline{w}_t)}_{A} \\
&+ \underbrace{  \left( \nabla f(\overline{w}_t) -  \nabla f(w_t)\right)^\top (\overline{w}_{t+1} - \overline{w}_t)}_{B} + \frac{L}{2} \|\overline{w}_{t+1} - \overline{w}_t\| \eqsp.
\end{split}
\eeq

\textbf{Term A}.
Using Lemma~\ref{lem:momentum}, we have that:
\beq \notag
\begin{split}
\nabla f(w_t)^\top (\overline{w}_{t+1} - \overline{w}_t) & \leq \nabla f(w_t)^\top \left[\frac{\beta_1}{1 - \beta_1} \tilde{\theta}_{t-1} \left[ \eta_{t-1} \hat{v}_{t-1}^{-1/2} - \eta_{t} \hat{v}_{t}^{-1/2}\right] - \eta_{t} \hat{v}_{t}^{-1/2} \tilde{g}_t \right]\\
& \leq  \frac{\beta_1}{1 - \beta_1}  \| \nabla f(w_t)\| \|\eta_{t-1} \hat{v}_{t-1}^{-1/2} - \eta_{t} \hat{v}_{t}^{-1/2} \| \|\tilde{\theta}_{t-1}\| - \nabla f(w_t)^\top\eta_{t} \hat{v}_{t}^{-1/2} \tilde{g}_t \eqsp,
\end{split}
\eeq
where the inequality is due to trivial inequality for positive diagonal matrix.

Using Lemma~\ref{lem:bound} and assumption H\ref{ass:guessbound} we obtain:
\beq\label{eq:termA1}
\begin{split}
\nabla f(w_t)^\top (\overline{w}_{t+1} - \overline{w}_t)  \leq  \frac{\beta_1 (1+\beta_1)}{1 - \beta_1} \major^2 [ \|\eta_{t-1} \hat{v}_{t-1}^{-1/2}\| - \|\eta_{t} \hat{v}_{t}^{-1/2} \|] - \nabla f(w_t)^\top\eta_{t} \hat{v}_{t}^{-1/2} \tilde{g}_t  \eqsp,
\end{split}
\eeq
where we have used the fact that $\eta_{t} \hat{v}_{t}^{-1/2} $ is a diagonal matrix such that $\eta_{t-1} \hat{v}_{t-1}^{-1/2} \succcurlyeq \eta_{t} \hat{v}_{t}^{-1/2}\succcurlyeq 0$ (decreasing stepsize and $\max$ operator).
Also note that:
\beq\label{eq:termA2}
\begin{split}
 - \nabla f(w_t)^\top\eta_{t} \hat{v}_{t}^{-1/2} \tilde{g}_t  &=  - \nabla f(w_t)^\top\eta_{t-1} \hat{v}_{t-1}^{-1/2} \bar{g}_t   -  \nabla f(w_t)^\top\left[ \eta_{t} \hat{v}_{t}^{-1/2} -\eta_{t} \hat{v}_{t}^{-1/2} \right] \bar{g}_t  \\ 
&   - \nabla f(w_t)^\top\eta_{t-1} \hat{v}_{t-1}^{-1/2} (\beta_1 g_{t-1} + m_{t+1})\\
 & \leq  - \nabla f(w_t)^\top\eta_{t-1} \hat{v}_{t-1}^{-1/2} \bar{g}_t +(1-a_t\beta_1)\major^2    [ \|\eta_{t-1} \hat{v}_{t-1}^{-1/2}\| - \|\eta_{t} \hat{v}_{t}^{-1/2} \| ] \\
 &  - \nabla f(w_t)^\top\eta_{t} \hat{v}_{t}^{-1/2} (\beta_1 g_{t-1} + m_{t+1}) \eqsp,
\end{split}
\eeq
where we have used Lemma~\ref{lem:bound} on $\|g_t\|$ and where that $\tilde{g}_t = \bar{g}_t  + \beta_1 g_{t-1} + m_{t+1} = g_t - \beta_1 m_t + \beta_1 g_{t-1} + m_{t+1} $.
Plugging \eqref{eq:termA2} into \eqref{eq:termA1} yields:
\beq\label{eq:termA}
\begin{split}
&\nabla f(w_t)^\top (\overline{w}_{t+1} - \overline{w}_t)\\
&  \leq   - \nabla f(w_t)^\top\eta_{t-1} \hat{v}_{t-1}^{-1/2} \bar{g}_t + \frac{1}{1 - \beta_1} (a_t\beta_1^2 -2 a_t \beta_1 + \beta 1)\major^2 [ \|\eta_{t-1} \hat{v}_{t-1}^{-1/2}\| - \|\eta_{t} \hat{v}_{t}^{-1/2} \|] \\
&  - \nabla f(w_t)^\top\eta_{t} \hat{v}_{t}^{-1/2} (\beta_1 g_{t-1} + m_{t+1}) \eqsp .
\end{split}
\eeq

\vspace{0.1in}

\textbf{Term B}.
By Cauchy-Schwarz (CS) inequality we have:
\beq\label{eq:termB1}
 \left( \nabla f(\overline{w}_t) -  \nabla f(w_t)\right)^\top (\overline{w}_{t+1} - \overline{w}_t) \leq  \| \nabla f(\overline{w}_t) -  \nabla f(w_t)\|  \|\overline{w}_{t+1} - \overline{w}_t\| \eqsp.
 \eeq
 Using smoothness assumption H\ref{ass:smooth}:
\beq\label{eq:termB2}
 \begin{split}
  \| \nabla f(\overline{w}_t) -  \nabla f(w_t)\| & \leq L \| \overline{w}_t - w_t\|\\
  & \leq L \frac{\beta_1}{1 - \beta_1} \|w_t - \tilde{w}_{t-1}\| \eqsp.
 \end{split}
 \eeq
By Lemma~\ref{lem:momentum} we also have:
 \beq\notag
 \begin{split}
\overline{w}_{t+1} - \overline{w}_t & = \frac{\beta_1}{1 - \beta_1} \tilde{\theta}_{t-1} \left[ \eta_{t-1} \hat{v}_{t-1}^{-1/2} - \eta_{t} \hat{v}_{t}^{-1/2}\right] - \eta_{t} \hat{v}_{t}^{-1/2} \tilde{g}_t \\
& = \frac{\beta_1}{1 - \beta_1} \tilde{\theta}_{t-1}\eta_{t-1} \hat{v}_{t-1}^{-1/2} \left[ I - (\eta_{t} \hat{v}_{t}^{-1/2}) (\eta_{t-1} \hat{v}_{t-1}^{-1/2})^{-1} \right] - \eta_{t} \hat{v}_{t}^{-1/2} \tilde{g}_t \\
& = \frac{\beta_1}{1 - \beta_1} \left[ I - (\eta_{t} \hat{v}_{t}^{-1/2}) (\eta_{t-1} \hat{v}_{t-1}^{-1/2})^{-1} \right] (\tilde{w}_{t-1} - w_t) - \eta_{t} \hat{v}_{t}^{-1/2} \tilde{g}_t \eqsp,
 \end{split}
 \eeq
 where the last equality is due to $ \tilde{\theta}_{t-1}\eta_{t-1} \hat{v}_{t-1}^{-1/2} = \tilde{w}_{t-1} - w_t$ by construction of $\tilde{\theta}_t$.
 Taking the norms on both sides, observing $\| I - (\eta_{t} \hat{v}_{t}^{-1/2}) (\eta_{t-1} \hat{v}_{t-1}^{-1/2})^{-1}\| \leq 1$ due to the decreasing stepsize and the construction of $\hat{v}_t$ and using CS inequality yield:
\beq\label{eq:termB3}
 \begin{split}
\|\overline{w}_{t+1} - \overline{w}_t\| & \leq \frac{\beta_1}{1 - \beta_1} \|\tilde{w}_{t-1} - w_t\| + \|\eta_{t} \hat{v}_{t}^{-1/2} \tilde{g}_t\| \eqsp.
 \end{split}
 \eeq 
 We recall Young's inequality with a constant $\delta \in (0,1)$ as follows:
$$
\pscal{X}{Y} \leq \frac{1}{\delta} \|X\|^2 + \delta \|Y\|^2 \eqsp.
$$

 Plugging \eqref{eq:termB2} and \eqref{eq:termB3} into \eqref{eq:termB1} returns:
 \beq \notag
 \begin{split}
 \left( \nabla f(\overline{w}_t) -  \nabla f(w_t)\right)^\top (\overline{w}_{t+1} - \overline{w}_t) \leq & L \frac{\beta_1}{1 - \beta_1} \|\eta_{t} \hat{v}_{t}^{-1/2} \tilde{g}_t\|  \|w_t - \tilde{w}_{t-1}\|\\
 & +  L\left(\frac{\beta_1}{1 - \beta_1} \right)^2 \|\tilde{w}_{t-1} - w_t\|^2 \eqsp.
  \end{split}
 \eeq
 
Applying Young's inequality with $\delta \to \frac{\beta_1}{1 - \beta_1}$ on the product $ \|\eta_{t} \hat{v}_{t}^{-1/2} \tilde{g}_t\|  \|w_t - \tilde{w}_{t-1}\|$ yields:
 \beq\label{eq:termB}
 \left( \nabla f(\overline{w}_t) -  \nabla f(w_t)\right)^\top (\overline{w}_{t+1} - \overline{w}_t) \leq  L \|\eta_{t} \hat{v}_{t}^{-1/2} \tilde{g}_t\|^2 +  2L\left(\frac{\beta_1}{1 - \beta_1} \right)^2 \|\tilde{w}_{t-1} - w_t\|^2\eqsp.
 \eeq
 
 The last term $ \frac{L}{2} \|\overline{w}_{t+1} - \overline{w}_t\|$ can be upper bounded using \eqref{eq:termB3}:
\beq\label{eq:term3} 
\begin{split}
 \frac{L}{2} \|\overline{w}_{t+1} - \overline{w}_t\|^2 & \leq  \frac{L}{2} \left[ \frac{\beta_1}{1 - \beta_1} \|\tilde{w}_{t-1} - w_t\| + \|\eta_{t} \hat{v}_{t}^{-1/2} \tilde{g}_t\|\right]\\
 &  \leq L \|\eta_{t} \hat{v}_{t}^{-1/2} \tilde{g}_t\|^2 + 2L  \left(\frac{\beta_1}{1 - \beta_1}\right)^2 \|\tilde{w}_{t-1} - w_t\|^2  \eqsp.
\end{split}
\eeq

Plugging \eqref{eq:termA}, \eqref{eq:termB} and \eqref{eq:term3} into \eqref{eq:smoothness} and taking the expectations on both sides give:
\beq \notag
\begin{split}
& \EE\left[f(\overline{w}_{t+1})  +   \frac{1}{1 - \beta_1}\tilde{\major}_t^2  \|\eta_{t} \hat{v}_{t}^{-1/2} \|  - \left( f(\overline{w}_{t}) + \frac{1}{1 - \beta_1}\tilde{\major}_t^2 \|\eta_{t-1} \hat{v}_{t-1}^{-1/2}\| \right)        \right] \\
& \leq \EE \left[ - \nabla f(w_t)^\top\eta_{t-1} \hat{v}_{t-1}^{-1/2} \bar{g}_t  - \nabla f(w_t)^\top\eta_{t} \hat{v}_{t}^{-1/2} ( \beta_1 g_{t-1} +m_{t+1})   \right]\\
& + \EE \left[ 2L \|\eta_{t} \hat{v}_{t}^{-1/2} \tilde{g}_t\|^2 + 4L  \left(\frac{\beta_1}{1 - \beta_1}\right)^2 \|\tilde{w}_{t-1} - w_t\|^2  \right] \eqsp,
\end{split}
\eeq
where $ \tilde{\major}_t^2 = (a_t\beta_1^2 + \beta_1)\major^2$.
Note that the expectation of $ \tilde{g}_t $ conditioned on the filtration $\mathcal{F}_{t}$ reads as follows
\beq\notag
\begin{split}
\EE\left[    \nabla f(w_t)^\top \bar{g}_t  \right]  = \EE\left[  \nabla  f(w_t)^\top (g_t  - \beta_1 m_{t})  \right] = (1-a_t\beta_1) \| \nabla f(w_t) \|^2 \eqsp.
\end{split}
\eeq

Summing from $t=1$ to $t=T$ leads to 
\beq\label{eq:bound1}
\begin{split}
& \frac{1}{\major} \sum_{t=1}^{T_{\sf M}} \left( (1 - a_t\beta_1)   \eta_{t-1} + (\beta_1 + a_t)   \eta_{t} \right) \|\nabla f(w_t)\|^2 \leq\\
&  \EE\left[  f(\overline{w}_{1}) + \frac{1}{1 - \beta_1}\tilde{\major}_t^2 \|\eta_{0} \hat{v}_{0}^{-1/2}\|    - \left(f(\overline{w}_{T_{\sf M}+1})  +   \frac{1}{1 - \beta_1}\tilde{\major}_t^2  \|\eta_{T_{\sf M}} \hat{v}_{T_{\sf M}}^{-1/2} \| \right)      \right]\\
& +2L  \sum_{t=1}^{T_{\sf M}}  \EE \left[  \|\eta_{t} \hat{v}_{t}^{-1/2} \tilde{g}_t\|^2 \right] + 4L \left(\frac{\beta_1}{1 - \beta_1}\right)^2 \sum_{t=1}^{T_{\sf M}}  \EE \left[  \|\tilde{w}_{t-1} - w_t\|^2  \right]\\
& \leq  \EE\left[  \Delta f  + \frac{1}{1 - \beta_1}\tilde{\major}^2_t \|\eta_{0} \hat{v}_{0}^{-1/2}\|    \right] +2L  \sum_{t=1}^{T_{\sf M}}  \EE \left[  \|\eta_{t} \hat{v}_{t}^{-1/2} \tilde{g}_t\|^2 \right] \\
& + 4L \left(\frac{\beta_1}{1 - \beta_1}\right)^2 \sum_{t=1}^{T_{\sf M}}  \EE \left[  \|\tilde{w}_{t-1} - w_t\|^2  \right]\eqsp,
\end{split}
\eeq
where we denote $ \Delta f := f(\overline{w}_{1}) - f(\overline{w}_{T_{\sf M}+1})$.
We note that by definition of $\hat{v}_t$, and a constant learning rate $\eta_t$, we have
\beq \notag
\begin{split}
\|\tilde{w}_{t-1} - w_t\|^2 & =\|\eta_{t-1} \hat{v}_{t-1}^{-1/2} (\theta_{t-1} + h_{t})\|^2 \\
& =\|\eta_{t-1} \hat{v}_{t-1}^{-1/2} (\theta_{t-1} + \beta_{1} \theta_{t-2} + (1-\beta_{1}) m_{t})\|^2\\
& \leq \|\eta_{t-1} \hat{v}_{t-1}^{-1/2}\theta_{t-1}\|^2 + \|\eta_{t-2} \hat{v}_{t-2}^{-1/2} \beta_{1} \theta_{t-2}\|^2 + (1-\beta_{1})^2 \|\eta_{t-1} \hat{v}_{t-1}^{-1/2}m_{t}\|^2 \eqsp.
\end{split}
\eeq
Using Lemma~\ref{lem:squarev} we have
\beq\notag
\begin{split}
& \sum_{t=1}^{T_{\sf M}} \EE \left[  \|\tilde{w}_{t-1} - w_t\|^2  \right]\\ 
& \leq (1 + \beta_1^2) \frac{\eta^{2} d T_{\sf M} (1- \beta_1)}{(1 - \beta_2)(1-\gamma)} + (1 - \beta_1)^2 \sum_{t=1}^{T_{\sf M}} \EE [ \|\eta_{t-1} \hat{v}_{t-1}^{-1/2}m_{t}\|] \eqsp.
\end{split}
\eeq

%And thus, setting the learning rate to a constant value $\eta$, noting that $\frac{1}{(1 - a_t\beta_1) + (\beta_1 + a_t)}$ is a decreasing function for all $t>0$ and is upper bounded by $1$, injecting in \eqref{eq:bound1} yields:
%\beq\notag
%\begin{split}
%&\EE[\|\nabla f(w_T)\|^2] = \frac{ 1 }{\sum_{j=1}^{T_{\sf M}} \eta_j}  \sum_{t=1}^{T_{\sf M}} \eta_{t} \|\nabla f(w_t)\|^2 \\
%& \leq  \sum_{t=1}^{T_{\sf M}} \frac{\major}{(1 - a_t\beta_1) + (\beta_1 + a_t)}  \frac{ 1 }{\sum_{j=1}^{T_{\sf M}} \eta_j}   \EE\left[  \Delta f  + \frac{1}{1 - \beta_1}\tilde{\major}_t^2 \|\eta_{0} \hat{v}_{0}^{-1/2}\|    \right]\\
%& +   \frac{ 4L \left(\frac{\beta_1}{1 - \beta_1}\right)^2 \major }{\sum_{j=1}^{T_{\sf M}} \eta_j}  (1 + \beta_1^2) \frac{\eta^{2} d T_{\sf M} (1- \beta_1)}{(1 - \beta_2)(1-\gamma)}  \sum_{t=1}^{T_{\sf M}} \frac{ 1}{(1 - a_t\beta_1) + (\beta_1 + a_t)}  \\
%& + \frac{ \major }{\sum_{j=1}^{T_{\sf M}} \eta_j} (1 - \beta_1)^2 \sum_{t=1}^{T_{\sf M}} \EE [ \|\eta_{t-1} \hat{v}_{t-1}^{-1/2}m_{t}\|] \sum_{t=1}^{T_{\sf M}} \frac{ 1}{(1 - a_t\beta_1) + (\beta_1 + a_t)}\\
%& +    \frac{ 2L\major}{\sum_{j=1}^{T_{\sf M}} \eta_j}   \sum_{t=1}^{T_{\sf M}}  \EE [  \|\eta_{t} \hat{v}_{t}^{-1/2} \tilde{g}_t\|^2 ] \sum_{t=1}^{T_{\sf M}} \frac{ 1}{(1 - a_t\beta_1) + (\beta_1 + a_t)}\eqsp,
%\end{split}
%\eeq

Assume $ a_m = \min_{1,...,T_M} a_t$ and denote $\tilde{\major}_m^2 = (a_m\beta_1^2 + \beta_1)\major^2$.
Setting a constant learning rate $\eta_t = \eta$ and plugging in \eqref{eq:bound1} yields:
\beq\notag
\begin{split}
&\EE[\|\nabla f(w_T)\|^2] = \frac{ 1 }{\sum_{j=1}^{T_{\sf M}} \eta_j}  \sum_{t=1}^{T_{\sf M}} \eta_{t} \|\nabla f(w_t)\|^2 =\frac{\sum_{1}^{T_M}\|\nabla f(w_t)\|^2}{T_M} \\
& \leq \frac{\major}{ T_M \eta((1 - a_m\beta_1) + (\beta_1 + a_m))}    \EE\left[  \Delta f  + \frac{1}{1 - \beta_1}\tilde{\major}_m^2 \|\eta_{0} \hat{v}_{0}^{-1/2}\|    \right]\\
& +  \frac{ 4L \left(\frac{\beta_1}{1 - \beta_1}\right)^2 \major}{T_M \eta((1 - a_m\beta_1) + (\beta_1 + a_m))}  (1 + \beta_1^2) \frac{\eta^{2} d T_{\sf M} (1- \beta_1)}{(1 - \beta_2)(1-\gamma)}\\
& + \frac{\major}{T_M \eta((1 - a_m\beta_1) + (\beta_1 + a_m))}  (1 - \beta_1)^2 \sum_{t=1}^{T_{\sf M}} \EE [ \|\eta_{t-1} \hat{v}_{t-1}^{-1/2}m_{t}\|]\\
& +  \frac{2L\major}{T_M \eta((1 - a_m\beta_1) + (\beta_1 + a_m))}   \sum_{t=1}^{T_{\sf M}}  \EE [  \|\eta_{t} \hat{v}_{t}^{-1/2} \tilde{g}_t\|^2 ]  \eqsp,
\end{split}
\eeq
where $T$ is a random termination number distributed according \eqref{eq:random} and $T_M$ is the maximum number of iteration.
Setting the stepsize to $\eta = \frac{1}{\sqrt{d T_{\sf M}}}$ yields :
\beq\notag
\begin{split}
\EE[\|\nabla f(w_T)\|^2] \leq C_{1,m} \sqrt{\frac{d}{T_{\sf M}}} +C_{2,m} \frac{1}{T_{\sf M}} +  \frac{\eta}{T_{\sf M}} D_{1,m} \EE [ \| \hat{v}_{t-1}^{-1/2}m_{t}\|] +  \frac{\eta}{T_{\sf M}} D_{2,m} \EE [ \| \hat{v}_{t-1}^{-1/2} \tilde{g}_{t}\|]  \eqsp,
\end{split}
\eeq
where
\beq\notag
\begin{split}
&C_{1,m} = \frac{\major}{(1 - a_m \beta_1) + (\beta_1 + a_m)}  \Delta f + \frac{ 4L \left(\frac{\beta_1}{1 - \beta_1}\right)^2 \major}{(1 - a_m\beta_1) + (\beta_1 + a_m)} \frac{(1 + \beta_1^2) (1- \beta_1)}{(1 - \beta_2)(1-\gamma)} \eqsp ,\\
&C_{2,m} =\frac{\major}{(1 - \beta_1) \left((1 - a_m\beta_1) + (\beta_1 + a_m)\right)}  (a_m\beta_1^2 + \beta_1)\major^2   \EE[ \| \hat{v}_{0}^{-1/2}  \| ] \eqsp.
\end{split}
\eeq

\textbf{Simple case as in~\cite{zhou2018convergence}:} if $\beta_1 = 0$ then $ \tilde{g}_{t} = g_t + m_{t+1}$ and $g_t = \theta_t$. Also using Lemma~\ref{lem:squarev} we have that:
\beq\notag
\sum_{t=1}^{T_{\sf M}} \eta_{t}^{2} \EE \left[\left\|\hat{v}_{t}^{-1/2} g_{t}\right\|_{2}^{2}\right] \leq  \frac{\eta^{2} d T_{\sf M}}{(1 - \beta_2)}  \eqsp;
\eeq
which leads to the final bound:
\beq\notag
\begin{split}
\EE[\|\nabla f(w_T)\|^2]  \leq \sqrt{\frac{d}{T_{\sf M}}} \tilde{C}_{1,m}  + \frac{1}{T_{\sf M}} \tilde{C}_{2,m} \eqsp,
\end{split}
\eeq
where
\beq \notag
\begin{split}
&\tilde{C}_{1,m} = C_{1,m} +  \frac{\major}{(1 - a_m\beta_1) + (\beta_1 + a_m)} \left[ \frac{a_m(1 - \beta_1)^2}{1-\beta_2} + 2L \frac{1}{1-\beta_2}  \right] \eqsp, \\
&\tilde{C}_{2,m} = C_{2,m} =\frac{\major}{(1 - \beta_1) \left((1 - a_m\beta_1) + (\beta_1 + a_m)\right)} \tilde{\major}_m^2   \EE[ \|\hat{v}_{0}^{-1/2} \|]\eqsp.
\end{split}
\eeq
\end{proof}

\newpage

%\vspace{0.2in}

\section{Proof of Lemma~\ref{lem:dnnh2} (Boundedness of the iterates H\ref{ass:boundedparam})}\label{app:lemdnnh2}

\begin{Lemma*}
Given the multilayer model \eqref{eq:dnnmodel}, assume the boundedness of the input data and of the loss function, i.e., for any $\xi \in \rset^p$ and $y \in \rset$ there is a constant $T >0$ such that:
\beq\label{eq:mildassumptions}
\|\xi\| \leq 1 \quad \textrm{a.s.} \quad \textrm{and} |\mathcal{L}'(\cdot, y)| \leq T \eqsp,
\eeq
where $\mathcal{L}'(\cdot, y)$ denotes its derivative w.r.t. the parameter. Then for each layer $\ell \in [1,L]$, there exists a constant $A_{(\ell)}$ such that:
$$
\|w^{(\ell)}\| \leq A_{(\ell)} \eqsp.
$$
\end{Lemma*}
\begin{proof}
For any index $\ell \in [1, L]$ we denote the output of layer $\ell$ by 
$$
h^{(\ell)}(w,\xi)= \sigma\left(w^{(\ell)} \sigma\left(w^{(\ell-1)} \ldots \sigma\left(w^{(1)} \xi \right)\right)\right) \eqsp.
$$
Given the sigmoid assumption we have $\|h^{(\ell)}(w,\xi)\| \leq 1$ for any $\ell \in [1,L]$ and any $(w, \xi) \in \rset^d \times \rset^p$.
We also recall that $\mathcal{L}(\cdot, y)$ is the loss function, which can be Huber loss or cross entropy.
Observe that at the last layer $L$:
\beq\label{eq:boundderivativeloss}
\begin{split}
\|\nabla_{w^{(L)}\|  \mathcal{L}(\textsf{MLN}( w, \xi), y)} & =  \|\mathcal{L}'(\textsf{MLN}( w, \xi), y)\nabla_{w^{(L)}}\textsf{MLN}( w, \xi)\|\\
&  = \|\mathcal{L}'(\textsf{MLN}( w, \xi), y)\sigma'(w^{(L)} h^{(L-1)}(w,\xi))h^{(L-1)}(w,\xi)\|\\
& \leq \frac{T}{4} \eqsp,
\end{split}
\eeq
where the last equality is due to mild assumptions \eqref{eq:mildassumptions} and to the fact that the norm of the derivative of the sigmoid function is upper bounded by $1/4$.

\vspace{0.1in}

From Algorithm~\ref{alg:optamsgrad}, and with $\beta_1 = 0$ for the sake of notation, we have for iteration index $t >0$:
\beq \notag
\begin{split}
\|w_t - \tilde{w}_{t-1}\|  = \|-\eta_t \hat{v}_t^{-1/2} (\theta_t + h_{t+1})\|  & = \|\eta_t \hat{v}_t^{-1/2} (g_t + m_{t+1})\|  \\
& \leq \hat{\eta} \|\hat{v}_t^{-1/2} g_t\| + \hat{\eta} a \|\hat{v}_t^{-1/2} g_{t+1}\| \eqsp,
\end{split}
\eeq
where $\hat{\eta} = \max \limits_{t >0} \eta_t$.
For any dimension $p \in [1,d]$, using assumption H\ref{ass:guessbound}, we note that 
$$\sqrt{\hat{v}_{t,p}} \geq \sqrt{1-\beta_2} g_{t,p} \quad \textrm{and} \quad m_{t+1} \leq  a \norm{g_{t+1}} \eqsp.$$
Thus:
\beq\notag
\begin{split}
\|w_t - \tilde{w}_{t-1}\|  \leq \hat{\eta} \left( \|\hat{v}_t^{-1/2} g_t\| +  a \|\hat{v}_t^{-1/2} g_{t+1}\| \right) \leq \hat{\eta} \frac{a +1}{\sqrt{1 - \beta_2}}  \eqsp.
\end{split}
\eeq
In short there exists a constant $B$ such that $\|w_t - \tilde{w}_{t-1}\| \leq B$.

\vspace{0.1in}

\textbf{Proof by induction:} As in~\cite{defossez2020convergence}, we will prove the containment of the weights by induction.
Suppose an iteration index $T$ and a coordinate $i$ of the last layer $L$ such that $w^{(L)}_{T, i} \geq \frac{T}{4\lambda} + B$.
Using \eqref{eq:boundderivativeloss}, we have
$$
\nabla_i f(w^{(L)}_t, \xi) \geq - \frac{T}{4} + \lambda \frac{T}{\lambda4} \geq 0 \eqsp,
$$
where $f(w, \xi) = \mathcal{L}(\textsf{MLN}( w, \xi), y) +\frac{\lambda}{2}\norm{w}^2$ and is the loss of our MLN.
This last equation yields $\theta^{(L)}_{T,i} \geq 0$ (given the algorithm and $\beta_1 = 0$) and using the fact that $\|w_t - \tilde{w}_{t-1}\| \leq B$ we have
\beq\label{eq:decrease}
0 \leq w^{(L)}_{T-1,i} - B \leq w^{(L)}_{T,i} \leq w^{(L)}_{T-1,i} \eqsp,
\eeq
which means that $| w^{(L)}_{T,i}| \leq w^{(L)}_{T-1,i}$.
So if the first assumption of that induction reasoning holds, i.e., $w^{(L)}_{T-1, i} \geq \frac{T}{4\lambda} + B$, then the next iterates $w^{(L)}_{T, i}$ decreases, see \eqref{eq:decrease} and go below $\frac{T}{4\lambda} + B$. This yields that for any iteration index $t >0$ we have 
$$
w^{(L)}_{T, i} \leq \frac{T}{4\lambda} + 2B \eqsp,
$$
since $B$ is the biggest jump an iterate can do since $\|w_t - \tilde{w}_{t-1}\| \leq B$.
Likewise we can end up showing that 
$$
|w^{(L)}_{T, i}| \leq \frac{T}{4\lambda} + 2B \eqsp,
$$
meaning that the weights of the last layer at any iteration is bounded in some matrix norm.

Now that we have shown this boundedness property for the last layer $L$, we will do the same for the previous layers and conclude the verification of assumption H\ref{ass:boundedparam} by induction.

For any layer $\ell \in [1, L-1]$, we have:
\beq\label{eq:gradientatell}
\nabla_{w^{(\ell)}}  \mathcal{L}(\textsf{MLN}( w, \xi), y)  =  \mathcal{L}'(\textsf{MLN}( w, \xi), y) \left(\prod_{j=1}^{\ell+1} \sigma'\left(w^{(j)} h^{(j-1)}(w,\xi) \right) \right) h^{(\ell-1)}(w,\xi)  \eqsp.
\eeq
This last quantity is bounded as long as we can prove that for any layer $\ell$ the weights $w^{(\ell)}$ are bounded in some matrix norm as $\|w^{(\ell)}\|_{F} \leq F_\ell$ with the Frobenius norm.
Suppose we have shown $\norm{w^{(r)}}_{F} \leq F_r$ for any layer $r > \ell$. 
Then having this gradient \eqref{eq:gradientatell} bounded we can use the same lines of proof for the last layer $L$ and show that the norm of the weights at the selected layer $\ell$ satisfy
$$
\|w^{(\ell)}\| \leq \frac{T \prod_{t > \ell} F_t}{4^{L-\ell+1}} + 2B\eqsp.
$$
Showing that the weights of the previous layers $\ell \in [1, L-1]$ as well as for the last layer $L$ of our fully connected feed forward neural network are bounded at each iteration, leads by induction, to the boundedness (at each iteration) assumption we want to check, thus proving Lemma~\ref{lem:dnnh2}.
\end{proof}

\vspace{0.2in}

%-----------------------------------------------------------------------------
%\vspace{0.4cm}

\end{document}